\documentclass{article}
\usepackage{arxiv}

\usepackage{graphicx}
\usepackage{hyperref}
\hypersetup{
	colorlinks   = true, %Colours links instead of ugly boxes
	urlcolor     = blue, %Colour for external hyperlinks
	linkcolor    = blue, %Colour of internal links
	citecolor   = purple %Colour of citations
}
\usepackage[ruled,vlined]{algorithm2e}
\usepackage{bbm}
\usepackage{amssymb}
\usepackage{caption}
\usepackage{subcaption}
\usepackage{lscape}

\usepackage{latexsym}
\usepackage{amsmath}
\usepackage{amsthm}
\usepackage{amsfonts}
\usepackage{xcolor}
\usepackage{multirow}
\usepackage[normalem]{ulem}
\newenvironment{manac1}[1][htb]
  {% Update algorithm name
  \begin{algorithm}[#1]%
  }{\end{algorithm}}
  
\newenvironment{manac2}[1][htb]
{% Update algorithm name
\begin{algorithm}[#1]%
}{\end{algorithm}}

\newenvironment{manac3}[1][htb]
{% Update algorithm name
\begin{algorithm}[#1]%
}{\end{algorithm}}
  
\newenvironment{marl}[1][htb]
{% Update algorithm name
\begin{algorithm}[#1]%
}{\end{algorithm}}

\allowdisplaybreaks

\newtheorem{assumption}{X.}

\newtheorem{theorem}{Theorem}
\newtheorem{remark}{Remark}
\newtheorem{lemma}{Lemma}
\newtheorem{definition}{Definition}

\title{Multi-agent Natural Actor-critic Reinforcement Learning Algorithms}

\author{
 Prashant Trivedi \\
  Industrial Engineering and Operations Research \\
  Indian Institute of Technology Bombay India\\
  \texttt{trivedi.prashant15@iitb.ac.in} \\
   \And
 Nandyala Hemachandra\\
  Industrial Engineering and Operations Research \\
  Indian Institute of Technology Bombay India\\
  \texttt{nh@iitb.ac.in}
}

\begin{document}

\maketitle

\begin{abstract}
Multi-agent actor-critic algorithms are an important part of the Reinforcement Learning paradigm. We propose three fully decentralized multi-agent natural actor-critic (MAN) algorithms in this work. The objective is to collectively find a joint policy that maximizes the average long-term return of these agents. In the absence of a central controller and to preserve privacy, agents communicate some information to their neighbors via a time-varying communication network. We prove convergence of all the 3 MAN algorithms to a globally asymptotically stable set of the ODE corresponding to actor update; these use linear function approximations. We show that the Kullback-Leibler divergence between policies of successive iterates is proportional to the objective function's gradient. We observe that the minimum singular value of the Fisher information matrix is well within the reciprocal of the policy parameter dimension. Using this, we theoretically show that the optimal value of the deterministic variant of the MAN algorithm at each iterate dominates that of the standard gradient-based multi-agent actor-critic (MAAC) algorithm. To our knowledge, it is a first such result in multi-agent reinforcement learning (MARL). 
To illustrate the usefulness of our proposed algorithms, we implement them on a bi-lane traffic network to reduce the average network congestion. We observe an almost 25\% reduction in the average congestion in 2 MAN algorithms; the average congestion in another MAN algorithm is on par with the MAAC algorithm. We also consider a generic $15$ agent MARL; the performance of the MAN algorithms is again as good as the MAAC algorithm.
\end{abstract}


\section{Introduction}
\label{sec: introduction_manac}

Reinforcement learning (RL) has been explored in recent years and is of great interest to researchers because of its broad applicability in many real-life scenarios. In RL, agents interact with the environment and take decisions sequentially. It is applied successfully to various problems, including elevator scheduling, robot control, etc. There are many instances where RL agents surpass human performance, such as openAI beating the world champion DOTA player, DeepMind beating the world champion of Alpha Star.

% \textcolor{red}{This para can be removed.} 
The sequential decision-making problems are generally modeled via Markov decision process (MDP). It requires the knowledge of system transitions and rewards. In contrast, RL is a data-driven MDP framework for sequential decision-making tasks; the transition probability matrices and the reward functions are not assumed, but their realizations are available as observed data. 

In RL, the purpose of an agent is to learn an optimal or nearly-optimal policy that maximizes the ``reward function" or functions of other user-provided  ``reinforcement signals" from the observed data. However, in many realistic scenarios, there is more than one agent. To this end, researchers explore the multi-agent reinforcement learning (MARL) methods, but most are centralized and relatively slow. Furthermore, these MARL algorithms use the standard/vanilla gradient, which has limitations. For example, the standard gradients cannot capture the angles in the state space and may not be effective in many scenarios. 
The natural gradients are more suitable choices because they capture the intrinsic curvature in the state space. In this work, we are incorporating natural gradients in the MARL framework.

In the multi-agent setup that we consider, the agents have some private information and a common goal. This goal could be achieved by deploying a central controller and converting the MARL problem into a single-agent RL problem. However, deploying a central controller often leads to scalability issues. On the other hand, if there is no central controller and the agents do not share any information then there is almost no hope of achieving the common goal. An intermediate model is to share some parameters via (possibly) a time-varying, and sparse communication matrix \cite{zhang2018fully}. The algorithms based on such intermediate methods are often attributed as consensus based algorithms. 

% \textcolor{red}{This can be removed because earlier we wrote it to incorporate the dynamic games keyword in the paper.}
The consensus based algorithm models can also be considered as intermediate between dynamic non-cooperative and cooperative game models. Non-cooperative games, as multi-agent systems, model situations where the agents do not have a common goal and do not communicate. On the contrary, cooperative games model situations where a central controller achieves a common goal using complete information.

Algorithm 2 of \cite{zhang2018fully} is a consensus based actor-critic algorithm. We call it MAAC (multi-agent actor-critic) algorithm. The MAAC algorithm uses the standard gradient and hence lacks in capturing the intrinsic curvature present in the state space. We propose three multi-agent natural actor-critic (MAN) algorithms and incorporate the curvatures via natural gradients. These algorithms use the linear function approximations for the state value and reward functions.
We prove the convergence of all the 3 MAN algorithms to a globally asymptotically stable equilibrium set of ordinary differential equations (ODEs) obtained from the actor updates. 
% The equilibrium set gives the local optima of the objective function.

Here is a brief overview of our two time-scale approach. Let $J(\theta)$ be the global MARL objective function of $n$ agents, where $\theta = (\theta^1, \dots, \theta^n)$ is the actor (or policy) parameter. For a given policy parameter $\theta$ of each MAN algorithm,
%For each MAN algorithm, a given policy parameter $\theta$, 
we first show in Theorem \ref{thm: critic_convergence_fisher} the convergence of critic parameters (to be defined later) on a faster time scale. Note that these critic parameters are updated via the communication matrix. We then show the  convergence of each agent's actor parameters to an asymptotically stable attractor set of its ODE. These actor updates use the natural gradients in the form of Fisher information matrix and advantage parameters (Theorem \ref{thm: MANAC-algo_1}, \ref{thm: actor_convergence_advantage_parameter} and \ref{thm: actor_convergence_advantage_parameter_fisher}).
The actor parameter $\theta$ is shown to converge on the slower time scale.

Our MAN algorithms use a log-likelihood function via the Fisher information matrix and incorporate the curvatures.
% Being a first order method, the standard gradients methods overlooks
%\sout{do not have information about} 
% the curvatures in the objective function. However, incorporating the Fisher information matrix to the standard gradients scales the parameter updates via a log-likelihood function.
We show that this log-likelihood function is indeed the KL divergence between the consecutive policies, and it is the gradient of the objective function up to scaling (Lemma \ref{lemma: kl_grad_J}). Unlike standard gradient methods, where the updates are restricted to the parameter space only, the natural gradient-based methods allow the updates to factor in the curvature of the policy distribution prediction space via the KL divergence between them. {Thus, 2 of our MAN algorithms, FI-MAN and FIAP-MAN, use a certain \emph{representation} of the objective function gradient in terms of the gradient of this KL divergence (Lemma \ref{lemma: kl_grad_J}). It turns out these two algorithms have much better empirical performance (Section \ref{subsubsec: traffic_network_details}).}

% \textcolor{blue}{2 paras here; more on various viewpoints of natural gradients and algos that provably converge to a better local minima }

We now point out a couple of important consequences of the representation learning aspect of our MAN algorithms for reinforcement learning. First, we show that under some conditions, our deterministic version of the FI-MAN algorithm converges to local minima with a better objective function value than the deterministic counterpart of the MAAC algorithm, Theorem \ref{thm: J_comp_t+1}. To the best of our knowledge, this is a new result in non-convex optimization; we are not aware of any algorithm that is \emph{proven} to converge to a better local maxima \cite{bottou2018optimization,nocedal2006numerical}. This relies on the important observation, which can be of independent interest, that $1/m$ is  \emph{uniform} upper bound on the smallest singular value of Fisher information matrix $G(\theta)$, Lemma \ref{lemma: sigma_min_less_1_m}; here $m$ is the common  dimension of the compatible policy parameter $\theta$ and the Fisher information matrix $G(\theta)$.

The natural gradient-based methods can be viewed as quasi-second order methods, as the Fisher information matrix $G(\cdot)$ is an invertible linear transformation of basis that is used in first-order optimization methods \cite{agarwal2021theory}. However, they are not regarded as second-order methods because the Fisher information matrix is not the Hessian of the \textit{objective function}. %\textcolor{blue}{to complete}

To validate the usefulness of our proposed algorithms, we
perform a comprehensive set of computational experiments in two settings: a bi-lane traffic network and an abstract MARL model. 
On a bi-lane traffic network model, the objective is to find the traffic signaling plan that reduces the overall network congestion. We consider two different arrival patterns between various origin-destination (OD) pairs. With the suitable linear function approximations to incorporate the humongous state space $(50^{16})$ and action space $(3^4)$, we observe a significant reduction ($\approx 25\%$) in the average network congestion in 2 of our MAN algorithms. One of our MAN algorithms that are only based on the advantage parameters and never estimate the Fisher information matrix inverse is on-par with the MAAC algorithm. In the abstract MARL model, we consider $15$ agents with 15 states and 2 actions in each state and generic reward functions  \cite{dann2014policy,zhang2018fully}. Each agent's reward is private information and hence not known to other agents. Our MAN algorithms either outperform or are on-par with the MAAC algorithm with high confidence.

\textbf{Organization of the paper:}
In Section \ref{sec: MARL_framework} we introduce the multi-agent reinforcement learning (MARL) and a novel natural gradients framework. We then propose MAN algorithms, namely FI-MAN, AP-MAN, and FIAP-MAN, and provide some theoretical insights on their relative performance in Section \ref{sec: MANAC}. Convergence proofs of all the algorithms are available in Section \ref{sec: Convergence_analysis}. Finally, we provide the computational experiments for modeling traffic network control and another abstract MARL problem in Section \ref{sec: experiments}. 

The detailed proof of theorems and lemmas are available in Appendix \ref{app: proofs}. More details of the computations and empirical observations are available in Appendix \ref{app: experiments_detailed}. We provide some backgrounds on MDP, single-agent actor-critic algorithm, and the MAAC algorithms in Appendix \ref{app: relevant_background}.

\section{MARL framework and natural gradients}
\label{sec: MARL_framework}

This section describes the multi-agent reinforcement learning (MARL) framework and the notion of natural gradients. We first describe some basic notations and the multi-agent Markov decision process (MDP).

Let $N = \{1,2,\dots, n\}$ denote the set of agents. Each agent independently interacts with a stochastic environment and takes a local action. We consider a fully decentralized setup in which a communication network connect the agents.
% Each node of the communication network represents an agent. 
This network is used to exchange information among agents in the absence of a central controller so that agents' privacy remains intact. The communication network is possibly time-varying and sparse. \textcolor{black}{We assume the communication among agents is synchronized, and hence there are no information delays. Moreover, only some parameters (that we define later) are shared among the neighbors of each agent. It also addresses an important aspect of the privacy protection of such agents.} Formally, the communication network is characterized by an undirected graph $\mathcal{G}_t = (N, \mathcal{E}_t)$, where $N$ is the set of all nodes (or agents) and $\mathcal{E}_t$ is the set of communication links available at time $t\in \mathbb{N}$. We say, agents $i,j \in N$ communicate at time $t$ if $(i,j)\in \mathcal{E}_t$. 

Let $\mathcal{S}$ denote the common state space available to all the agents. At any time $t$, each agent observes a common state $s_t\in \mathcal{S}$, and takes a local action $a^i_t$ from the set of available actions $\mathcal{A}^i$.
\textcolor{black}{We assume that for any agent $i\in N$, the entire action set $\mathcal{A}^i$ is feasible in every state, $s\in \mathcal{S}$.} The action $a^i_t$ is taken as per a local policy $\pi^i: \mathcal{S} \times \mathcal{A}^i \rightarrow [0,1]$, where $\pi^i(s_t,a^i_t)$ is the probability of taking action $a^i_t$ in state $s_t$ by agent $i\in N$.  Let $\mathcal{A} := \prod_{i=1}^n \mathcal{A}^i$ be the joint action space of all the agents. To each state and action pair, every agent receives a finite reward from the local reward function ${R}^i: \mathcal{S} \times \mathcal{A} \rightarrow \mathbb{R}$. Note that the reward is private information of the agent and it is not known to other agents. The state transition probability of MDP is given by $P: \mathcal{S} \times \mathcal{A} \times \mathcal{S} \rightarrow [0,1]$. Using only local rewards and actions it is hard for any classical reinforcement learning algorithm to maximize the averaged reward determined by the joint actions of all the agents. To this end, we consider the multi-agent networked MDP given in \cite{zhang2018fully}. The multi-agent networked MDP is defined as $(\mathcal{S}, \{\mathcal{A}^i\}_{i\in N}, P, \{R^i\}_{i\in N}, \{\mathcal{G}_t\}_{t\geq 0})$, with each component described as above.  Let joint policy of all agents be denoted by $\pi:\mathcal{S} \times \mathcal{A} \rightarrow [0,1]$ satisfying $\pi(s,a) = \prod_{i\in N} \pi^i(s,a^i)$. 
Let $a_t = (a^1_t,\dots, a^n_t)$ be the action taken by all the agents at time $t$. Depending on the action $a^i_t$ taken by agent $i$ at time  $t$, the agent receives a random reward $r^i_{t+1}$ with the expected value $R^i(s_t,a_t)$. Moreover, with probability $P(s_{t+1}| s_t,a_t)$ the multi-agent MDP shifts to next state $s_{t+1} \in \mathcal{S}$. 

Due to the large state and action space, it is often helpful to consider the parameterized policies \cite{grondman2012survey,sutton2018reinforcement}. 
We parameterize the local policy, $\pi^i(\cdot, \cdot)$ by $\theta^i\in \Theta^i  \subseteq \mathbb{R}^{m_i}$,  where $\Theta^i$ is the compact set. To find the global policy parameters we can pack all the local policy parameters as $\theta = [(\theta^1)^{\top}, \dots, (\theta^n)^{\top}]^{\top}\in \Theta \subseteq  \mathbb{R}^m$, where $\Theta := \prod_{i\in N}\Theta^i$, and $m=\sum_{i=1}^n m_i$. The parameterized joint policy is then given by $\pi_{\theta}(s,a) = \prod_{i\in N}\pi^i_{\theta^i}(s,a^i)$. The objective of the agents is to collectively find a joint policy $\pi_{\theta}$ that maximizes the averaged long-term return, provided each agent has local information only. 
For a given policy parameter $\theta$, let the globally averaged long-term return be denoted by $J(\theta)$ and defined as 
\begin{equation*}
\label{eqn: objective_fn}
J(\theta) = lim_{T\rightarrow \infty}~ \frac{1}{T} \mathbb{E} \left( \sum_{t=0}^{T-1} \frac{1}{n} \sum_{i\in {N}} r^i_{t+1} \right)
= \sum_{s\in \mathcal{S}} d_{\theta}(s) \sum_{a\in \mathcal{A}}  \pi_{\theta}(s,a) \bar{R}(s,a),
\end{equation*}
where $\bar{R}(s,a) = \frac{1}{n} \sum_{i\in {N}} R^i(s,a)$ is the globally averaged reward function. Let $\bar{r}_t = \frac{1}{n}\sum_{i\in {N}}r^i_t$. Thus $\bar{R}(s,a) = \mathbb{E}[\bar{r}_{t+1}|s_t=s,a_t=a]$.
Therefore, the joint objective of the agents is to solve the following optimization problem 
\begin{equation}
\label{eqn: objective}
max_{\theta} ~ J(\theta) = max_{\theta} ~ \sum_{s\in \mathcal{S}} d_{\theta}(s) \sum_{a\in \mathcal{A}}  \pi_{\theta}(s,a) \bar{R}(s,a).
\end{equation}
% where $\bar{R}(s,a) = \frac{1}{n} \sum_{i\in {N}} R^i(s,a)$ is the globally averaged reward function. Let $\bar{r}_t = \frac{1}{n}\sum_{i\in {N}}r^i_t$. Thus $\bar{R}(s,a) = \mathbb{E}[\bar{r}_{t+1}|s_t=s,a_t=a]$.
Like single-agent RL \cite{bhatnagar2009natural}, we also require the following regularity assumption on networked multi-agent MDP and parameterized policies. 
\begin{assumption}
% [Regularity of Networked MDP]
\label{ass: regularity}
For each agent $i\in N$, the local policy function $\pi^i_{\theta^i}(s,a^i)>0$ for any $s \in \mathcal{S}, a^i\in \mathcal{A}^i$ and $\theta^i \in \Theta^i$. Also $\pi^i_{\theta^i}(s,a^i)$ is continuously differentiable with respect to parameters $\theta^i$ over $\Theta^i$. Moreover, for any $\theta \in \Theta$, $P^{\theta}$ is the transition matrix for the Markov chain $\{s_t\}_{t\geq 0}$ induced by policy $\pi_{\theta}$, that is, for any $s,s^{\prime} \in \mathcal{S}$,
\begin{equation*}
\label{eqn: P^theta}
    P^{\theta}(s^{\prime}|s) = \sum_{a\in \mathcal{A}} \pi_{\theta}(s,a) P(s^{\prime}|s,a).
\end{equation*}
Furthermore, the Markov chain $\{s_t\}_{t\geq 0}$ is assumed to be ergodic under $\pi_{\theta}$ with stationary distribution $d_{\theta}(s)$ over $\mathcal{S}$. 
\end{assumption}

The regularity assumption X. \ref{ass: regularity} on a multi-agent networked MDP is standard in the work of single agent actor-critic algorithms with function approximations \cite{konda2000actor,bhatnagar2009natural}. The continuous differentiability of policy $\pi_{\theta}(\cdot, \cdot)$ with respect to $\theta$ is required in policy gradient theorem \cite{sutton2018reinforcement}, and it is commonly satisfied by well-known class of functions such as neural networks or deep neural networks. Moreover, assumption X. \ref{ass: regularity} also implies that the Markov chain $\{(s_t,a_t)\}_{t\geq 0}$ has stationary distribution $\tilde{d}_{\theta}(s,a) =  d_{\theta}(s) \cdot \pi_{\theta}(s,a)$ for any $s \in \mathcal{S}, a\in \mathcal{A}$.

Based on the objective function given in Equation (\ref{eqn: objective}) the  global state-action value function associated with state-action pair $(s,a)$ for a given policy $\pi_{\theta}$ is defined as
\begin{equation}
\label{eqn: Global_Q_fnction}
Q_{\theta}(s,a) = \sum_{t\geq 0} \mathbb{E}[\bar{r}_{t+1}-J(\theta) | s_0 =s, a_0=a, \pi_{\theta}].
\end{equation}
Note that the global state-action value function, $Q_{\theta}(s,a)$ given in Equation (\ref{eqn: Global_Q_fnction}) is motivated from the gain and bias relation for the average reward criteria of the single agent MDP as given in say, Section 8.2.1 in \cite{puterman2014markov}. It captures the expected sum of fluctuations of the global rewards about the globally averaged objective function (`average adjusted sum of rewards'  \cite{mahadevan1996average}) when action $a = (a^1,a^2, \dots, a^n)$ is taken in state $s\in \mathcal{S}$ at time $t=0$,  and thereafter the policy $\pi_{\theta}$ is followed. 
% Above equation is also motivated from the gain and bias relation for the average criteria of single agent MDP as given in say, Section 8.2.1 in \cite{puterman2014markov}.
%\textcolor{black}{this Q fn may be motivated the Bellman/DP eqn; in the average case, such eqns have a bias and gain terms -- Puterman interprets them, via, graphs also for {\em MDPs}. we can try those explanations here, if they make sense.}
%\textcolor{black}{not clear about the `cumulative return'} 
Similarly, the global state value function is defined as
\begin{equation}
\label{eqn: global_V_fn}
V_{\theta}(s) = \sum_{a\in \mathcal{A}} \pi_{\theta}(s,a) \cdot Q_{\theta}(s,a).
\end{equation}
We now state the policy gradient theorem for MARL setup \cite{zhang2018fully}. To this end, define the global advantage function  $A_{\theta}(s,a) := Q_{\theta}(s,a) - V_{\theta}(s)$.
\textcolor{black}{Note that the advantage function captures the benefit of taking action $a$ in state $s$ and thereafter following the policy $\pi_{\theta}$ over the case when policy $\pi_{\theta}$ is followed from state $s$ itself.}
For the multi-agent setup, we define the local advantage function $A^i_{\theta}: \mathcal{S} \times \mathcal{A} \rightarrow \mathbb{R}$ for each agent $i\in N$ as $A^i_{\theta}(s,a) := Q_{\theta}(s,a) - \widetilde{V}^i_{\theta}(s,a^{-i})$, where $\widetilde{V}^i_{\theta}(s,a^{-i}) := \sum_{a^i\in A^i} \pi^i_{\theta^i} (s,a^i) Q_{\theta}(s,a^i,a^{-i})$. Note that $\widetilde{V}^i_{\theta}(s,a^{-i})$ represents the value of state $s$ to an agent $i\in N$ when policy, $\pi(\cdot,\cdot)$ is parameterized by $\theta  = [(\theta^1)^{\top}, (\theta^2)^{\top}, \dots, (\theta^n)^{\top}]$, and all other agents are taking action $a^{-i} = (a^1, \dots, a^{i-1}, a^{i+1}, \dots, a^n)$.

% \textcolor{black}{both the weight vectors of policy and advantage approx are related as in Eq 3 and policy gradient theorem in the paper by Rick Sutton et al -- to discuss }

\begin{theorem}[Policy gradient theorem for MARL \cite{zhang2018fully}]
\label{thm: pgt}
Under assumption X. \ref{ass: regularity}, for any $\theta \in \Theta$, and each agent $i\in N$, the gradient of $J(\theta)$ with respect to $\theta^i$ is given by 
\begin{equation*}
\label{eqn: pgt_grad_J}
\begin{aligned}
    \nabla_{\theta^i}J(\theta) &=& \mathbb{E}_{s\sim d_{\theta}, a\sim \pi_{\theta}}[\nabla_{\theta^i} \log \pi^{i}_{\theta^i}(s,a^i) \cdot A_{\theta}(s,a)]
    \\
    &=& \mathbb{E}_{s\sim d_{\theta}, a\sim \pi_{\theta}}[\nabla_{\theta^i} \log \pi^{i}_{\theta^i}(s,a^i) \cdot A^i_{\theta}(s,a)].
\end{aligned}
\end{equation*}
\end{theorem}

% \textcolor{black}{probably we don't need boundedness of the second derivative of $J$ as in their final theorem, for SA to converge toa local optima?? the later SA theory dropped this, looks like?}

%\textcolor{black}{should we just write details of policy gradient theorem when natural gradients are used? } 
The proof of this theorem is available in \cite{zhang2018fully}. For the sake of completeness we provide the proof in Appendix \ref{app: pgt}. \textcolor{black}{The idea of the proof is as follows: We first recall the policy gradient theorem for single agent. Now using the fact that for multi-agent case the global policy is product of  local policies, i.e., $\pi_{\theta}(s,a) = \prod_{i=1}^n \pi^i_{\theta^i}(s,a^i)$, and $\sum_{a^i\in \mathcal{A}^i} \pi^i_{\theta^i} (s,a^i) = 1$, hence $\nabla_{\theta^i} \left[\sum_{a^i\in \mathcal{A}^i} \pi^i_{\theta^i} (s,a^i)\right] = 0$, we show $\nabla_{\theta^i} J(\theta) = \mathbb{E}_{s\sim d_{\theta},a \sim \pi_{\theta}}[\nabla_{\theta^i} \log \pi^i_{\theta^i}(s,a^i) \cdot Q_{\theta}(s,a)] $. Now, observe that adding/subtracting any function $\Lambda$ that is independent of the action $a^i$ taken by agent $i\in N$ to $Q_{\theta}(s,a)$ doesn't make any difference in the above expected value. In particular, considering two such $\Lambda$ functions $V_{\theta}(s)$ and $\tilde{V}_{\theta}^i(s,a^{-i})$, we have desired results. }

We refer to $\psi^i(s,a^i) := \nabla_{\theta^i} \log \pi^{i}_{\theta^i} (s,a^i)$ as the score function. We will see in Section \ref{subsec: manac2} that the same score function is called the compatible features. This is because the above policy gradient theorem with function approximations require the compatibility condition (Theorem 2 \cite{sutton1999policy}).
The policy gradient theorem for MARL relates the gradient of the global objective function w.r.t. $\theta^i$ and the local advantage function $A^i_{\theta}(\cdot, \cdot)$. It also suggests that the global objective function's gradients can be obtained solely using the local score function, $\psi^i(s,a^i)$ if agent $i\in N$ has an unbiased estimate of the advantage functions $A^i_{\theta}$ or $A_{\theta}$. However, estimating the advantage function requires the rewards $r^i_t$ of \textit{all} the agents $i\in N$; therefore, these functions cannot be well estimated by any agent $i\in N$ alone. 
To this end, \cite{zhang2018fully} have proposed two fully decentralized actor-critic algorithms based on the consensus network. These algorithms work in a fully decentralized fashion and empirically achieve the same performance as a centralized algorithm in the long run. 
% \sout{However, \textcolor{black}{We use algor2 ...; drop `however'} we are using .}
We use algorithm 2 of \cite{zhang2018fully} which we are calling as multi-agent actor-critic (MAAC) algorithm. 
%\textcolor{black}{that performance is same is shown by them?} 

In the fully decentralized setup, we consider the weight matrix $C_t = [c_t(i,j)]$, depending on the network topology of communication network $\mathcal{G}_t$. Here $c_t(i,j)$ represents the weight of the message transmitted from agent $i$ to agent $j$ at time $t$. At any time $t$, the local parameters are updated by each agent using this weight matrix. For generality, we take the weight matrix $C_t$ to be random. This is either because $\mathcal{G}_t$ is a time-varying graph or the randomness in the consensus algorithm \cite{boyd2006randomized}. The weight matrix satisfies the following assumptions \cite{zhang2018fully}.
\begin{assumption}
\label{ass: comm_matrix}
The sequence of non-negative random matrices $\{C_t\}_{t \geq 0} \subseteq \mathbb{R}^{n\times n}$ satisfy the following:
\begin{enumerate}
    \item $C_t$ is row stochastic, i.e., $C_t \mathbbm{1} = \mathbbm{1}$. Moreover, $\mathbb{E}(C_t)$ is column stochastic, i.e., $\mathbbm{1}^{\top}\mathbb{E}(C_t)=\mathbbm{1}^{\top}$. Furthermore, there exists a constant $\gamma \in (0,1)$ such that for any $c_t(i,j)>0$, we have $c_t(i,j)\geq \gamma$.
    \item Weight matrix $C_t$ respects $\mathcal{G}_t$, i.e., $c_t(i,j) = 0$ if $(i,j)\notin \mathcal{E}_t$. 
    \item The spectral norm of $\mathbb{E}[C_t^{\top} (I - \mathbbm{1}\mathbbm{1}^{\top}/n)C_t]$ is smaller than one. 
    \item Given the $\sigma$-algebra generated by the random variables before time $t$, $C_t$ is conditionally independent of $r^i_{t+1}$ for each $i\in {N}$.
\end{enumerate}
\end{assumption}

Assumption X. \ref{ass: comm_matrix}(1) of considering a doubly stochastic matrix is standard in the work of consensus-based algorithms \cite{mathkar2016nonlinear}. It is often helpful in the convergence of the update to a common vector \cite{bianchi2013performance}. To prove the stability of the consensus update (see Appendix A of \cite{zhang2018fully} for detailed proof), we require the lower bound on the weights of the matrix \cite{nedic2009distributed}. 
Assumption X. \ref{ass: comm_matrix}(2) is required for the connectivity of $\mathcal{G}_t$. To provide the geometric convergence in distributed optimization, authors in \cite{nedic2017achieving} provide the connection between the time-varying network and the spectral norm property. The same connection is required for convergence in our work also. To this end, we have assumption X. \ref{ass: comm_matrix}(3) above. Assumption X. \ref{ass: comm_matrix}(4) on the conditional independence of $C_t$ and $r_{t+1}$ is common in many practical multi-agent systems. The Metropolis matrix \cite{xiao2005scheme} given below satisfy all the assumptions. It is defined only based on the local information of the agents. Let $N_t(i):= \{j\in N:(i,j)\in \mathcal{E}_t\}$ be set of neighbors of agent $i\in N$ at time $t$ in the weight matrix $C_t$, and $d_t(i) = |N_t(i)|$ is the degree of agent $i\in N$. The weights $c_t(i,j)$ in the Metropolis matrix are given by
\begin{eqnarray*}
    c_t(i,j) &=& \frac{1}{1+max\{d_t(i),d_t(j)\}}, ~~ \forall ~ (i,j)\in \mathcal{E}_t
    \\
    c_t(i,i)  &=& 1- \sum_{j\in N_t(i)} c_t(i,j),~~ \forall~ i\in N.
\end{eqnarray*} 
Now, we will outline the actor-critic algorithm using linear function approximations in a fully decentralized setting.
The actor-critic algorithm consists of two steps -- critic step and actor step. At each time $t$, the actor suggests a policy parameter $\theta_t$. The critic evaluates its value using the policy parameters and criticizes or gives the feedback to the actor. Using this feedback from critic, the actor then update the policy parameters, and this continues until convergence. Let the global state value temporal difference (TD) error be defined as $\bar{\delta}_t = \bar{r}_{t+1} - J(\theta) + V_{\theta}(s_{t+1}) - V_{\theta}(s_t)$.
% It captures the difference in the actual reward obtained, and the reward estimates. 
It is known that the state value temporal difference error is an unbiased estimate of the advantage function $A_{\theta}$  \cite{sutton2018reinforcement}, i.e.,
% \textcolor{black}{some (other) motivation/discussion for TD error??}
\begin{equation}
    \mathbb{E}[\bar{\delta}_t~|~ s_t=s, a_t =a, \pi_{\theta}] = A_{\theta}(s,a),~~ \forall~ s\in \mathcal{S},~a\in \mathcal{A}. 
\label{eqn: td_unbiased_advantage}
\end{equation}
The TD error specifies how different the new value is from the old prediction. Often in many applications \cite{corke2005networked,dall2013distributed} the state space is either large or infinite.
%\sout{Therefore, exact state value is not available.} 
To this end, we use the linear function approximations for state value function. Later on, we also use the linear function approximation for the advantage function in Section \ref{subsec: manac2}. Let the state value function $V_{\theta}(s)$ be approximated using the linear function as $ V_{\theta}(s; v) := v^{\top}\varphi(s)$, where $\varphi(s) = [\varphi_1(s), \dots, \varphi_L(s)]^{\top} \in \mathbb{R}^L$ is the feature associated with state $s$, and $v\in \mathbb{R}^L$. 
% \textcolor{black}{why $V_\theta$ when the basis function is $\varphi$?}
Note that $L << |\mathcal{S}|$, hence the value function is approximated using very small number of features. Moreover, let $\mu^i_t$ be the estimate of the global objective function $J(\theta)$ by agent $i\in N$ at time $t$. Note that $\mu^i_t$ tracks the long-term return to each agent $i \in N$. The MAAC algorithm is based on the consensus network (details in Appendix \ref{app: MARL_algo2}) and consists of the following updates for objective function estimate and the critic parameters
\begin{eqnarray}
\label{eqn: marl_obj_updates}
\tilde{\mu}^i_t = (1-\beta_{v,t}) \cdot \mu^i_t + \beta_{v,t} \cdot r^i_{t+1}; ~~\mu^i_{t+1} =  \sum_{j\in {N}} c_t(i,j) \tilde{\mu}^j_t
\\
\tilde{v}^i_t = v^i_t + \beta_{v,t} \cdot \delta^i_t \cdot \nabla_v V_t(v^i_t); ~~ v^i_{t+1} = \sum_{j\in {N}}c_t(i,j)\tilde{v}^j_t,
\label{eqn: marl_critic_updates}
\end{eqnarray}
where $\beta_{v,t}> 0$ is the critic step-size and $\delta^i_t = r^i_{t+1} - \mu^i_t + V_{t+1}(v^i_t) - V_{t}(v^i_t)$ is the local TD error. \textcolor{black}{Here $V_{t+1}(v^i_t) := v_t^{i^{\top}} \varphi(s_{t+1})$, and hence $V_{t+1}(v^i_t) = V_{\theta}(s_{t+1}; v^i_t)$. It is a linear function approximation of the state value function, $V_{\theta}(s_{t+1})$ by agent $i\in N$.}
Note that the estimate of the advantage function as given in Equation (\ref{eqn: td_unbiased_advantage}) require $\bar{r}_{t+1}$ which is not available to each agent $i\in N$. Therefore, we parameterize the reward function $\bar{R}(s,a)$ used in the critic update as well. 
% The first algorithm is based on the local advantage function $A^i_{\theta}$ which requires estimating the action-value function ${Q}_{\theta}$ of policy $\pi_{\theta}$. We assume that each agent maintains its own parameter $w^i$ and uses ${Q}(\cdot,\cdot,w)$ as a local estimate of ${Q}_{\theta}$. This $w^i$ is also shared by each agent to its neighbours over the network in order to reach a consensual estimate of ${Q}_{\theta}$. Note that for online implementation of Algorithm \ref{alg: MARL-algo1}, the joint actions $a_{t+1}$ is needed to evaluate the action-value TD-error $\delta_t^i$. The detailed algorithm is given in Appendix \ref{sup: MARL_algo1}
% Since the advantage function $A_{\theta}$ is an unbiased estimate of $\bar{\delta}_t$, thus $r^i_{t+1}$ can't be used to estimate it, because it requires $\bar{r}_{t+1}$, and hence $\delta^i_t$ cannot be used for the estimation. Thus, a globally averaged reward function $\bar{R}$ is used in the critic update as well. 
Let $\bar{R}(s,a)$ be approximated using a linear function as $\bar{R}(s,a;\lambda) = \lambda^{\top}f(s,a)$, where $f(s,a) = [f_1(s,a), \dots, f_M(s,a)]^{\top} \in \mathbb{R}^M, ~M << |\mathcal{S}||\mathcal{A}|$ are the features associated with state action pair $(s,a)$. To obtain the estimate of $\bar{R}(s,a)$ we use the following least square minimization
\begin{equation*}
\label{eqn: min_reward_parameter}
    min_{\lambda} ~ \sum_{s\in \mathcal{S}, a\in \mathcal{A}} \tilde{d}_{\theta}(s,a) [ \bar{R}(s,a) - \bar{R}(s,a; \lambda)]^2,
\end{equation*}
where $\bar{R}(s,a) := \frac{1}{n}\sum_{i\in {N}} R^i(s,a)$, and $\tilde{d}_{\theta}(s,a) = d_{\theta}(s) \cdot \pi_{\theta}(s,a)$.
The above optimization problem can be \textit{equivalently characterized} as
\begin{equation*}
min_{\lambda} ~ \sum_{i\in {N}} \sum_{s\in \mathcal{S}, a\in \mathcal{A}} \tilde{d}_{\theta}(s,a) [ {R}^i(s,a) - \bar{R}(s,a; \lambda) ]^2,
\end{equation*}
as both the optimization problems have the same stationary points. For more details see Appendix \ref{app: derivative_reward_optimizations}. 

Taking first order derivative with respect to $\lambda$ implies that we should also do the following as a part of critic update:
\begin{equation}
    \tilde{\lambda}^i_t = \lambda^i_t + \beta_{v,t} \cdot [r^i_{t+1} - \bar{R}_t(\lambda^i_t)] \cdot \nabla_{\lambda}\bar{R}_t(\lambda^i_t); ~~
    \lambda^i_{t+1} = \sum_{j\in {N}} c_t(i,j) \tilde{\lambda}^j_t,
\label{eqn: marl_lmbda_update}
\end{equation}
where $\bar{R}_t(\lambda^i_t)$ is the linear function approximation of the global reward $\bar{R}_t(s_t,a_t)$ by agent $i\in N$ at time $t$, i.e., $\bar{R}_t(\lambda^i_t) = \lambda_t^{i^{\top}} f(s_t,a_t)$.
The TD error with parameterized reward, $\bar{R}_t(\cdot)$ is given by $\tilde{\delta}^i_t := \bar{R}_{t}(\lambda^i_t) - \mu^i_t + V_{t+1}(v^i_t) - V_{t}(v^i_t)$. 
\textcolor{black}{Note that each agent $i\in N$ knows its local reward function $r^i_{t+1}(s_t, a_t)$, but at the same time he/she also seeks to get some information about the global reward, $\bar{r}_{t+1}(s_t, a_t)$ because the objective is to maximize the globally averaged reward function. Therefore, in above Equation (\ref{eqn: marl_lmbda_update}), each agent $i\in N$ uses  $\bar{R}_t(\lambda^i_t)$ as an estimate of the global reward function.} Each agent $i\in N$ then updates the policy/actor parameters as 
%\textcolor{black}{it seeks because ...? }
\begin{eqnarray}
    \theta^i_{t+1} &=& \theta^i_t + \beta_{\theta,t} \cdot \tilde{\delta}^i_t \cdot \psi^i_t,
\label{eqn: marl_actor_update}
\end{eqnarray}
where $\beta_{\theta,t}>0$ is the actor step size. Note that we have used $\tilde{\delta}^i_t \cdot \psi^i_t$ instead of $\nabla_{\theta^i} J(\theta)$ in the actor update. However, $\tilde{\delta}^i_t  \cdot \psi^i_t$ may not be an unbiased estimate of the gradient of objective function $\nabla_{\theta^i} J(\theta)$. This follows from the policy gradient Theorem for $\nabla_{\theta^i} J(\theta)$; consider the following difference 
\begin{eqnarray*}
    & & \mathbb{E}_{s_t\sim d_{\theta}, a_t \sim \pi_{\theta}}[\tilde{\delta}^i_{t} \cdot \psi^i_{t}] - \nabla_{\theta^i} J(\theta) 
    \\
    &=&  \mathbb{E}_{s_t\sim d_{\theta}, a_t \sim \pi_{\theta}}[(\bar{R}_{t}(\lambda^i_t) - \mu^i_t + V_{t+1}(v^i_t) - V_{t}(v^i_t))\cdot \psi^i_t] - \mathbb{E}_{s_t\sim d_{\theta}, a_t \sim \pi_{\theta}}[ (\bar{R}(s_t,a_t) - \mu^i_t  +V_{t+1}(v^i_t) - V(s_t)) \cdot \psi^i_t]
    \\
    &=& \mathbb{E}_{s_t\sim d_{\theta}, a_t \sim \pi_{\theta}}[(\bar{R}_{t}(\lambda^i_t) - \bar{R}(s_t,a_t)) \cdot \psi^i_{t} ] + \mathbb{E}_{s_t \sim d_{\theta}}[(V_{\theta}(s_t) - V_{t}(v^i_t)) \cdot  \psi^i_{t} ]
    % \\
    % &=& \mathbb{E}_{s_t\sim d_{\theta}, a_t \sim \pi_{\theta}}[(\bar{R}_{t}(\lambda^i_t) - \bar{R}(s_t,a_t)) \psi^i_{t,\theta}] + \mathbb{E}_{s_t \sim d_{\theta}}[(V_{\theta}(s_t) - \varphi_t^{\top} v_{\theta})  \psi^i_{t,\theta}]
\end{eqnarray*}
This implies,
% \begin{equation}
%     \mathbb{E}[\tilde{\delta}^i_{t,\theta} \psi^i_{t,\theta}] = \nabla_{\theta^i} J(\theta) + \underbrace{\mathbb{E}_{s_t\sim d_{\theta}, a_t \sim \pi_{\theta}}[(\bar{R}_{t}(\lambda^i_t) - \bar{R}(s_t,a_t)) \psi^i_{t,\theta}] + \mathbb{E}_{s_t \sim d_{\theta}}[(V_{\theta}(s_t) - V_{t}(v^i_t)))  \psi^i_{t,\theta}]}_{bias}
% \end{equation}
\begin{equation}
\label{eqn: biased_estimate_J}
    \mathbb{E}_{s_t\sim d_{\theta}, a_t \sim \pi_{\theta}}[\tilde{\delta}^i_t \cdot \psi^i_t ] = \nabla_{\theta^i} J(\theta) + b ~~ where,
\end{equation}
$b = \mathbb{E}_{s_t\sim d_{\theta}, a_t \sim \pi_{\theta}}[(\bar{R}_{t}(\lambda^i_t) - \bar{R}(s_t,a_t)) \cdot \psi^i_{t} ] + \mathbb{E}_{s_t \sim d_{\theta}}[(V_{\theta}(s_t) - V_{t}(v^i_t)) \cdot  \psi^i_{t} ] $ is the bias term. \textcolor{black}{The bias captures the sum of the expected linear approximation errors in the reward and value functions. If these approximation errors are small, the  convergence point of the ODE  corresponding to the actor update (as given in Section \ref{sec: Convergence_analysis}) is close to the local optima of $J(\theta)$. %\textcolor{black}{
In fact, in Section \ref{sec: Convergence_analysis}, we show that the actor parameters converge to asymptotically stable equilibrium set of ODEs corresponding to the actor updates, hence possibly nullifying the bias.}
% Therefore, while doing the asymptotic analysis we ignore the above bias term, and hence $\mathbb{E}_{s_t\sim d_{\theta}, a_t \sim \pi_{\theta}}[\tilde{\delta}^i_t \cdot \psi^i_t] \approx \nabla_{\theta^i} J(\theta)$.
To prove the convergence of actor-critic algorithm we require the following conditions on the step sizes $\beta_{v,t},\beta_{\theta,t}$
\begin{equation}
\label{eqn: robbins_monro_conds}
(a)~\sum_{t} \beta_{v,t} = \sum_{t} \beta_{\theta,t} = \infty
; ~~~(b)~\sum_t \left(\beta_{v,t}^2 + \beta_{\theta,t}^2 \right) <\infty,
\end{equation}
moreover, $\beta_{\theta,t} = o(\beta_{v,t})$, and $lim_{t} \frac{\beta_{v,t+1}}{\beta_{v,t}} = 1$, i.e., critic update is made at the faster time scale than the actor update. Condition in (a) ensures that the discrete time steps $\beta_{v,t}, \beta_{\theta,t}$ used in the critic and actor steps do cover the entire time axis while retaining $\beta_{v,t}, \beta_{\theta,t} \rightarrow 0$. We also require the error due to the estimates used in critic and the actor updates are asymptotically negligible almost surely. So, condition in (b) asymptotically suppresses the variance in the estimates \cite{borkar2009book}; see \cite{thoppe2019concentration} for some recent developments that do away with this requirement.

The multi-agent actor-critic scheme given in the MAAC algorithm uses standard (or vanilla) gradients. However, they are most useful for reward functions that have single optima and whose gradients are isotropic in magnitude for any direction away from its optimum \cite{amari1998natural}. None of these properties are valid in typical reinforcement learning environments. Apart from this, the performance of standard gradient-based reinforcement learning algorithms depends on the coordinate system used to define the objective function. It is one of the most significant drawbacks of the standard gradient \cite{kakade2001natural,bagnell2003covariant}.

Moreover, in many applications such as robotics, the state space contains angles, so the state space has manifolds (curvatures). The objective function will then be defined in that curved space, making the policy gradients methods inefficient. We thus require a method that incorporates the knowledge about curvature of the space into the gradient. The natural gradients are the most ``natural" choices in such cases. 
The following section describes the natural gradients and uses them along with the policy gradient theorem in the multi-agent setup. 
% \textcolor{black}{we also  investigate the relations among these algorithms and their effect on the quality of the local optima attained by them}.

\subsection{Natural gradients and the Fisher information matrix}
\label{subsec: natural_gradients_fisher}
Recall, the objective function, $J(\cdot)$ given in Equation (\ref{eqn: objective}) is parameterized by $\theta$. For the single agent actor-critic methods involving natural gradients, $\widetilde{\nabla} J(\theta)$, authors in \cite{peters2003reinforcement,bhatnagar2009natural} have defined it via Fisher information matrix, $G(\theta)$ and standard gradients, $\nabla J(\theta)$ as
\begin{equation}
\label{eqn: natral_gradient_new}
    \widetilde{\nabla}_{\theta} J(\theta) = G(\theta)^{-1}  \nabla_{\theta} J(\theta),
\end{equation}
where  $G(\theta)$ is a positive definite matrix defined as 
\begin{equation}
\label{eqn: fisher_defn}
    G(\theta) := \mathbb{E}_{s \sim d_{\theta}, a \sim \pi_{\theta}} [\nabla_{\theta} \log \pi_{\theta}(s,a)\nabla_{\theta} \log \pi_{\theta}(s,a)^{\top}].
\end{equation}
The above Fisher information matrix is the covariance of the score function. It can also be interpreted  via KL divergence\footnote{\url{https://towardsdatascience.com/natural-gradient-ce454b3dcdfa}} between the policy $\pi(\cdot,\cdot)$ parameterize at $\theta$ and $\theta + \Delta \theta$ as below \cite{martens2020new,ratliff2013information},
\begin{equation}
\label{eqn: approx_kl}
KL(\pi_{\theta}(\cdot,\cdot)|| \pi_{\theta + \Delta \theta}(\cdot,\cdot)) \approx \frac{1}{2} \Delta \theta^{\top} \cdot G(\theta) \cdot \Delta \theta. 
\end{equation}
The above expression is obtained from the second-order Taylor expansion of $\log \pi_{\theta + \Delta \theta}(s, a)$, and using the fact that the sum of the probabilities is one. In above, the right-hand term is a quadratic involving positive definite matrix $G(\theta)$, and hence $G(\theta)$ approximately captures the curvature of the KL divergence between policy distributions at $\theta$ and $\theta + \Delta \theta$.

We now relate the Fisher information matrix, $G(\theta)$ to the objective function, $J(\theta)$. 

\begin{lemma}
\label{lemma: kl_grad_J}
The gradient of the KL divergence between two consecutive policies is approximately proportional to the gradient of the objective function, i.e.,
% up to $\rho_t$ scaling, i.e., a scaling of $- \frac{1}{\rho_t}$
\begin{equation*}
% \label{eqn: grad_kl_relation_grad_j}
\nabla KL(\pi_{\theta_t}(\cdot, \cdot) || \pi_{\theta_t + \Delta \theta_t}(\cdot, \cdot)) \propto  \nabla J(\theta_t).
\end{equation*}
\end{lemma}
% \textcolor{black}{Lagrange multiplier $- \frac{1}{\rho_t}$ sign??   }
\begin{proof}
From Equation (\ref{eqn: approx_kl}), the KL divergence is a function of the Fisher information matrix and delta change in the policy parameters.
% Consider maximizing  the objective function keeping the KL divergence a constant. To this end, 
We find the optimal step-size $\Delta \theta_t^{\star}$ via the following optimization problem
\begin{equation*}
\begin{aligned}
\Delta \theta_t^{\star} = & ~argmax_{\Delta \theta_t}~ J(\theta_t + \Delta \theta_t)
\\
& s.t. ~~KL(\pi_{\theta_t}(\cdot,\cdot)|| \pi_{\theta_t + \Delta \theta_t}(\cdot,\cdot)) = c.
\end{aligned}
\end{equation*}
% Note that in above optimization problem, the KL divergence between $\theta$ and $\theta + \Delta \theta$ is restricted to a constant $c$. This is motivated by the fact that at each time step, the direction/curvature is fixed in any gradient descent methods and then the optimal step size is obtained to find the next sub-optimal point.
% \textcolor{black}{does `less than equal to' constraint lead to the same result??}
Writing the Lagrange $\mathcal{L}(\theta_t + \Delta \theta_t; \rho_t)$ (where $\rho_t$ is the Lagrangian multiplier)
%\sout{\textcolor{black}{$\gamma_t$ is used in Kushner-Clark lemma in App. may be we can use $\rho_t$ or $a_t$ or $c_t$??}} 
of the above optimization problem and using the first order Taylor approximation along with the approximation of the KL divergence as given in Equation (\ref{eqn: approx_kl}), we have
%\textcolor{black}{need to use another multiplier other than $\mu_t$}
\begin{eqnarray*}
	\Delta \theta_t^{\star}&=& argmax_{\Delta \theta_t}~ J(\theta_t + \Delta \theta_t)  + \rho_t (KL(\pi_{\theta_t}(\cdot,\cdot)|| \pi_{\theta_t + \Delta \theta_t}(\cdot,\cdot)) - c),
	\\
	&\approx& argmax_{\Delta \theta_t}~ J(\theta_t) + \Delta \theta_t^{\top} \nabla_{\theta} J(\theta_t)  + \frac{1}{2} \cdot \rho_t \cdot \Delta \theta_t^{\top} \cdot G(\theta_t) \cdot \Delta \theta_t   - \rho_t c.
\end{eqnarray*}
Setting the derivative (w.r.t. $\Delta \theta_t$) of above Lagrangian to zero, we have
\begin{equation}
\label{eqn: stationary_eqn}
\nabla_{\theta} J(\theta_t)  + \rho_t \cdot \Delta \theta_t^{\star{^\top}} \cdot G(\theta_t)  = 0~\implies~ \Delta \theta_t^{\star} = -\frac{1}{\rho_t} G(\theta_t)^{-1} \nabla J(\theta_t),
\end{equation}
i.e., upto the factor of $-\frac{1}{\rho_t}$, we get an optimal step-size in terms of the standard gradients and the Fisher information matrix at point $\theta_t$. Moreover, from Equations (\ref{eqn: approx_kl}) and (\ref{eqn: stationary_eqn}) we have, 
\begin{equation}
\label{eqn: KL_two_eqns}
\nabla KL(\pi_{\theta_t}(\cdot,\cdot)|| \pi_{\theta_t + \Delta \theta_t}(\cdot,\cdot)) \approx G(\theta_t) \Delta \theta_t^{\star}; ~~~ and ~~~G(\theta_t) \Delta \theta_t^{\star}   =  -\frac{1}{\rho_t} \nabla J(\theta_t).
\end{equation}
and hence, 
\begin{equation}
\label{eqn: grad_kl_relation_grad_j}
\nabla KL(\pi_{\theta_t}(\cdot,\cdot)|| \pi_{\theta_t + \Delta \theta_t}(\cdot,\cdot)) \approx - \frac{1}{\rho_t} \nabla J(\theta_t) .
\end{equation}
This ends the proof.
\end{proof}
The above lemma relates the gradient of the objective function to the gradient of KL divergence between the policies separated by $\Delta \theta_t$. 
%\textcolor{black}{It is easy to see that the gradient of the KL divergence depends on the policy used.} 
%It is approximately the same as the gradient of the objective function up to $\rho_t$ scaling. 
% However, the gradient of the objective function depends on the single-stage reward of the `policy' used. So, gradient of KL divergence and the gradient of objective function both depends on the policy used, and hence the equality is justified.} \textcolor{black}{we can avoid the last 2 sentences, `however ,,,'. it is for the  response sheet.}
\textcolor{black}{The above equation provides a valuable observation because we can adjust the updates (of actor parameter $\theta_t$) just by moving in the prediction space of the parameterized policy distributions. Thus, those MAN algorithms discussed later that rely on Fisher information matrix $G(\cdot)$ \emph{implicitly} use the above representation for $\nabla J(\cdot)$.
%\textcolor{black}{}. 
This again justifies the natural gradients as given in Equation (\ref{eqn: natral_gradient_new}). We recall these aspects in Section \ref{subsec: kl_boltzmann} for the Boltzmann policies.}

\subsection{Multi-agent natural policy gradient theorem and rank-one update of $G_{t+1}^{i^{-1}}$}

In this section, we provide the details of natural policy gradient methods and the Fisher information matrix in the multi-agent setup. Similar to Equation (\ref{eqn: natral_gradient_new}), in multi-agent setup the natural gradient of the objective function is
\begin{equation}
\label{eqn: g_theta_each_agent}
    \widetilde{\nabla}_{\theta^i} J(\theta) = G(\theta^i)^{-1}  \nabla_{\theta^i} J(\theta),~ \forall~ i\in N,
\end{equation}
where $G(\theta^i) := \mathbb{E}_{s \sim d_{\theta}, a \sim \pi_{\theta}} [\nabla_{\theta^i} \log \pi^i_{\theta^i}(s,a^i)\nabla_{\theta^i} \log \pi^i_{\theta^i}(s,a^i)^{\top}]$
is again a positive definite matrix for each agent $i\in N$. We now present the policy gradient theorem for multi-agent setup involving the natural gradients.

\begin{theorem}[Policy gradient theorem for MARL with natural gradients]
\label{thm: pgt_with_natural_gradients}
Under assumption X. \ref{ass: regularity}, the natural gradient of $J(\theta)$ with respect to $\theta^i$ for each $i\in N$ is given by 
\begin{equation*}
\label{eqn: pgt_natural_grad_J}
\begin{aligned}
    \widetilde{\nabla}_{\theta^i}J(\theta) &=& G(\theta^i)^{-1}\mathbb{E}_{s\sim d_{\theta}, a\sim \pi_{\theta}}[\nabla_{\theta^i} \log \pi^{i}_{\theta^i}(s,a^i) \cdot  A_{\theta}(s,a)]
    \\
    &=& G(\theta^i)^{-1} \mathbb{E}_{s\sim d_{\theta}, a\sim \pi_{\theta}}[\nabla_{\theta^i} \log \pi^{i}_{\theta^i}(s,a^i) \cdot A^i_{\theta}(s,a)].
\end{aligned}
\end{equation*}
\end{theorem}

\begin{proof}
\label{proof: pgt_natural}
The proof follows from the multi-agent policy gradient Theorem \ref{thm: pgt} and using Equation (\ref{eqn: g_theta_each_agent}), i.e.,  $\widetilde{\nabla}_{\theta^i} J(\theta) = G(\theta^i)^{-1}  \nabla_{\theta^i} J(\theta),~~ \forall i\in N$. 
\end{proof}

It is well known that inverting a Fisher information matrix is computationally heavy \cite{kakade2001natural,peters2008natural}. 
Whereas, in our natural gradient-based multi-agent actor-critic methods we require the inverse of Fisher information matrix  $G(\theta^i)^{-1},~~ \forall~ i\in N$. 
% Thus, the Fisher information matrix inverse has to be obtained at a faster time scale, i.e., before the actor step.
To this end, we derive the procedure for recursively estimating the $G(\theta^i)^{-1}$ for each agent $i\in N$ at the faster time scale \cite{bhatnagar2009natural}. %\textcolor{black}{rank-one update?} 
Let $G^{i^{-1}}_{t+1}, t\geq 0$ be the $t$-th estimate of $G(\theta^i)^{-1}$.
% We will show later that $G^{i^{-1}}_{t+1} \rightarrow G(\theta^i)^{-1}$ as $t\rightarrow \infty$ with probability one.
Consider the sample averages at time $t \geq 0$,
\begin{equation*}
    G_{t+1}^i = \frac{1}{t+1} \sum_{l=0}^{t} \psi^i_l\psi_l^{i^{\top}},
\end{equation*}
where $\psi^i_l = \psi^i_l(s_l, a^i_l)$ is the score function of agent $i\in N$ at time $l$. Thus, $G^i_{t+1}$ will be recursively obtained as 
\begin{equation*}
    G^i_{t+1}  =\left(1-\frac{1}{t+1}\right) G^i_{t} + \frac{1}{t+1}\psi^i_t \psi_t^{i^{\top}}.
\end{equation*}
More generally, one can consider the following recursion
\begin{equation}
\label{eqn: fisher_iterative_update}
   G^i_{t+1}  =(1-\beta_{v,t}) G^i_{t} + \beta_{v,t}\psi^i_t \psi_t^{i^{\top}},
\end{equation}   
where $\beta_{v,t} >0$ is the step size as earlier.
Using the idea of stochastic convergence, if $\theta^i$ is held constant, $G^i_{t+1}$ converge to $G(\theta^i)$ with probability one. For all $i\in N$, we write the recursion for the Fisher information matrix inverse, $G_{t+1}^{i^{-1}}$ using Sherman-Morrison matrix inversion \cite{sherman1950adjustment} (see also  \cite{bhatnagar2009natural}) as follows
\begin{equation}
\label{eqn: sherman-morrison}
    G_{t+1}^{i^{-1}} = \frac{1}{1-\beta_{v,t}}\left[ G_{t}^{i^{-1}} - \beta_{v,t} \frac{(G_{t}^{i^{-1}}\psi^i_t)(G_{t}^{i^{-1}}\psi_t^{i})^{\top}}{1-\beta_{v,t} + \beta_{v,t}\psi_t^{i^{\top}}G_{t}^{i^{-1}}\psi_t^i} \right].
\end{equation}
The Sherman-Morrison update is done at a faster time scale $\beta_{v,t}>0$, to ensure that Fisher information inverse estimates are available before actor update.

% \textcolor{black}{a cross-ref to some insights, like advantages of using natural gradients and FI matrix $G(\theta)$ in $\theta$ parameter updates of objective fn $J(\theta)$ is needed here?}

The following section provides three multi-agent natural actor-critic (MAN) RL algorithms involving consensus matrices. Moreover, we will also investigate the relations among these algorithms and their effect on the quality of the local optima they attain.

\section{Multi-agent natural actor-critic (MAN) algorithms}
\label{sec: MANAC}

This section provides three multi-agent natural actor-critic (MAN) reinforcement learning algorithms. Two of the three MAN algorithms explicitly use the Fisher information matrix inverse, whereas one uses the linear function approximation of the advantage parameters.
% and never explicitly estimates the Fisher information matrix inverse.

\subsection{FI-MAN: Fisher information based multi-agent natural actor-critic algorithm}
\label{subsec: manac1}

Our first multi-agent natural actor-critic algorithm uses the fact that natural gradients can be obtained via the Fisher information matrix and the standard gradients as given in Equation (\ref{eqn: g_theta_each_agent}).  The updates of  objective function estimate, critic, and the rewards parameters in FI-MAN algorithm are the same as given in Equations (\ref{eqn: marl_obj_updates}), (\ref{eqn: marl_critic_updates}), and (\ref{eqn: marl_lmbda_update}), respectively. The major difference between  the MAAC and the FI-MAN algorithm is in the actor update. FI-MAN algorithm uses the following actor update
\begin{equation}
    \label{eqn: actor_update_manac1}
    \theta^i_{t+1} \leftarrow \theta^i_t + \beta_{\theta,t} \cdot G_t^{i^{-1}} \cdot \tilde{\delta}^i_t \cdot \psi^i_t,~~ \forall~ i\in N,
\end{equation}
where $\beta_{v,t}$ and $\beta_{\theta,t}$ are the critic and the actor step-sizes respectively. The pseudo code of the algorithm is given in FI-MAN algorithm.

% he online implementation of FI-MAN actor-critic with linear function approximation has $\mathcal{O}(n+M+L+m_i)$ as memory complexity for each agent $i\in N$. This is a significant reduction from the tabular case when $n$ is large.

\begin{manac1}[!ht]
% 	\SetAlgoNoLine
	\KwIn{Initial values of $\mu^i_0, \tilde{\mu}^i_0,v^i_0, \tilde{v}^i_0, \lambda^i_0, \tilde{\lambda}^i_0, \theta^i_0, G_0^{i^{-1}},~ \forall~ i\in {N}$; initial state $s_0$; stepsizes $\{\beta_{v,t}\}_{t\geq 0}, \{\beta_{\theta,t}\}_{t\geq 0}$.
	\\
	Each agent $i$ implements $a^i_0 \sim \pi_{\theta^i_0}(s_0,\cdot)$.
	\\
	Initialize the step counter $t\leftarrow 0$.
	}
	\Repeat{Convergence}
    {\For {all $i\in N$}
    {Observe state $s_{t+1}$, and reward $r^i_{t+1}$. \\
    Update: $\tilde{\mu}^i_t \leftarrow (1-\beta_{v,t}) \cdot \mu^i_t + \beta_{v,t} \cdot r^i_{t+1}$.
    \\
    $\tilde{\lambda}^i_t \leftarrow \lambda^i_t + \beta_{v,t} \cdot [r^i_{t+1} - \bar{R}_t(\lambda^i_t)]  \cdot \nabla_{\lambda}\bar{R}_t(\lambda^i_t)$,~ 
    where $\bar{R}_t(\lambda^i_t) = \lambda^{i^{\top}}_t f(s_t,a_t)$.
    \\
    Update: $\delta^i_t \leftarrow r^i_{t+1} - \mu^i_t + V_{t+1}(v^i_t) - V_{t}(v^i_t)$, 
    ~where $V_{t+1}(v^i_t) = v^{i^{\top}}_t \varphi(s_{t+1})$.
    \\
    \textbf{Critic Step:} $\tilde{v}^i_t \leftarrow v^i_t + \beta_{v,t} \cdot \delta^i_t \cdot \nabla_v V_t(v^i_t)$, ~ 
    \\
    Update: $\tilde{\delta}^i_t \leftarrow \bar{R}_{t}(\lambda^i_t) - \mu^i_t + V_{t+1}(v^i_t) - V_{t}(v^i_t)$;~ $\psi^i_t \leftarrow \nabla_{\theta^i} \log \pi^i_{\theta^i_t}(s_t,a^i_t)$.
    \\
    \textbf{Actor Step:} $\theta^i_{t+1} \leftarrow \theta^i_t + \beta_{\theta,t} \cdot G_t^{i^{-1}}  \cdot \tilde{\delta}^i_t \cdot \psi^i_t$.
    \\
    Send $\tilde{\mu}^i_t, \tilde{\lambda}^i_t, \tilde{v}^i_t$ to the neighbors over $\mathcal{G}_t$.}
    {\For {all $i\in {N}$}
    {\textbf{Consensus Update:}~~$\mu^i_{t+1} \leftarrow  \sum_{j\in {N}} c_t(i,j) \tilde{\mu}^j_t$;
    \\
    $\lambda^i_{t+1} \leftarrow \sum_{j\in {N}} c_t(i,j) \tilde{\lambda}^j_t;~~v^i_{t+1} \leftarrow \sum_{j\in {N}}c_t(i,j)\tilde{v}^j_t.$
    \\
    \textbf{Fisher Update:} 
    $G_{t+1}^{i^{-1}} \leftarrow \frac{1}{1-\beta_{v,t}}\left[ G_{t}^{i^{-1}} - \beta_{v,t} \frac{(G_{t}^{i^{-1}}\psi^i_t)(G_{t}^{i^{-1}}\psi^i_t)^{\top}}{1-\beta_{v,t} + \beta_{v,t}\psi_t^{i^{\top}}G_{t}^{i^{-1}}\psi_t^i} \right].$}}
    Update: $t\leftarrow t+1$.}
\label{alg: MANAC-fisher}
\caption{Fisher information based multi-agent natural actor critic} 
\end{manac1}

Note that the FI-MAN algorithm explicitly uses $G^{i^{-1}}_t$ in the actor update. Though the Fisher information inverse matrix is updated according to the Sherman-Morrison inverse at a faster time scale, it may be better to avoid explicit use of the Fisher inverse in the actor update. To this end, we use the linear function approximation of the advantage function. This leads to the AP-MAN  algorithm, i.e., advantage parameters based multi-agent natural actor-critic algorithm.  

\subsection{AP-MAN: Advantage parameters based multi-agent natural actor critic algorithm %\textcolor{black}{is this AP-MAN using $G(\cdot)$; is it to be called *MAN?}
}
\label{subsec: manac2}

Consider the local advantage function $A^i(s,a^{i}): \mathcal{S}\times \mathcal{A} \rightarrow \mathbb{R}$ for each agent $i\in N$.  Let the local advantage function $A^i(s,a^i)$ be linearly approximated as $A^i(s,a^i; w^i) := w^{i^{\top}} \psi^{i}(s,a^i)$, where $\psi^i(s,a^i) = \nabla_{\theta^i} \log \pi^i_{\theta^i}(s,a^i)$ are the compatible features, and $w^i \in \mathbb{R}^{m_i}$ are the advantage function parameters. Recall, the same $\psi^i(s,a^i)$ was used to represent the score function in the policy gradient theorem, Theorem \ref{thm: pgt}. However, it also serves as the compatible feature while approximating the advantage function as it satisfy the compatibility condition in the policy gradient theorem with linear function approximations (Theorem 2 \cite{sutton1999policy}). The compatibility condition as given in \cite{sutton1999policy} is for single agent however, we are using it explicitly for each agent $i\in N$. \textcolor{black}{Whenever there is no confusion, we write $\psi^{i}$ instead of $\psi^i(s,a^i)$, to save space.}
We can tune $w^i$ in such a way that the estimate of least squared error in linear function approximation of advantage function is minimized, i.e., 
\begin{equation}
    \mathcal{E}^{\pi_{\theta}}(w^i) = \frac{1}{2} \sum_{s\in \mathcal{S}, a^i\in \mathcal{A}^i} \tilde{d}_{\theta}(s,a^i)[w^{i^{\top}}\psi^i(s,a^i) - A^i(s,a^i)]^2,
\label{eqn: error_in_advantage_parameter}
\end{equation}
is minimized. Here $\tilde{d}_{\theta}(s,a^i) = d_{\theta}(s)\cdot \pi^i_{\theta^i}(s,a^i)$  as defined earlier. Taking the derivative of Equation (\ref{eqn: error_in_advantage_parameter}), we have $\nabla_{w^i} \mathcal{E}^{\pi_{\theta}}(w^i) = \sum_{s\in \mathcal{S}, a^i\in \mathcal{A}^i} \tilde{d}_{\theta}(s,a^i)[w^{i^{\top}}\psi^i - A^i(s,a^i)] \psi^i$. Noting that parameterized TD error $\tilde{\delta}^i_t$ is an unbiased estimate of the local advantage function $A^i(s,a^i)$, we will use the following estimate of $\nabla_{w^i} \mathcal{E}^{\pi_{\theta}}(w^i)$,
\begin{equation*}
    \widehat{\nabla_{w^i}} \mathcal{E}^{\pi_{\theta}}(w^i_t) =   \psi^i_t \psi^{i^{\top}}_t w^i_t - \tilde{\delta}^i_t \psi^i_t.
\end{equation*}
Hence the update of advantage parameter $w^i$ in the AP-MAN algorithm is 
\begin{eqnarray}
\label{eqn: adv_parameter_update_manac2}
    w^i_{t+1} &=& w^i_t - \beta_{v,t} \widehat{\nabla_{w^i}} \mathcal{E}^{\pi_{\theta}}(w^i_t) \nonumber
    % \\
    % &=& w^i_t - \beta_{v,t} (\psi^i_t \psi^{i^{\top}}_t w^i_t - \tilde{\delta}^i_t \psi^i_t)
    \\
    &=& (I - \beta_{v,t} \psi^i_t \psi^{i^{\top}}_t)w^i_t + \beta_{v,t} \tilde{\delta}^i_t \psi^i_t.
\end{eqnarray}
The updates of the objective function estimate, critic, and reward parameters in the AP-MAN  algorithm are the same as given in Equations (\ref{eqn: marl_obj_updates}), (\ref{eqn: marl_critic_updates}), and (\ref{eqn: marl_lmbda_update}), respectively. Additionally, in the critic step we update the advantage parameters as given in Equation (\ref{eqn: adv_parameter_update_manac2}). \textcolor{black}{For single agent RL with natural gradients, \cite{peters2008natural,peters2003reinforcement} show that $\widetilde{\nabla}_{\theta} J(\theta) = w$. In MARL with natural gradient, we separately verified and hence use $\widetilde{\nabla}_{\theta^i} J(\theta) = w^i$ for each agent $i\in N$ in the actor update of AP-MAN algorithm.} The AP-MAN actor-critic algorithm thus uses the following actor update 
\begin{equation}
    \label{eqn: actor_update_manac2}
    \theta^i_{t+1} \leftarrow \theta^i_t + \beta_{\theta,t} \cdot w^i_{t+1}.
\end{equation}
The algorithm's pseudo-code involving advantage parameters is given in the AP-MAN algorithm. 
% AP-MAN actor-critic algorithm has one benefit over other algorithms. It never explicitly stores an estimate of the Fisher information matrix inverse, and as a result, it requires less computation. 
% The online implementation of AP-MAN actor-critic with linear function approximation has $\mathcal{O}(n+M+L+K+m_i)$ as memory complexity for each agent $i\in N$. This is again a significant reduction from the tabular case when $n$ is large.

\begin{manac2}[!ht]
\caption{Advantage parameters based multi-agent natural actor critic} 
% 	\SetAlgoNoLine
	\KwIn{Initial values of $\mu^i_0, \tilde{\mu}^i_0,v^i_0, \tilde{v}^i_0, \lambda^i_0, \tilde{\lambda}^i_0, \theta^i_0, w^i_0, ~ \forall i\in {N}$; initial state $s_0$; stepsizes $\{\beta_{v,t}\}_{t\geq 0}, \{\beta_{\theta,t}\}_{t\geq 0}.$\\
	Each agent $i$ implements $a^i_0 \sim \pi_{\theta^i_0}(s_0,\cdot).$
	\\
	Initialize the step counter $t\leftarrow 0.$
	}
	\Repeat{Convergence}
    {\For {all $i\in {N}$}
    {Observe state $s_{t+1}$, and reward $r^i_{t+1}.$ \\
    Update: $\tilde{\mu}^i_t \leftarrow (1-\beta_{v,t}) \cdot \mu^i_t + \beta_{v,t} \cdot r^i_{t+1}.$
    \\
    $\tilde{\lambda}^i_t \leftarrow \lambda^i_t + \beta_{v,t} \cdot [r^i_{t+1} - \bar{R}_t(\lambda^i_t)]  \cdot \nabla_{\lambda}\bar{R}_t(\lambda^i_t)$,~ 
    where $\bar{R}_t(\lambda^i_t) = \lambda^{i^{\top}}_t f(s_t,a_t)$.
    \\
    Update: $\delta^i_t \leftarrow r^i_{t+1} - \mu^i_t + V_{t+1}(v^i_t) - V_{t}(v^i_t)$, 
    ~where $V_{t+1}(v^i_t) = v^{i^{\top}}_t \varphi(s_{t+1})$.
    \\
    \textbf{Critic Step:} $\tilde{v}^i_t \leftarrow v^i_t + \beta_{v,t} \cdot \delta^i_t \cdot \nabla_v V_t(v^i_t).$
    \\
    Update: $\tilde{\delta}^i_t \leftarrow \bar{R}_{t}(\lambda^i_t) - \mu^i_t + V_{t+1}(v^i_t) - V_{t}(v^i_t)$;~ $\psi^i_t \leftarrow \nabla_{\theta^i} \log \pi^i_{\theta^i_t}(s_t,a^i_t).$
    \\
    Update: ${w}^i_{t+1} \leftarrow (I-\beta_{v,t} \psi^i_t \psi^{i^{\top}}_t)w^i_t + \beta_{v,t} \tilde{\delta}^i_t \psi^i_t.$
    \\
    \textbf{Actor Step:} $\theta^i_{t+1} \leftarrow \theta^i_t + \beta_{\theta,t} \cdot w^i_{t+1}.$ 
    \\
    Send $\tilde{\mu}^i_t, \tilde{\lambda}^i_t, \tilde{v}^i_t$ to the neighbors over $\mathcal{G}_t$.}
    {\For {all $i\in {N}$}
    {\textbf{Consensus Update:} $\mu^i_{t+1} \leftarrow  \sum_{j\in {N}} c_t(i,j) \tilde{\mu}^j_t$;
    \\
    $\lambda^i_{t+1} \leftarrow \sum_{j\in {N}} c_t(i,j) \tilde{\lambda}^j_t;~~v^i_{t+1} \leftarrow \sum_{j\in {N}}c_t(i,j)\tilde{v}^j_t.$}}
    Update: $t\leftarrow t+1$.}
\label{alg: MANAC-advantage-parameters}
\end{manac2}

\begin{remark}
We want to emphasize that the AP-MAN algorithm does not explicitly use the inverse of Fisher information matrix, $G(\theta^i)^{-1}$ in the actor update (as also in \cite{bhatnagar2009natural}); hence it requires fewer computations. However, it involves the linear function approximation of the advantage function, as given in Equation (\ref{eqn: adv_parameter_update_manac2}), that itself requires $\psi^i_t \psi^{i^{\top}}_t$ which is an  unbiased estimate of the Fisher information matrix (see Equation (\ref{eqn: fisher_defn})). We will see later in Section \ref{subsec: algo_actor_comparison} that the performance of the AP-MAN algorithm is almost the same as the MAAC algorithm. We empirically verify this observation in the computational experiments in Section \ref{sec: experiments}.
\end{remark}

\begin{remark}
The advantage function is a linear combination of $Q_{\theta}(s, a)$ and $V_{\theta}(s)$; therefore, the linear function approximation of the advantage function alone enjoys the benefit of approximating the $Q_{\theta}(s, a)$ or $V_{\theta}(s)$. Moreover, MAAC uses the linear function approximation of $V_{\theta}(s)$; hence, we expect the behavior of AP-MAN to be similar to that of MAAC; this comes out in our computational experiments in \textcolor{black}{Section} \ref{sec: experiments}. 
\end{remark}

The FI-MAN algorithm is based solely on the Fisher information matrix and the AP-MAN algorithm on the advantage function approximation. Our next algorithm, FIAP-MAN algorithm, i.e., Fisher information and advantage parameter based multi-agent natural actor-critic algorithm, combine them in a certain way. We see the benefits of this combination in Sections \ref{subsec: algo_actor_comparison} and \ref{subsubsec: traffic_network_details}. In particular, in Section \ref{subsubsec: traffic_network_details} we illustrate these benefits in 2 arrival distributions in the traffic network congestion model.

\subsection{FIAP-MAN: Fisher information and advantage parameter based multi-agent natural actor-critic algorithm}
\label{subsec: manac3}

Recall in Section \ref{subsec: manac2}, for each agent $i\in N$, the local advantage function $A^i(s,a^{i}): \mathcal{S}\times \mathcal{A} \rightarrow \mathbb{R}$ has linear function approximation $A^i(s,a^i; w^i) = w^{i^{\top}} \psi^{i}(s,a^i)$, where $\psi^i(s,a^i)$ are the compatible features as before, and $\textcolor{black}{w^i \in \mathbb{R}^{m_i}}$ are the advantage function parameters. In AP-MAN algorithm the Fisher inverse $G(\theta^i)^{-1}$ is not estimated explicitly, however in FIAP-MAN algorithm we explicitly estimate $G(\theta^i)^{-1}$ along with the advantage parameters and hence use the following estimate of $\nabla_{w^i} \mathcal{E}^{\pi_{\theta}}(w^i)$,
\begin{equation*}
    \widehat{\nabla_{w^i}} \mathcal{E}^{\pi_{\theta}}(w^i_t) =   G^{i^{-1}}_{t}(\psi^i_t \psi^{i^{\top}}_t w^i_t - \tilde{\delta}^i_t \psi^i_t),~~ \forall~ i\in N.
\end{equation*}
The update of advantage parameters, $w^i$ along with the critic update in the FIAP-MAN algorithm is given by
\begin{eqnarray}
\label{eqn: adv_parameter_update_manac3}
    w^i_{t+1} &=& w^i_t - \beta_{v,t} \widehat{\nabla_{w^i}} \mathcal{E}^{\pi_{\theta}}(w^i_t) \nonumber
    \\
    &=& w^i_t - \beta_{v,t} G^{i^{-1}}_t(\psi^i_t \psi^{i^{\top}}_t w^i_t - \tilde{\delta}^i_t \psi^i_t) \nonumber
    \\
    &=& (1 - \beta_{v,t})w^i_t + \beta_{v,t} G^{i^{-1}}_t \tilde{\delta}^i_t \psi^i_t.
\end{eqnarray}

\begin{remark}
\textcolor{black}{
Note that we take $G^{i^{-1}}_t \psi^i_t \psi^{i^{\top}}_t = I,~ \forall ~i \in N$, though $G_{t+1} = \frac{1}{t+1} \sum_{l=0}^{t} \psi_l {\psi_l}^{\top}$. A similar approximation is also implicitly  made in natural gradient algorithms in \cite{bhatnagar2009natural,bhatnagar2009natural_tr} for single agent RL.
% Therefore,  we use the above update for the advantage parameters in our FIAP-MAN algorithm.
Convergence of FIAP-MAN algorithm with above approximate update in MARL is given in Section \ref{sec: Convergence_analysis}. Later, we use these updates in our computations to demonstrate their superior performance in multiple instances of traffic network (Section \ref{sec: experiments}).}
\end{remark}

The updates of the objective function estimate, critic, and reward parameters in the FIAP-MAN algorithm are the same as given in Equations (\ref{eqn: marl_obj_updates}), (\ref{eqn: marl_critic_updates}), and (\ref{eqn: marl_lmbda_update}) respectively. Similar to the AP-MAN algorithm, the actor update in FIAP-MAN algorithm is
\begin{equation}
    \label{eqn: actor_update_manac3}
    \theta^i_{t+1} \leftarrow \theta^i_t + \beta_{\theta,t} \cdot w^i_{t+1}, ~~\forall~i\in N.
\end{equation}
\textcolor{black}{Again for the same reason as in the AP-MAN algorithm, we use the advantage parameters, $w^i$ instead of $\widetilde{\nabla}_{\theta^i}J(\theta)$, i.e., $\widetilde{\nabla}_{\theta^i}J(\theta) = w^i$}. The algorithm's pseudo code involving the Fisher information matrix inverse and advantage parameters is given in FIAP-MAN algorithm.

\begin{manac3}[h!]
\caption{Fisher information and advantage parameters based multi-agent natural actor-critic} 
% 	\SetAlgoNoLine
	\KwIn{Initial values of $\mu^i_0, \tilde{\mu}^i_0,v^i_0, \tilde{v}^i_0, \lambda^i_0, \tilde{\lambda}^i_0, \theta^i_0, w^i_0, G^{i^{-1}}_0, ~ \forall~ i\in {N}$; initial state $s_0$; stepsizes $\{\beta_{v,t}\}_{t\geq 0}, \{\beta_{\theta,t}\}_{t\geq 0}$.\\
	Each agent $i$ implements $a^i_0 \sim \pi_{\theta^i_0}(s_0,\cdot)$.
	\\
	Initialize the step counter $t\leftarrow 0$.
	}
	\Repeat{Convergence}
    {\For {all $i\in N$}
    {Observe state $s_{t+1}$, and reward $r^i_{t+1}$. \\
    Update: $\tilde{\mu}^i_t \leftarrow (1-\beta_{v,t}) \cdot \mu^i_t + \beta_{v,t} \cdot r^i_{t+1}$.
    \\
    $\tilde{\lambda}^i_t \leftarrow \lambda^i_t + \beta_{v,t} \cdot [r^i_{t+1} - \bar{R}_t(\lambda^i_t)]  \cdot \nabla_{\lambda}\bar{R}_t(\lambda^i_t)$,~ 
    where $\bar{R}_t(\lambda^i_t) = \lambda^{i^{\top}}_t f(s_t,a_t)$.
    \\
    Update: $\delta^i_t \leftarrow r^i_{t+1} - \mu^i_t + V_{t+1}(v^i_t) - V_{t}(v^i_t)$, 
    ~where $V_{t+1}(v^i_t) = v^{i^{\top}}_t \varphi(s_{t+1})$.
    \\
    \textbf{Critic Step:} $\tilde{v}^i_t \leftarrow v^i_t + \beta_{v,t} \cdot \delta^i_t \cdot \nabla_v V_t(v^i_t)$.
    \\
    Update: $\tilde{\delta}^i_t \leftarrow \bar{R}_{t}(\lambda^i_t) - \mu^i_t + V_{t+1}(v^i_t) - V_{t}(v^i_t)$;~ $\psi^i_t \leftarrow \nabla_{\theta^i} \log \pi^i_{\theta^i_t}(s_t,a^i_t)$.
    \\
    Update: ${w}^i_{t+1} \leftarrow (1-\beta_{v,t}) w^i_t + \beta_{v,t} G^{i^{-1}}_{t}\tilde{\delta}^i_t \psi^i_t$.
    \\
    \textbf{Actor Step:} $\theta^i_{t+1} \leftarrow \theta^i_t + \beta_{\theta,t} \cdot w^i_{t+1}$.
    \\
    Send $\tilde{\mu}^i_t, \tilde{\lambda}^i_t, \tilde{v}^i_t$ to the neighbors over $\mathcal{G}_t$.}
    {\For {all $i\in {N}$} 
    {\textbf{Consensus Update:} $\mu^i_{t+1} \leftarrow  \sum_{j\in N} c_t(i,j) \tilde{\mu}^j_t$;
    \\
    $\lambda^i_{t+1} \leftarrow \sum_{j\in N} c_t(i,j) \tilde{\lambda}^j_t;~~v^i_{t+1} \leftarrow \sum_{j\in N}c_t(i,j)\tilde{v}^j_t $.
    \\
    \textbf{Fisher Update:} 
    % $G^i_{t+1}  \leftarrow (1-\beta_{v,t}) G^i_{t} + \beta_{v,t}\psi^i_{t} \psi_{t}^{i^{\top}} $
    % \\
    $G_{t+1}^{i^{-1}} \leftarrow \frac{1}{1-\beta_{v,t}}\left[ G_{t}^{i^{-1}} - \beta_{v,t} \frac{(G_{t}^{i^{-1}}\psi^i_t)(G_{t}^{i^{-1}}\psi^i_t)^{\top}}{1-\beta_{v,t} + \beta_{v,t}\psi_t^{i^{\top}}G_{t}^{i^{-1}}\psi_t^i} \right]$.}}
    Update: $t\leftarrow t+1$.}
\label{alg: MANAC-advantage-parameters-fisher}
\end{manac3}

% The online implementation of FIAP-MAN with linear function approximation also has $\mathcal{O}(n+M+L+K+m_i)$ as memory complexity for each agent $i\in N$. This is again a significant reduction from the tabular case when $n$ is large.

% \textcolor{black}{this complexity is suddenly appearing now; it was written earlier while discussing why FIAP-MAN is not doing well. possible to say similar complexity for others also? }

% \begin{remark}
% \label{rem: second-oreder-method}
% \textcolor{black}{i think we should remove this remark as this is expanded a lot in new Sec 3.5 now??}Though the Fisher information matrix captures the KL divergence curvature between the policies parameterized at consecutive iterates, the MAN algorithms are not regarded as second-order methods. It is because the Fisher information matrix is not the Hessian of the `objective function' $J(\cdot)$. 
% % However, being the curvature of the KL divergence, we observe a period to period improvement in the objective function if the maximum eigenvalue of the Fisher matrix is less than 1, which holds in our experiments in Section \ref{sec: experiments}.
% \end{remark}

In the next section, we investigate the relation between the actor updates of the MAAC and 3 MAN algorithms. In particular, we compare the actor parameters of all the algorithms.

\subsection{Relationship between actor updates in algorithms} 
\label{subsec: algo_actor_comparison}
Recall, the actor update for each agent $i\in N$ in FIAP-MAN algorithm is $\theta^i_{t+1} = \theta^i_t + \beta_{\theta,t} w^i_{t+1}$, where $w^i_{t+1} = (1 - \beta_{v,t})w^i_t + \beta_{v,t} G^{i^{-1}}_t \tilde{\delta}^i_t \psi^i_t$. Therefore, the actor update of FIAP-MAN algorithm is
\begin{equation}
\begin{aligned}
    \theta^i_{t+1} &= \theta^i_t + \beta_{\theta,t} \left[(1 - \beta_{v,t})w^i_t + \beta_{v,t} G^{i^{-1}}_t \tilde{\delta}^i_t \psi^i_t \right] 
    \\
    &= \theta^i_t + \beta_{\theta,t}(1 - \beta_{v,t})w^i_t + \beta_{v,t} \left\lbrace \beta_{\theta,t}  G^{i^{-1}}_t \tilde{\delta}^i_t \psi^i_t \right \rbrace.
\end{aligned}
\label{eqn: actor_combined_manac3}
\end{equation}
The above update is almost the same as the actor update of the FI-MAN algorithm with an additional term involving advantage parameters $w^i_t$. However, the contribution of the second term is negligible after some time $t$. Moreover, the third term in Equation (\ref{eqn: actor_combined_manac3}) is a positive fraction of the second term in the actor update of FI-MAN algorithm (Equation (\ref{eqn: actor_update_manac1})). Therefore, the actor parameters in FIAP-MAN and FI-MAN algorithms are almost the same after time $t$. Hence, both the algorithms are expected to converge  almost to the same local optima.

Similarly,  consider the actor update of the AP-MAN algorithm, i.e., $\theta^i_{t+1} = \theta^i_t + \beta_{\theta,t} w^i_{t+1}$, where $w^i_{t+1} = (I - \beta_{v,t} \psi^i_t \psi^{i^{\top}}_t)w^i_t + \beta_{v,t} \tilde{\delta}^i_t \psi^i_t$. Therefore, the actor update of AP-MAN algorithm is
\begin{equation}
\label{eqn: FIAP_actor_combined}
\begin{aligned}
     \theta^i_{t+1} &= \theta^i_t + \beta_{\theta,t} \left[ (I - \beta_{v,t} \psi^i_t \psi^{i^{\top}}_t)w^i_t + \beta_{v,t} \tilde{\delta}^i_t \psi^i_t \right]
    \\
    &= \theta^i_t + \beta_{\theta,t} (I - \beta_{v,t} \psi^i_t \psi^{i^{\top}}_t)w^i_t +  \beta_{v,t} \left\lbrace \beta_{\theta,t} \tilde{\delta}^i_t \psi^i_t \right\rbrace.
\end{aligned}
\end{equation}
Again, the second term in the above equation is negligible after some time $t$, and the third term of Equation (\ref{eqn: FIAP_actor_combined}) is a positive fraction of the second term in the actor update of the MAAC algorithm (see Equation (\ref{eqn: marl_actor_update})). Hence the actor update in the AP-MAN algorithm is almost the same as the MAAC algorithm; therefore, AP-MAN and MAAC are expected to converge almost to the same local optima. 
% The above arguments show that the FI-MAN and FIAP-MAN converge to almost the same optimal point, and AP-MAN and MAAC converge almost to the same optimal point; possibly the same as the optimal point of the FI-MAN algorithm. 

\subsection{\textcolor{black}{Comparison of variants of MAN and MAAC algorithms}}
\label{subsec: compare_variants}
\textcolor{black}{In this section, we show that under some conditions, a variant of the FI-MAN algorithm dominates the corresponding variant of the MAAC algorithm for all $t\geq t_0$, for some $t_0 <\infty$. 
Using this, we derive the corresponding result that compares the objective function $J(\cdot)$ among these variants of algorithms.
For this purpose, we propose a model to evaluate the  `efficiency' of MAAC and FI-MAN algorithms in terms of their goal; maximization of MARL objective function, $J(\theta)$. This comparison exploits the intrinsic property of the Fisher information matrix $G(\theta)$ (Lemma \ref{lemma: sigma_min_less_1_m}, an uniform upper bound on its minimum eigenvalue).  }

\textcolor{black}{Let $\theta^M_t$ and $\theta^N_t$
be the actor parameters in MAAC and FI-MAN algorithms, respectively. Recall the actor updates in MAAC and FI-MAN algorithms were:
\begin{equation}
\label{eqn: actor_updates}
    \textbf{MAAC:} ~ \theta^M_{t+1} = \theta^M_t + \beta_{\theta,t} \tilde{\delta}_t \cdot \psi_t;  \hspace{2 mm} \textbf{FI-MAN:} ~ {\theta}^N_{t+1} = \theta^N_t + \beta_{\theta,t} G^{-1}_{t}\tilde{\delta}_t \cdot \psi_t.
\end{equation}
However, as given in Equation (\ref{eqn: biased_estimate_J}) these updates use the biased estimate of $\nabla J(\cdot)$. Moreover, the Fisher information matrix inverse is updated via the Sherman-Morrison iterative method. In this Section, we work with the deterministic variants of these algorithms where we use $\nabla J(\cdot)$ instead of $\tilde{\delta}_t \cdot \psi_t$, and $G(\theta^N_t)^{-1}$ instead of using $G_t^{{-1}}$ in the actor updates. This avoids the approximation errors as in Equation (\ref{eqn: biased_estimate_J}); however, the same is not possible in the computations since the gradient of the objective function is not known. For ease of notation, we denote the actor parameters in the deterministic variants of MAAC and FI-MAN algorithms by $\tilde{\theta}^{M}$, and $\tilde{\theta}^{N}$, respectively. In particular, we consider the following actor updates. 
\begin{equation}
\label{eqn: actro_parameter_deterministic}
\begin{aligned}
    \textbf{Deterministic MAAC:}~~ \tilde{\theta}^M_{t+1} &= \tilde{\theta}^M_t + \beta_{\tilde{\theta},t} \nabla J(\tilde{\theta}^M_t); 
    \\
    \textbf{Deterministic FI-MAN:}~~ \tilde{\theta}^N_{t+1} &= \tilde{\theta}^N_t + \beta_{\tilde{\theta},t} G(\tilde{\theta}^N_t)^{-1} {\nabla} J(\tilde{\theta}^N_t)
    \\
    &= \tilde{\theta}^N_t + \beta_{\tilde{\theta},t}  \widetilde{\nabla} J(\tilde{\theta}^N_t)
\end{aligned}
\end{equation}
We give sufficient conditions when the objective function value of the limit point of the deterministic FI-MAN algorithm is not worse off than the value by deterministic MAAC algorithm while using the above actor updates.
We want to emphasize that with these updates, the actor parameters in Equation \eqref{eqn: actro_parameter_deterministic} will converge to a local maxima under some conditions (for example, the strong Wolfe's conditions) on the step-size \cite{nocedal2006numerical}. Let  $\tilde{\theta}^{M^{\star}}$ and $\tilde{\theta}^{N^{\star}}$ be the corresponding local maxima. The existence of the local maxima for the deterministic MAN algorithms is also guaranteed via the Wolfe's conditions in the natural gradients space. Note that these local maxima need not be the same as the one obtained from actor updates given in Equation (\ref{eqn: actor_updates}). However, the result given below {may be} valid for the MAAC and FI-MAN algorithms because both $\tilde{\delta} \cdot \psi$ and $\nabla J(\cdot)$ go to zero asymptotically.}

\textcolor{black}{We also assume that both algorithms use the same sequence of step-sizes, $\{\beta_{\tilde{\theta},t}\}$.
Moreover, proof of the results in this section uses Taylor's series expansion and comparison of the objective function, $J(\cdot)$, rather than its estimate $\mu$. Similar ideas are used in \cite{aoki2016optimization,patchell1971separability} where the certainty equivalence principle holds, i.e., the random variables are replaced by their expected values. However, we work with the estimates in {convergence theory/proofs and } computations since the value of the global objective is unknown to the agents.}

\textcolor{black}{First, we {need to } bound the minimum singular value of the Fisher information matrix as in the following Lemma.
\begin{lemma}
\label{lemma: sigma_min_less_1_m}
For the Fisher information matrix, $G(\theta) = \mathbb{E}[\psi \psi^{\top}]$, such that $|| \psi ||\leq 1$, the minimum singular value, $\sigma_{\min}(G(\theta))$ is upper bounded by $\frac{1}{m}$,
\begin{equation}
\sigma_{\min}(G(\theta)) \leq \frac{1}{m}.
\end{equation}
\end{lemma}
The proof of this Lemma is based on the observation that the trace of matrix $\psi \psi^{\top}$ is $||\psi||^2$. We include the proof in Appendix \ref{app: sigma_min_less_1_m}. 
\begin{remark}
In literature, the compatible features are assumed to be uniformly bounded, Assumption X. \ref{ass: feature_conds}. For the linear architecture of features that we are using, assuming this bound to be 1 is not restrictive. 
The features $\psi$, that we use in our computational experiments in Section \ref{sec: experiments}, automatically meet the condition of being normalized by 1, i.e., $||\psi|| \leq 1$.
\end{remark}
\begin{lemma}
\label{lemma: J_per_period_improvement}
Let $J(\cdot)$ be twice continuously differentiable function on a compact set $\Theta$, so that $|\{\nabla^2 J(\tilde{\theta}^M_t)\}_{(i,j)}| \leq H, ~ \forall~i,j\in [m]$ for some $H< \infty$. Moreover, let $J(\tilde{\theta}^M_t) \leq J(\tilde{\theta}^N_t)$, $|| \nabla J(\tilde{\theta}^M_t)|| \leq || \nabla J(\tilde{\theta}^N_t)||$, and $\beta_{\tilde{\theta},{t}}  \frac{mH}{2} + {1} - m^2 \leq 0$. Then, $J(\tilde{\theta}^M_{t+1}) \leq J(\tilde{\theta}^N_{t+1})$.
\end{lemma}}
\begin{proof}
\label{proof: J_per_period_improvement}
\textcolor{black}{The Taylor series expansion of a twice differentiable function $J(\tilde{\theta}^M_{t+1})$ with Lagrange form of remainder \cite{folland1999real} is
\begin{equation}
\label{eqn: taylor_MAAC}
    J(\tilde{\theta}^M_{t+1}) = J(\tilde{\theta}^M_t + \Delta \tilde{\theta}^M_t) = J(\tilde{\theta}^M_t)  + \Delta \tilde{\theta}_t^{M^{\top}} \nabla J(\tilde{\theta}^M_t) + R_M(\Delta \tilde{\theta}^M_t), ~where
\end{equation}
$R_M(\Delta \tilde{\theta}^M_t) = \frac{1}{2!} \Delta \tilde{\theta}_t^{M^{\top}} \nabla^2 J(\tilde{\theta}^M_t + c_M \cdot \Delta \tilde{\theta}^M_t) \Delta \tilde{\theta}^M_t $ for some $c_M\in (0,1)$. }

\textcolor{black}{Similarly, the Taylor series expansion of $J(\tilde{\theta}^N_{t+1})$  with Lagrange form of remainder is
\begin{equation}
\label{eqn: taylor_FIMAN}
    J(\tilde{\theta}^N_{t+1}) = J(\tilde{\theta}^N_t + \Delta \tilde{\theta}^N_t) = J(\tilde{\theta}^N_t)  + \Delta \tilde{\theta}_t^{N^{\top}} \widetilde{\nabla} J(\tilde{\theta}^N_t) + R_N(\Delta \tilde{\theta}^N_t), ~where
\end{equation}
%\textcolor{black}{we will avoid same $R_1$ for both; say $R_N$ and $R_M$} 
$R_N(\Delta \tilde{\theta}^N_t) = \frac{1}{2!} \Delta \tilde{\theta}_t^{N^{\top}} \widetilde{\nabla}^2 J(\tilde{\theta}^N_t + c_N \cdot \Delta \tilde{\theta}^N_t) \Delta \tilde{\theta}^N_t $ for some $c_N \in (0,1)$. }

\textcolor{black}{Now, consider the difference $J(\tilde{\theta}^M_{t+1}) - J(\tilde{\theta}^N_{t+1})$,
\begin{eqnarray}
    &=& J(\tilde{\theta}^M_t) - J(\tilde{\theta}^N_t) + \Delta \tilde{\theta}^{M^{\top}}_{t} \nabla J(\tilde{\theta}^M_t) - \Delta \tilde{\theta}^{N^{\top}}_{t} \widetilde{\nabla} J(\tilde{\theta}^N_t) + R_M(\Delta \tilde{\theta}^M_t) - R_N(\Delta \tilde{\theta}^N_t) \nonumber
    \\
    &\overset{(i)}{\leq}& \Delta \tilde{\theta}^{M^{\top}}_{t} \nabla J(\tilde{\theta}^M_t) - \Delta \tilde{\theta}^{N^{\top}}_{t} G(\tilde{\theta}^N_t)^{-1} \nabla J(\tilde{\theta}^N_t)+ R_M(\Delta \tilde{\theta}^M_t) \nonumber
    \\
    &\overset{(ii)}{=}& (\tilde{\theta}^M_{t+1} - \tilde{\theta}^M_t)^{\top} J(\tilde{\theta}^M_t)  - (\tilde{\theta}^N_{t+1} - \tilde{\theta}^N_t)^{\top} G(\tilde{\theta}^N_t)^{-1} \nabla J(\tilde{\theta}^N_t) 
    +  R_M(\Delta \tilde{\theta}^M_t) \nonumber
    % \\
    % &\overset{(iii)}{=}& \beta_{\tilde{\theta},t} \nabla J(\tilde{\theta}^M_t)^{\top} \nabla J(\tilde{\theta}^M_t) - \beta_{\tilde{\theta},t} (G(\tilde{\theta}^N_t)^{{-1}} \nabla J(\tilde{\theta}^N_t))^{\top} (G(\tilde{\theta}^N_t)^{-1}  \nabla J(\tilde{\theta}^N_t)) + R_M(\Delta \tilde{\theta}^M_t)  \nonumber
    \\
    &\overset{(iii)}{=}& \beta_{\tilde{\theta},{t}} \left(|| \nabla J(\tilde{\theta}^M_t) ||^2 -  || G(\tilde{\theta}^N_t)^{-1} \nabla J(\tilde{\theta}^N_t))||^2 \right) + R_M(\Delta \tilde{\theta}^M_t) \nonumber
    \\
    &\overset{(iv)}{\leq}& \beta_{\tilde{\theta},{t}} \left(|| \nabla J(\tilde{\theta}^N_t) ||^2 -  || G(\tilde{\theta}^N_t)^{-1} \nabla J(\tilde{\theta}^N_t))||^2 \right) + R_M(\Delta \tilde{\theta}^M_t)
    \label{eqn: difference_t}
\end{eqnarray}
where $(i)$ follows because $J(\tilde{\theta}^M_t) \leq J(\tilde{\theta}^N_t)$, $R_N(\Delta \tilde{\theta}^N_t) \geq 0$ and $\widetilde{\nabla} J(\tilde{\theta}^N_t) = G(\tilde{\theta}^N_t)^{-1} \nabla J(\tilde{\theta}^N_t)$. $(ii)$ uses the fact that $\Delta \tilde{\theta}^M_t = \tilde{\theta}^M_{t+1} - \tilde{\theta}^M_t;~ \Delta \tilde{\theta}^N_t = \tilde{\theta}^N_{t+1} - \tilde{\theta}^N_t$. $(iii)$ is the consequence of updates in Equation (\ref{eqn: actro_parameter_deterministic}). Finally, $(iv)$ follows from the fact that $||\nabla J(\tilde{\theta}^M_t)|| \leq ||\nabla J(\tilde{\theta}^N_t)|| $.}

\textcolor{black}{Now, from Equation (\ref{eqn: difference_t}) and using the fact that for any positive definite matrix $\mathbf{A}$ and a vector $\mathbf{v}$, $||\mathbf{A} \mathbf{v}|| \geq \sigma_{\min}(\mathbf{A}) ||\mathbf{v}||$, we have
\begin{eqnarray*}
    || G(\tilde{\theta}^N_t)^{-1} \nabla J(\tilde{\theta}^N_t))||^2 &\geq& \sigma_{\min}^2( G(\tilde{\theta}^N_t)^{-1}) || \nabla J(\tilde{\theta}^N_t))||^2,
    \\
    % \implies~ - || G(\tilde{\theta}^N_t)^{-1} \nabla J(\tilde{\theta}^N_t))||^2 &\leq& - \sigma_{\min}^2( G(\tilde{\theta}^N_t)^{-1}) || \nabla J(\tilde{\theta}^N_t))||^2
    % \\
    \implies~ - || G(\tilde{\theta}^N_t)^{-1} \nabla J(\tilde{\theta}^N_t))||^2 &\leq& - \frac{1}{\sigma_{\min}^2( G(\tilde{\theta}^N_t))} || \nabla J(\tilde{\theta}^N_t))||^2.
\end{eqnarray*}
Therefore, from Equation (\ref{eqn: difference_t}), we have
%\textcolor{black}{last line has inequality??}
\begin{eqnarray}
\label{eqn: upper_bound_t}
    J(\tilde{\theta}^M_{t+1}) - J(\tilde{\theta}^N_{t+1}) &\leq&  \beta_{\tilde{\theta},t} \left( 1 - \frac{1}{\sigma_{\min}^2(G(\tilde{\theta}^N_t))} \right) || \nabla J(\tilde{\theta}^N_t))||^2 +  R_M(\Delta \tilde{\theta}^M_t) \nonumber
    \\
     &\overset{(v)}{\leq}&  \beta_{\tilde{\theta},{t}} \left( 1 - m^2 \right) || \nabla J(\tilde{\theta}^N_t))||^2 
    + R_M(\Delta \tilde{\theta}^M_t),
    % \\
    % &\overset{(vi)}{\leq} & \beta_{\tilde{\theta},{t}} \left( \frac{1}{m} - m \right) || \nabla J(\tilde{\theta}^N_t))||^2_1 
    % + R_1(\Delta \tilde{\theta}^M_t)
    \label{eqn: 94_t}
\end{eqnarray}
{where} $(v)$ follows from Lemma \ref{lemma: sigma_min_less_1_m} as $\sigma_{\min}(G(\tilde{\theta}^N_t)) \leq \frac{1}{m}$, implies
$-\frac{1}{\sigma_{\min}^2(G(\tilde{\theta}^N_t)) } \leq -m^2$.}
% \begin{eqnarray*}
%     \sigma_{\min}(G(\tilde{\theta}^N_t)) \leq \frac{1}{m}
%     \\
%     \implies~ \sigma_{\min}^2(G(\tilde{\theta}^N_t)) \leq \frac{1}{m^2}
%     \\
%     \implies~ - \sigma_{\min}^2(G(\tilde{\theta}^N_t)) \geq -\frac{1}{m^2}
%     \\
%     \implies~ - \frac{1}{\sigma_{\min}^2(G(\tilde{\theta}^N_t)) } \leq -m^2
% \end{eqnarray*}}
%\textcolor{black}{to check and type many intermediate steps -- reciprocal, sign changes, lower and upper bounds, etc.}
%\textcolor{black}{this may not hold if $x_i$s are fractions??}. \textcolor{black}{why are we moving from L2 norm to L1 norm?} 

\textcolor{black}{Since $J(\cdot)$ is twice continuously differentiable function on the compact set $\Theta$, we have for all $i, j\in [m]$,  $|\{\nabla^2 J(\tilde{\theta}^M_t\}_{(i,j)}| \leq H < \infty$. In above Taylor's series expansion, we then have
\begin{eqnarray*}
    | R_M(\Delta \tilde{\theta}^M_t)| \leq \frac{H}{2!} || \Delta \tilde{\theta}^M_t ||^2_1
    &\overset{(vii)}{=}& \beta_{\tilde{\theta},{t}}^2 \frac{H}{2}  || \nabla J(\tilde{\theta}^M_t) ||^2_1
    \\
    &\overset{(viii)}{\leq}& \beta_{\tilde{\theta},{t}}^2 \frac{mH}{2}  || \nabla J(\tilde{\theta}^M_t) ||^2 
    \\
    &\overset{(ix)}{\leq}& \beta_{\tilde{\theta},{t}}^2 \frac{mH}{2}  || \nabla J(\tilde{\theta}^N_t) ||^2,
\end{eqnarray*}
where $(vii)$ follows from actor update of the deterministic FI-MAN algorithm. $(viii)$ holds because for any $\mathbf{x} \in \mathbb{R}^l$, the following is true: 
%\textcolor{black}{another LHS needed?}
$||\mathbf{x}||_2 \leq ||\mathbf{x}||_1 \leq \sqrt{l}~ ||\mathbf{x}||_2$ \cite{horn2012matrix}. $(ix)$ comes from the assumption that $|| \nabla J(\tilde{\theta}^M_t) ||^2 \leq || \nabla J(\tilde{\theta}^N_t) ||^2$.
Now, from Equation (\ref{eqn: 94_t}) and above upper bound, we have 
%\textcolor{black}{from $\ell_2$ and $\ell_1$ needs a multiplier or something?}
%\textcolor{black}{need to see Eq 99 as $(1 - m^2) < 0)$}
\begin{eqnarray} 
    J(\tilde{\theta}^M_{t+1}) - J(\tilde{\theta}^N_{t+1}) 
    &\leq& \beta_{\tilde{\theta},{t}} \left( 1 - m^2 \right) || \nabla J(\tilde{\theta}^N_t))||^2 +
    \beta_{\tilde{\theta},{t}}^2 \frac{mH}{2}  || \nabla J(\tilde{\theta}^N_t) ||^2 \nonumber
    \\
    &=& \beta_{\tilde{\theta},{t}} \left\lbrace \beta_{\tilde{\theta},{t}} \frac{mH}{2} + {1} - m^2    \right\rbrace || \nabla J(\tilde{\theta}^N_t) ||^2 \nonumber
    \\
    &\leq& 0, \label{eqn: t+1_inequality}
\end{eqnarray}
where the last inequality follows from the assumption $ \beta_{\tilde{\theta},{t}} \frac{mH}{2}  +{1} - m^2  \leq 0$. Therefore, $J(\tilde{\theta}^M_{t+1}) \leq J(\tilde{\theta}^N_{t+1})$. }
% Therefore, by principle of Mathematical induction, we have $J(\tilde{\theta}^M_t) \leq J(\tilde{\theta}^M_t)$ for all $t\geq t_0$.
\end{proof}

% \textcolor{black}{now, what does this mean as far as computations are concerned -- say, our own computations?}

%\textcolor{black}{these $\Delta \theta$s are more accurate; $\tilde{\delta} \psi$ used in algos is biased estimate of $\nabla J$}

%\textcolor{black}{$G^{-1}$ is via ${(\psi \psi^{\top})}^{-1}$ where $G = E[\psi \psi^{\top}]$, i.e., this is an unbiased estimate}
% \textcolor{black}{in the limiting case, $\nabla J$ is zero; so, these biased estimates are unbiased or bias is going to zero. }

% \textcolor{black}{rates of convergence of alogs is compared, rather than  final values}

\begin{theorem}
\label{thm: J_comp_t+1}
\textcolor{black}{Let $J(\cdot)$ be twice continuously differentiable function on a compact set $\Theta$, so that $|\{\nabla^2 J(\tilde{\theta}^M_t)\}_{(i,j)}| \leq H, ~ \forall~i,j\in [m]$ for some $H< \infty$. Moreover, let $J(\tilde{\theta}^M_{t_0}) \leq J(\tilde{\theta}^N_{t_0})$ for some $t_0 > 0$, and for every $t\geq t_0$, let $|| \nabla J(\tilde{\theta}^M_t)|| \leq || \nabla J(\tilde{\theta}^N_t)||$, and $\beta_{\tilde{\theta},t} \frac{mH}{2}  + 1 - m^2 \leq 0$. Then, $J(\tilde{\theta}^M_t) \leq J(\tilde{\theta}^N_t)$, for all $t\geq t_0$. 
}
\textcolor{black}{Further for the local maxima $\tilde{\theta}^{M^{\star}}$, and $\tilde{\theta}^{N^{\star}}$ of the updates in Equation (\ref{eqn: actro_parameter_deterministic}) we have $J(\tilde{\theta}^{M^{\star}}) \leq J(\tilde{\theta}^{N^{\star}})$.}
\end{theorem}

% \textcolor{black}{much better; we will fine tune these}
\begin{proof}
\label{proof: PMI_MAN_dominance}
\textcolor{black}{We prove this theorem via Principle of Mathematical Induction (PMI). From assumption, we have $J(\tilde{\theta}^M_{t_0}) \leq J(\tilde{\theta}^N_{t_0})$. Now, using $t=t_0$ in the Lemma \ref{lemma: J_per_period_improvement}, under the assumptions, we have $J(\tilde{\theta}^M_{t_0+1}) \leq J(\tilde{\theta}^N_{t_0+1})$. Thus, the base case of PMI is true.}

\textcolor{black}{Next, we assume that $J(\tilde{\theta}^M_t) \leq J(\tilde{\theta}^N_t)$ for any $t=t_0+k$, where $k\in \mathbb{Z}^{+}$. Also, from assumption,  for every $t\geq t_0+k$, we have
% we have $|\{\nabla^2 J(\tilde{\theta}^M_t)\}_{(i,j)}| \leq M_t$, for all $i,j\in [m]$; 
$|| \nabla J(\tilde{\theta}^M_t)|| \leq || \nabla J(\tilde{\theta}^N_t)||$, and $\beta_{\tilde{\theta},t} \frac{mH}{2}  + 1 - m^2 \leq 0$. Therefore, again from Lemma \ref{lemma: J_per_period_improvement} we have $J(\tilde{\theta}^M_{t_0+k+1}) \leq J(\tilde{\theta}^N_{t_0+k+1})$. From PMI we have 
$$
J(\tilde{\theta}^M_t) \leq J(\tilde{\theta}^N_t),~~ \forall~~t \geq t_0.
$$
Next, consider the limiting case. 
%Note that $\nabla J(\theta^{M^{\star}}) = \nabla J(\theta^{N^{\star}}) = 0$. 
Taking the limit $t \rightarrow \infty$ in the above equation and using the fact that $J(\cdot)$ is continuous on the compact set $\Theta$, we have
\begin{equation}
    lim_{t\rightarrow \infty} J(\tilde{\theta}^M_t) =  J(\tilde{\theta}^{M^{\star}}), ~ and~
    lim_{t\rightarrow \infty} J(\tilde{\theta}^N_t) = J(\tilde{\theta}^{N^{\star}}).
\end{equation}
This ends the proof.}
\end{proof}

The following subsection will highlight some consequences of using the Boltzmann policy.

\subsection{KL divergence based natural gradients for Boltzmann policy}
\label{subsec: kl_boltzmann} 
% \textcolor{black}{need to change the title for this subsection?}

One specific policy that is often used in RL literature is the Boltzmann policy \cite{doan2021finite}.  Recall, the parameterized Boltzmann policy is given as
\begin{equation}
\label{eqn: boltzmann}
\pi_{\theta_t}(s,a) = \frac{\exp(q^{\top}_{s,a} \theta_t)}{\sum_{b\in\mathcal{A}} \exp(q^{\top}_{s,b} \theta_t)},
\end{equation}
where $q^{\top}_{s,a}$ is the feature for any $(s,a)$. Here the features $q_{s,a}$ are assumed to be uniformly bounded by 1.
\begin{lemma}
\label{lemma: KL_boltzmann}
For the Boltzmann policy as given in Equation (\ref{eqn: boltzmann}) the KL divergence between policy parameterized by $\theta_t$ and $\theta_t + \Delta \theta_t$ is given by
\begin{equation}
    KL(\pi_{\theta_t}(s,a)|| \pi_{\theta_t + \Delta \theta_t} (s,a)) = \mathbb{E} \left[\log \left( \sum_{b\in \mathcal{A}} \pi_{\theta_t}(s,b)  \exp(\Delta q^{\top}_{s,ba} \Delta \theta_t)\right) \right],
\end{equation}
where $\Delta q^{\top}_{s,ba} = q^{\top}_{s,b} - q^{\top}_{s,a}$.
\end{lemma}
The proof of this lemma is deferred to Appendix \ref{app: KL_boltzmann}. The above KL divergence represents that we have a non-zero curvature if the action taken is better than the averaged action. From above equation we also observe that $\exp(\Delta q_{s, ba}\Delta \theta_t) \neq 1 $ if and only if ${\Delta q_{s,ba}}$ is orthogonal to $\Delta \theta_t$. So, except when they are orthogonal, $\log (\sum_{b \in \mathcal{A}}\pi_{\theta_t}(s,b) \cdot \exp{(\Delta q_{s, ba} \Delta \theta_t})) \neq 0$  as $ \sum_{b \in \mathcal{A}} \pi_{\theta_t}(s,b) = 1$. Thus, the curvature is non-zero, larger or smaller depends on the direction $\Delta \theta_t$ makes with that of feature difference $q_{s, b}^{\top} - q^{\top}_{s,a}$; if the angle is zero, it is better. 

Next lemma provide the derivative of KL divergence between policy $\pi_{\theta_t}$, and $\pi_{\theta_t + \Delta \theta_t}$. 
\begin{lemma}
\label{lemma: grad_KL_softmax}
For the Boltzmann policy as given in Equation (\ref{eqn: boltzmann}) we have 
\begin{equation}
\label{eqn: grad_kl_explicit}
    \nabla KL(\pi_{\theta_t}(\cdot,\cdot)|| \pi_{\theta_t + \Delta \theta_t}(\cdot,\cdot)) = -\mathbb{E}[\nabla \log \pi_{\theta_t + \Delta \theta_t} (s,a)],
\end{equation}
\end{lemma}
i.e., $\psi_{\theta_{t+1}} = \nabla \log \pi_{\theta_t + \Delta \theta_t}$ is an unbiased estimate of $\nabla KL(\pi_{\theta_t}(\cdot,\cdot)|| \pi_{\theta_t + \Delta \theta_t}(\cdot,\cdot))$. Proof of this lemma is available in Appendix \ref{app: grad_KL_softmax}. The proof uses the fact that for Boltzmann policies the compatible features are same as the features associated to policy, except normalized to be mean zero for each state \cite{sutton1999policy}.

Recall, in Equation (\ref{eqn: approx_kl}) we provide the approximate derivative of KL divergence using second order Taylor approximation of the objective function as 
\begin{equation}
\label{eqn: equality_of_derivatives}
\nabla KL(\pi_{\theta_t}(\cdot,\cdot)|| \pi_{\theta_t + \Delta \theta_t}(\cdot,\cdot)) \approx G(\theta_t) \Delta \theta_t.
\end{equation}
Moreover, from Equation (\ref{eqn: KL_two_eqns}) we have $G(\theta_t) \Delta \theta_t  =  -\frac{1}{\rho_t} \nabla_{\theta} J(\theta_t)$. 
Also, from Lemma \ref{lemma: kl_grad_J} we have, 
$$ \nabla KL(\pi_{\theta_t}(\cdot,\cdot)|| \pi_{\theta_t + \Delta \theta_t}(\cdot,\cdot)) \approx -\frac{1}{\rho_t} \nabla J(\theta_t). $$

%\textcolor{black}{Thus, $\nabla KL$ is a new (appropriate) representation of the gradient of the MARL objective function $J(\cdot)$ that is implicitly exploited by MAN algorithms to have better performance. }

Thus, combining Lemma \ref{lemma: grad_KL_softmax}, and Equation (\ref{eqn: equality_of_derivatives}) for Boltzmann policies we have 
\begin{equation}
\mathbb{E}[\nabla \log \pi_{\theta_t + \Delta \theta_t} (s,a)] \approx \frac{1}{\rho_t} \nabla J(\theta_t),
\label{eqn: gradKL-gradJ}
\end{equation}
i.e., $\nabla \log \pi_{\theta_t + \Delta \theta_t} = \psi_{\theta_{t+1}}$ 
%\textcolor{black}{looks ok; but, to discuss}
is approximately an unbiased estimate of $\nabla J(\theta_t)$ upto scaling of $\frac{1}{\rho_t}$ for the Boltzmann policies. 
% So, to get the estimate of the gradient of the objective function, we need to fnd the score function at the next iterate. 
% \textcolor{red}{Can add more??}
% \textcolor{red}{to remove this} \textcolor{black}{ok} Also, from Equation (\ref{eqn: biased_estimate_J}) we have $\delta_t \psi_t$ as a biased estimate of $\nabla J(\theta_t)$, implying that the Boltzmann policies can have better estimate of $\nabla J(\theta_t)$. Using $\psi_{\theta_{t+1}}$ in the actor update of MAAC we have,
% $$
% \theta_{t+1} = \theta_t + \beta_{\tilde{\theta},t} \psi_{ \theta_t + \Delta \theta} = \theta_t + \beta_{\tilde{\theta},t} \psi_{\theta_{t+1}},
% $$
% which is a fixed point equation in $\theta_{t+1}$; might be theoretically helpful, but not implementable, because $\theta_{t+1}$ itself requires $\psi_{\theta_{t+1}}$. Thus, in our computations we stick to the biased estimate, i.e., $\delta_t \psi_t$. Similarly, for the FI-MAN algorithm, we have a fixed point equation as
% $$
% \theta_{t+1} = \theta_t + \beta_{\theta,t} G_t^{-1} \psi_{\theta_{t+1}},
% $$
% which is again theoretically sound, but for the same reason, it is not implementable. However, similar to Equation (\ref{eqn: grad_kl_relation_grad_j}), we again have Equation (\ref{eqn: gradKL-gradJ}) relating the gradient of the objective function to the gradient of KL divergence between the policies separated by $\Delta \theta$ which renders actor updates in the prediction space. 
It is a valuable observation because to move along the gradient of objective function $J(\cdot)$, we can adjust the updates (of actor parameter) just by moving in the  $\pi_{\theta_t}$ prediction space via the compatible features.

% \textcolor{red}{write for the objective and grad KL relation}

% \textcolor{black}{that this is not reflected in iterate updates of algo?? it may suggest some new algos as well?}  \textcolor{red}{To write something from discussion..} \textcolor{black}{ok, write this soon} \textcolor{red}{the nature of results we have in this section implies that we have an unbiased estimator of the gradient of the objective function.  For FI-MAN we have  again this is fixed point equation, and hence not implementable. Therefore, in our implementations we have used the updates involving the TD error. Note that for other algorithms like AP-MAN and FIAP-MAN we are using the advantage parameters, and hence above arguments are not needed.} \textcolor{black}{so, this is an advantage these softmax class of policies enjoyed. so, when compared to other classes, which we didn't do, it may come out or other way, softmax class may do worse; say, other parameterized ones like normal or so.}
% \end{remark}

We now prove the convergence of FI-MAN, AP-MAN, and FIAP-MAN algorithms. The proofs majorly use the idea of two-time scale stochastic approximations from \cite{borkar2009book}.

%\textcolor{black}{we can add Kaitang Zhang also; else, leave just with Vivek Borkar} 

\section{Convergence analysis}
\label{sec: Convergence_analysis}
This section provides the convergence proof of all the 3 MAN algorithms. 
% To prove the convergence of the algorithms 
To this end, we need following assumptions on the features $\varphi(s)$ of the value function, and $f(s,a)$ of the rewards for any $s\in \mathcal{S}, a\in \mathcal{A}$. This assumption is similar to \cite{zhang2018fully}, and also used in the convergence proof for single-agent natural actor-critic \cite{bhatnagar2009natural}. 

\begin{assumption}
\label{ass: feature_conds}
The feature vectors $\varphi(s)$, and $f(s,a)$ associated with the value function and the reward function are uniformly bounded for any $s\in \mathcal{S}, a\in \mathcal{A}$. Moreover, let the feature matrix $\varPhi \in \mathbb{R}^{|\mathcal{S}|\times L}$  have $[\varphi_l(s), s\in \mathcal{S}]^{\top}$ as its $l$-th column for any $l\in [L]$, and feature matrix $F \in \mathbb{R}^{|\mathcal{S}||\mathcal{A}|\times M}$  have $[f_m(s,a), s\in \mathcal{S},a\in \mathcal{A}]^{\top}$ as its $m$-th column for any $m\in [M]$, then $\varPhi$ and $F$ have full column rank, and for any $\omega \in \mathbb{R}^L$, we have $\varPhi \omega \neq \mathbbm{1}$.
\end{assumption}

Apart from assumption X. \ref{ass: feature_conds}, let $D^s_{\theta} = [d_{\theta}(s), s\in \mathcal{S}]$, and $\bar{R}_{\theta} = [\bar{R}_{\theta}(s), s\in \mathcal{S}]^{\top} \in \mathbb{R}^{|\mathcal{S}|}$ with $\bar{R}_{\theta}(s) = \sum_a \pi_{\theta}(s,a)\cdot \bar{R}(s,a)$.
Define the operator $T^V_{\theta}: \mathbb{R}^{|\mathcal{S}|}\rightarrow \mathbb{R}^{|\mathcal{S}|}$ for any state value vector $X\in \mathbb{R}^{|\mathcal{S}|}$ as, 
\begin{equation*}
\label{eqn: T_V_theta}
T^V_{\theta} (X) = \bar{R}_{\theta} - J(\theta)\mathbbm{1} + P^{\theta}X.
\end{equation*}
% We will now give the convergence of the multi-agent natural (MAN) actor-critic algorithms. 
The proof of all the 3 MAN algorithms are done in two steps: (a) convergence of the objective function estimate, critic update, and rewad parameters keeping the actor parameters $\theta^i$ fixed for all agents $i\in N$, and (b) convergence of the actor parameters to an asymptotically stable equilibrium set of the ODE corresponding to the actor update. To this end, we require the following assumption on Fisher information matrix $G^i_t$ and its inverse $G^{i^{-1}}_t$. This assumption is used by \cite{bhatnagar2009natural} in case of single-agent natural actor-critic algorithms, here we have suitably modified it for multi-agent setup.
\begin{assumption}
\label{ass: g_inv}
For each agent $i\in N$, the recursions of Fisher information matrix $G^i_t$, and its inverse $G^{i^{-1}}_t$ as given in Equations (\ref{eqn: fisher_iterative_update}) and (\ref{eqn: sherman-morrison}) respectively
satisfy 
$ sup_{t,\theta^i,s,a^i} || G_t^i || < \infty;~~ sup_{t,\theta^i,s,a^i} || G_t^{i^{-1}} || < \infty. $
\end{assumption}

Assumption X. \ref{ass: g_inv} ensures that the FI-MAN and FIAP-MAN actor-critic algorithms does not get stuck in a non-stationary point. 
%\textcolor{black}{should we say that this was used by Shalabh albeit for single agent RL?} 
A sufficient condition for both requirements of assumption X. \ref{ass: g_inv} is following: for each agent $i\in N$, for some scalars $c^i_1,c^i_2>0$,
\begin{equation*}
c^i_1 ||x^i||^2 \leq x^{i^{\top}} \psi^i_{s,a^i}\psi^{i^{\top}}_{s,a^i} x^i \leq c^i_2 ||x^i||^2, ~~ \forall ~s\in \mathcal{S}, a^i\in \mathcal{A}^i, x^i,\theta^i\in \mathbb{R}^{m_i}.
\end{equation*}
Therefore, for each agent $i\in N$, we have $\bar{c}^i_1 ||x^i||^2 \leq x^{i^{\top}} G^i x^{i} \leq \bar{c}^i_2 ||x^i||^2$,
and the eigenvalues of $G^i$ lie between some positive scalars $\bar{c}^i_1$, and $\bar{c}^i_2$. Since, $\bar{c}^i_1, \bar{c}^i_2 > 0, ~~ \forall ~ i \in N$, thus from Propositions A.9 and A.15 of \cite{bertsekas_nlp} the algorithms will not get stuck at non-stationary points. Moreover using stochastic approximation technique we know that $G^i_t$ will converge asymptotically to $G(\theta^i)$ almost surely if $\theta^i$ is held fixed \cite{borkar2009book}.
% (\textcolor{black}{this sentence in bracket needed?} this is easy to verify by comparing it with the stochastic approximation scheme given in \cite{borkar2009book}). 

To ensure the existence of local optima of $J(\theta)$ we make the following assumptions on policy parameters, $\theta^i_t$ for each agent $i\in N$.
\begin{assumption}
\label{ass: existence_local_min}
The policy parameters $\{\theta^i_t\}_{t\geq 0}$ update includes a projection operator $\Gamma^i: \mathbb{R}^{m_i} \rightarrow \Theta^i \subset \mathbb{R}^{m_i}$,  that projects, for every $t\geq 0$, $\theta^i_t$ onto a compact set $\Theta^i$. Moreover, $\Theta = \prod_{i=1}^n \Theta^i$ is large enough to include at least one local optima of $J(\theta)$. 
\end{assumption}

For each agent $i\in N$, let $\hat{\Gamma}^i$ be the transformed projection operator defined for any $\theta \in \Theta$ with $h: \Theta \rightarrow \mathbb{R}^{\sum_{i\in N} m_i}$ being a continuous function as follows:
\begin{equation}
\label{eqn: proj_transformed}
    \hat{\Gamma}^{i}(h(\theta)) = lim_{0<\eta\rightarrow 0}\frac{\Gamma^i(\theta^i + \eta h(\theta)) - \theta^i}{\eta}.
\end{equation}
If the above limit is not unique, $\hat{\Gamma}^i(h(\theta))$ denotes the set of all possible limit points of Equation (\ref{eqn: proj_transformed}). The above projection operator is useful in convergence proof of the policy parameters. It is an often-used technique to ensure boundedness of iterates in stochastic approximation algorithms. However, we do not require a projection operator in computations because iterates remain bounded.
% We will now provide the convergence proofs for all the multi-agent natural actor-critic algorithms. 

We begin by proving the convergence of the updates given in Equations (\ref{eqn: marl_obj_updates}), (\ref{eqn: marl_critic_updates}), and (\ref{eqn: marl_lmbda_update}), respectively. The following theorem will be common in the proof of all the 3 MAN algorithms.

\begin{theorem}
\label{thm: critic_convergence_fisher}
Under assumptions X. \ref{ass: regularity}, X. \ref{ass: comm_matrix}, and X. \ref{ass: feature_conds}, for any policy $\pi_{\theta}$, with sequences $\{\lambda^i_t\}, \{\mu^i_t\}, \{v^i_t\}$, we have $lim_t~ \mu^i_t = J(\theta),~ lim_t~ \lambda^i_t = \lambda_{\theta}$, and $lim_t~v^i_t = v_{\theta}$ a.s. for each agent $i\in N$, where $J(\theta),~ \lambda_{\theta}$, and $v_{\theta}$ are unique solutions to 
\begin{eqnarray*}
    F^{\top}D_{\theta}^{s,a}(\bar{R} - F\lambda_{\theta}) &=& 0,
    \\
    \varPhi^{\top} D_{\theta}^s[T_{\theta}^V(\varPhi v_{\theta}) - \varPhi v_{\theta}] &=& 0.
\end{eqnarray*}
\end{theorem}
Proof of this theorem follows from \cite{zhang2018fully}. For the sake of completeness, we briefly present the proof in Appendix \ref{app: critic_convergence_proof}.

\subsection{Convergence of FI-MAN actor-critic algorithm}
\label{subsec: convergence_MANAC1}

To prove the convergence of FI-MAN algorithm we first show the convergence of recursion for the Fisher information matrix inverse as in Equation (\ref{eqn: sherman-morrison}). 
\begin{theorem}
\label{thm: fisher_convergence_algo_1}
For each agent $i \in N$, and given parameter $\theta^i$,
% $G_t^{i^{-1}}, t\geq 1$, as given in Equation (\ref{eqn: sherman-morrison}),
we have $G_t^{i^{-1}} \rightarrow G(\theta^i)^{-1}$ as $t\rightarrow \infty$ with probability one.
\end{theorem}

\begin{proof}
\label{proof: fisher_convergence_algo_1}
Consider the following difference for each agent $i\in N$,
% $\forall~ i\in N$,
\begin{eqnarray*}
    ||G^{i^{-1}}_t - G(\theta^i)^{-1}|| &=& || G(\theta^i)^{-1} G(\theta^i)G^{i^{-1}}_t - G(\theta^i)^{-1}G^{i}_t G^{i^{-1}}_t ||
    \\
    &=& || G(\theta^i)^{-1} (G(\theta^i) - G^{i}_t) G^{i^{-1}}_t ||
    \\
    &\leq& sup_{\theta^i} || G(\theta^i)^{-1}||~  sup_{t,s,a^i} ||G^{i^{-1}}_t||~ || G(\theta^i) - G_t^i||
    \\
    &\rightarrow& 0~ as ~t\rightarrow \infty.
\end{eqnarray*}
The above inequality  follows from the induced matrix norm property and assumption X. \ref{ass: g_inv}.
\end{proof} 
Next, we prove the convergence of actor update. To this end, we view $-r^i_{t+1}$ as the cost incurred at time $t$. Hence, transform the actor recursion in the FI-MAN algorithm as
\begin{equation*}
    \label{eqn: actor_update_manac1_cost_view}
    \theta^i_{t+1} \leftarrow \theta^i_t - \beta_{\theta,t} \cdot G_t^{i^{-1}} \cdot \tilde{\delta}^i_t \cdot \psi^i_t.
\end{equation*}
The convergence of the FI-MAN actor-critic algorithm with linear function approximation is given in the following theorem.
\begin{theorem}
\label{thm: MANAC-algo_1}
Under the assumptions X. \ref{ass: regularity} - X. \ref{ass: existence_local_min}, the sequence $\{\theta^i_t\}_{t\geq 0}$ obtained from the actor step of the FI-MAN algorithm converges almost surely to asymptotically stable equilibrium set of the ODE
\begin{equation}
\label{eqn: actor_ode}
    \dot{\theta}^i = \hat{\Gamma}^i[-G(\theta^{i})^{-1}\mathbb{E}_{s_t\sim d_{\theta}, a_t\sim \pi_{\theta}} ( \tilde{\delta}^i_{t,\theta} \psi^i_{t,\theta})],~~~ \forall~i\in N.
\end{equation}
\end{theorem}

\begin{proof}
\label{proof: MANAC-algo_1}
Let $\mathcal{F}_{t,1} = \sigma (\theta_{\tau}, \tau \leq t)$ be the $\sigma$-field generated by $\{\theta_{\tau}\}_{\tau\leq t}$. Let
\begin{eqnarray*}
    \xi_{t+1,1}^i &=& -G(\theta^i_t)^{-1} \left\lbrace \tilde{\delta}^i_t \psi^i_t - \mathbb{E}_{s_t\sim d_{\theta_t}, a_t\sim \pi_{\theta_t}}(\tilde{\delta}^i_t \psi^i_t | \mathcal{F}_{t,1}) \right\rbrace
    \\
    \xi_{t+1,2}^i &=&  -G(\theta^i_t)^{-1} \mathbb{E}_{s_t\sim d_{\theta_t}, a_t\sim \pi_{\theta_t}}( ( \tilde{\delta}^i_t -\tilde{\delta}^i_{t, \theta_t}) \psi^i_t | \mathcal{F}_{t,1})
\end{eqnarray*}
where $\tilde{\delta}^i_{t,\theta_t}$ is defined as
\begin{equation}\label{delta-tilde}
    \tilde{\delta}^i_{t,\theta_t} = f_t^{\top}\lambda_{\theta_t} - J(\theta_t) + \varphi_{t+1}^{\top}v_{\theta_t} - \varphi_t^{\top}v_{\theta_t}.
\end{equation}
The actor update in the FI-MAN algorithm with local projection then becomes,
\begin{equation}
\label{eqn: actor_update_in_fisher_thm}
    \theta^i_{t+1} = \Gamma^i [\theta^i_t - \beta_{\theta,t} G(\theta^{i}_t)^{-1}\mathbb{E}_{s_t\sim d_{\theta_t}, a_t\sim \pi_{\theta_t}} ( \tilde{\delta}^i_{t} \psi^i_{t} | \mathcal{F}_{t,1}) + \beta_{\theta,t} \xi^i_{t+1,1} + \beta_{\theta,t} \xi^i_{t+1,2}].
\end{equation}
For a.s. convergence to the set of asymptotically stable equilibria of the ODE Equation (\ref{eqn: actor_ode}) for each $i \in N$, we appeal to Kushner-Clark lemma (see appendix \ref{app: K-C_lemma}). Below we verify the 3 main conditions of it. 

First, note that $\xi_{t+1,2} = o(1)$ since critic converges at the faster time scale, i.e., $\tilde{\delta}^i_t \rightarrow \tilde{\delta}^i_{t,\theta_t}$ a.s. 
%other part, we use Kushner-Clark lemma (see appendix \ref{app: K-C_lemma}) the update in Equation (\ref{eqn: actor_update_in_fisher_thm}) converges a.s to the set of asymptotically stable equilibria of ODE Equation (\ref{eqn: actor_ode}) for each $i\in N$. This concludes the proof. 
%Moreover, 
Next, let $M^{1,i}_t = \sum_{\tau = 0}^t \beta_{\theta,\tau} \xi^i_{\tau+1,1}$; then $\{M^{1,i}_t\}$ is a martingale sequence. The sequences $\{z^i_t\}, \{\psi^i_t\}, \{G^{i^{-1}}_t\}$, and $\{\varphi^i_t\}$ are all bounded (by assumptions), and so is the sequence $\{\xi_{t,1}^i\}$ (Here $z^i_t = [\mu^i_t, (\lambda^i_t)^{\top},(v^i_t)^{\top}]^{\top}$ is the same
% \textcolor{black}{what's the difference? or same?}
vector used in the proof of Theorem \ref{thm: critic_convergence_fisher}). Hence, we have $\sum_t \mathbb{E}[|| M^{1,i}_{t+1} - M^{1,i}_t||^2 ~|~ \mathcal{F}_{t,1}] < \infty$ a.s., and hence the martingale sequence $\{M^{1,i}_t\}$ converges a.s. \cite{neveu1975discrete}. 
%Thus, assumption 3 in X. \ref{ass: K-C_lemma} (given in Appendix \ref{app: K-C_lemma}) is satisfied. 
So, for any $\epsilon>0$, we have $\underset{t\rightarrow \infty}{lim}~ \mathbb{P}[sup_{p\geq t} || \sum_{\tau=t}^p \beta_{\theta,\tau}\xi^i_{\tau,1}|| \geq \epsilon] = 0$, as needed.

Regarding continuity of  $g^{1,i}(\theta_t) = -G(\theta^{i}_t)^{-1} \mathbb{E}_{s_t\sim d_{\theta_t}, a_t\sim \pi_{\theta_t}} ( \tilde{\delta}^i_{t} \psi^i_{t} | \mathcal{F}_{t,1})$ we note that, 

%\textcolor{black}{the above can replace this para}

%We now argue that $g^{1,i}(\theta_t) = -G(\theta^{i}_t)^{-1} \mathbb{E}_{s_t\sim d_{\theta_t}, a_t\sim \pi_{\theta_t}} ( \tilde{\delta}^i_{t} \psi^i_{t} | \mathcal{F}_{t,1})$ is continuous in $\theta^i_t$. To see this, note that
\begin{eqnarray*}
    g^{1,i}(\theta_t) &=& -G(\theta^{i}_t)^{-1} \mathbb{E}_{s_t\sim d_{\theta_t}, a_t\sim \pi_{\theta_t}} ( \tilde{\delta}^i_{t} \psi^i_{t} | \mathcal{F}_{t,1})
    \\
    &=& -G(\theta^{i}_t)^{-1} \sum_{s_t\in \mathcal{S}, a_t\in \mathcal{A}} d_{\theta_t}(s_t)\cdot \pi_{\theta_t}(s_t,a_t)\cdot \tilde{\delta}^i_{t,\theta_t} \cdot \psi^i_{t,\theta_t}. 
\end{eqnarray*}
Firstly, $\psi^i_{t,\theta_t}$ is continuous by assumption X. \ref{ass: regularity}. The term $d_{\theta_t}(s_t)\cdot \pi_{\theta_t}(s_t,a_t)$ is continuous in $\theta^i_t$ since it is the stationary distribution and solution to $d_{\theta_t}(s)\cdot \pi_{\theta_t}(s,a) = \sum_{s^{\prime}\in \mathcal{S}, a^{\prime} \in \mathcal{A}} P^{\theta_t}(s^{\prime}, a^{\prime} 
| s,a) \cdot d_{\theta_t}(s^{\prime})\cdot \pi_{\theta_t}(s^{\prime},a^{\prime}) $ and $\sum_{s\in \mathcal{S}, a\in \mathcal{A}}d_{\theta_t}(s)\cdot \pi_{\theta_t}(s,a) = 1$, where $P^{\theta_t}(s^{\prime}, a^{\prime} 
| s,a) = P(s^{\prime}| s,a) \cdot \pi_{\theta_t}(s^{\prime},a^{\prime})$. From assumption X. \ref{ass: regularity},  $\pi_{\theta_t}(s,a)>0$ and hence the above set of linear equations has a unique solution that is continuous in $\theta_t$ by assumption X. \ref{ass: regularity}.
% The unique solution to this set of linear equations can be verified to be continuous in $\theta_t$, noting that $\pi_{\theta_t}(s,a)>0$ by assumption X. \ref{ass: regularity}. 
Moreover, $\tilde{\delta}^i_{t,\theta_t}$ in Equation (\ref{delta-tilde}) is continuous in $\theta^i_t$ since $v_{\theta_t}$ as the unique solution to the linear equation $\varPhi^{\top} D_{\theta}^s[T_{\theta}^V(\varPhi v_{\theta}) - \varPhi v_{\theta}] = 0$ is  continuous in $\theta_t$.
%\textcolor{black}{these can be combined } Also,  $v_{\theta_t}$ is  continuous in $\theta_t$.
Thus, $g^{1,i}(\theta_t)$ is continuous in $\theta^i_t$, as needed in Kushner-Clark lemma.
%Now using Kushner-Clark lemma (see appendix \ref{app: K-C_lemma}) the update in Equation (\ref{eqn: actor_update_in_fisher_thm}) converges a.s to the set of asymptotically stable equilibria of ODE Equation (\ref{eqn: actor_ode}) for each $i\in N$. This concludes the proof. 

\end{proof} 

%\textcolor{black}{is the above system linear? }

Next, we provide the convergence of the AP-MAN algorithm. Recall, AP-MAN algorithm does not use the Fisher information inverse estimate; instead, it uses the linear function approximation of the advantage parameters. 

\subsection{Convergence of AP-MAN actor-critic algorithm}
\label{subsec: convergence_MANAC2}
The convergence of critic step, the reward parameters and the objective function estimate are the same as in Theorem \ref{thm: critic_convergence_fisher}. So, we show the convergence of advantage parameters and actor updates as given in the AP-MAN algorithm. Similar to the FI-MAN algorithm we again consider the transformed problem; rewards replaced with
costs. 
% \textcolor{black}{we will drop this sentence? This transformation, however, only affects the recursion given in Equation (\ref{eqn: adv_parameter_update_manac2}).} 
Thus the transformed recursion can be written as
\begin{equation}
\label{eqn: w_recursion_algo2_cost}
    {w}^i_{t+1} \leftarrow (I-\beta_{v,t} \psi^i_t \psi^{i^{\top}}_t)w^i_t - \beta_{v,t} \tilde{\delta}^i_t \psi^i_t.
\end{equation}

\begin{theorem}
\label{thm: adv_convergence}
Under the assumptions X. \ref{ass: feature_conds} and X. \ref{ass: g_inv}, for each agent $i\in N$, with actor parameters $\theta^i$, we have $w_t^i \rightarrow -G(\theta^i)^{-1} \mathbb{E}[\tilde{\delta}^i_{t,\theta} \psi^i_t]$ as $t\rightarrow \infty$ with probability one. 
\end{theorem}

\begin{proof}
\label{proof: adv_convergence}
For each $i\in N$, the ODE associated with the recursion in Equation (\ref{eqn: w_recursion_algo2_cost}) for a given $\theta$ is
\begin{equation}
\label{eqn: ode_w_rec_algo2}
\dot w^{i}= \mathbb{E}_{s_t\sim d_{\theta}, a_t \sim \pi_{\theta}} [-\psi^i_t \psi^{i^{\top}}_t w^i - \tilde{\delta}^i_{t,\theta} \psi^i_t] =:  g^{2,i}(w^i)
\end{equation}
Note that $g^{2,i}(w^i)$ is Lipschitz continuous in $w^i$. Let $g^{2,i}_{\infty}(w^i) = lim_{r\rightarrow \infty} \frac{g^{2,i}(rw^i)}{r}$. Note that $g^{2,i}_{\infty}(w^i)$ exists and satisfies  $g^{2,i}_{\infty}(w^i)= -G(\theta^i)w^i$. For the ODE $\dot w^{i}= -G(\theta^i)w^i $ the origin is an asymptotically stable equilibrium with $V_1(w^i) = w^{i^{\top}} w^i/2$ as the associated Lyapunov function (since $G(\theta^i)$ is positive definite). Now define a sequence $\{M^{2,i}_t\}$ as $M^{2,i}_t= \hat{g}^{2,i}(w^i_t)- \mathbb{E}[\hat{g}^{2,i}(w^i_t)|\mathcal{F}_{t,2}]$, where $\hat{g}^{2,i}(w^i_t)=-\psi^i_t \psi^{i^{\top}}_t w^i - \tilde{\delta}^i_t \psi^i_t $ and $\mathcal{F}_{t,2} = \sigma (w^i_r, M^{2,i}_r, r\leq t)$. Note that $\{M^{2,i}_t\}$ is the martingale difference sequence. Thus there exists a constant $C_0 < \infty$ such that 
\begin{equation*}
    \mathbb{E}[||M^{2,i}_{t+1}||^2 ~|~ \mathcal{F}_{t,2}] \leq C_0(1+||w^i_t||^2),~~ \forall~ t\geq 0.
\end{equation*}
For the ODE given in Equation (\ref{eqn: ode_w_rec_algo2}) consider the function $V_2(w^i)$ defined by
\begin{equation*}
    V_2(w^i) = \frac{1}{2}(w^i + G(\theta^i)^{-1} \mathbb{E}[\tilde{\delta}^i_{t,\theta} \psi^i_t])^{\top}(w^i + G(\theta^i)^{-1} \mathbb{E}[\tilde{\delta}^i_{t,\theta} \psi^i_t]).
\end{equation*}
Since $G(\theta^i)^{-1}$ is a positive definite matrix, we have
\begin{eqnarray*}
    \frac{dV_2(w^i)}{dt}
    &=& \nabla V_2(w^i)^{\top} \dot {w}^i
    \\
    &=& -(w^i +G(\theta^i)^{-1} \mathbb{E}[\tilde{\delta}^i_{t,\theta} \psi^i_t])^{\top} (G(\theta^i)w^i + \mathbb{E}[\tilde{\delta}^i_{t,\theta} \psi^i_t]) 
    \\
    &=& -(w^i + G(\theta^i)^{-1} \mathbb{E}[\tilde{\delta}^i_{t,\theta} \psi^i_t])^{\top} ~G(\theta^i) ~(w^i + G(\theta^i)^{-1} \mathbb{E}[\tilde{\delta}^i_{t,\theta} \psi^i_t])
    \\
    &<& 0, ~~ \forall~ w^i \neq -G(\theta^i)^{-1} \mathbb{E}[\tilde{\delta}^i_{t,\theta} \psi^i_t].
\end{eqnarray*}
Thus, $w(\theta^i)=  -G(\theta^i)^{-1} \mathbb{E}[\tilde{\delta}^i_{t,\theta} \psi^i_t] $ is an asymptotically stable equilibrium solution for the ODE given in Equation  (\ref{eqn: ode_w_rec_algo2}). Now, from the Theorem 2.2 of \cite{borkar2000ode} recursion (\ref{eqn: w_recursion_algo2_cost}) converges to $w(\theta^i)$ with probability one. This concludes the proof.

\end{proof}
We now consider the convergence of actor update of the AP-MAN algorithm.
\begin{theorem}
\label{thm: actor_convergence_advantage_parameter}
Under the assumptions X. \ref{ass: regularity} - X. \ref{ass: existence_local_min}, the sequence $\{\theta^i_t\}$ obtained from the actor step of the AP-MAN algorithm converges a.s. to  asymptotically stable equilibrium set of
\begin{equation}
\label{eqn: actor_ode_advantage_parameter}
    \dot{\theta}^i = \hat{\Gamma}^i[-G(\theta^{i})^{-1} \mathbb{E}_{s_t\sim d_{\theta}, a_t\sim \pi_{\theta}} ( \tilde{\delta}^i_{t,\theta} \psi^i_{t,\theta})],~~~ \forall~i\in N.
\end{equation}
\end{theorem}

\begin{proof}
\label{proof: MANAC-algo_2}
The proof is similar to that of Theorem \ref{thm: MANAC-algo_1}. So, we avoid writing many details here.
% to avoid the repetition.
Let $\mathcal{F}_{t,3} = \sigma (\theta_{\tau}, \tau \leq t)$ be the sigma field generated by $\{\theta_{\tau}\}_{\tau\leq t}$. The actor update in the AP-MAN algorithm with local projection is
\begin{equation}
\label{eqn: actor_update_in_AP_thm}
    \theta^i_{t+1} = \Gamma^i [\theta^i_t - \beta_{\theta,t}  \mathbb{E}_{s_t\sim d_{\theta_t}, a_t\sim \pi_{\theta_t}} (G(\theta^{i}_t)^{-1} \tilde{\delta}^i_{t} \psi^i_{t} | \theta^i_t) + \beta_{\theta,t} \xi_{t+1,3}^i],
\end{equation}
where $\xi_{t+1,3}^i = G(\theta^i_t)^{-1} \mathbb{E}_{s_t\sim d_{\theta_t}, a_t\sim \pi_{\theta_t}}( ( \tilde{\delta}^i_t -\tilde{\delta}^i_{t, \theta_t}) \cdot \psi^i_{t, \theta} ~|~ \mathcal{F}_{t,3}) = o(1)$, since critic converges at the faster time scale. Similar to Theorem \ref{thm: MANAC-algo_1}, we have  $g^{2,i}(\theta_t) = -G(\theta^{i}_t)^{-1} \mathbb{E}_{s_t\sim d_{\theta_t}, a_t\sim \pi_{\theta_t}} ( \tilde{\delta}^i_{t} \psi^i_{t} | \mathcal{F}_{t,3})$ that is continuous in $\theta^i_t$. The result follows by appealing to Kushner-Clark lemma.
% And all the conditions of Kushner-Clark lemma are satisfied in similar way as in Theorem \ref{thm: MANAC-algo_1}. So, the update in Equation (\ref{eqn: actor_update_in_fisher_thm}) converges a.s to the set of asymptotically stable equilibria of ODE given in Equation (\ref{eqn: actor_ode}) for each $i\in N$. This concludes the proof. 

\end{proof}

\subsection{Convergence of FIAP-MAN actor-critic algorithm}
\label{subsec: convergence_MANAC3}

The critic convergence, the convergence of reward parameters, and objective function estimate are the same as in Theorem \ref{thm: critic_convergence_fisher}. Thus, we provide proof for actor convergence only.
Like FI-MAN and AP-MAN algorithms, we again consider the transformed problem; rewards replaced with costs. 
%This transformation, however, only affects the recursion given in Equation (\ref{eqn: adv_parameter_update_manac3}).
Therefore, we consider the following recursion 
\begin{equation}
\label{eqn: adv_parameter_update_manac3_cost}
    w^i_{t+1} = (1 - \beta_{v,t})w^i_t - \beta_{v,t} G^{i^{-1}}_t \tilde{\delta}^i_t \psi^i_t.
\end{equation}

\begin{theorem}
\label{thm: adv_fisher_convergence}
Under the assumptions X. \ref{ass: feature_conds} and X. \ref{ass: g_inv}, for each agent $i\in N$, with actor parameters $\theta^i$, we have
% the $w_t^i, ~ t\geq 0$ given in Equation (\ref{eqn: adv_parameter_update_manac3_cost}) satisfy
$w_t^i \rightarrow -G(\theta^i)^{-1} \mathbb{E}[\tilde{\delta}^i_{t,\theta} \psi^i_t]$ as $t\rightarrow \infty$ with probability one. 
\end{theorem}

\begin{proof}
\label{proof: adv_fisher_convergence}
To prove this, we first note that for each $i\in N$,  $sup_{t,\theta^i, s_t,a^i_t} ||G_t^{i^{-1}} \tilde{\delta}^i_t \psi^i_t||< \infty$ with probability one. For the conditions given in Equation (\ref{eqn: robbins_monro_conds}) there exists $N_0$ such that for all $t \geq N_0$, $w^i_{t+1}$ is a convex combination of $w^i_t$ and a uniformly bounded quantity, $G^{i^{-1}}_t \tilde{\delta}^i_t \psi^i_t$. The overall sequence $w^i_t$ of iterates, for any $w^i_0 \in \mathbb{R}^{m_i}$ remains bounded with probability one. One can re-write the Equation (\ref{eqn: adv_parameter_update_manac3_cost}) as
\begin{equation*}
    {w}^i_{t+1} \leftarrow (1-\beta_{v,t}) w^i_t - \beta_{v,t} G(\theta^i)^{-1} \mathbb{E}[\tilde{\delta}^i_{t,\theta} \psi^i_t|\theta^i] - M^{3,i}_t + \beta_{v,t} \xi^i_{t+1,4}+ \beta_{v,t}\xi^i_{t+1,5},
\end{equation*}
where 
$M^{3,i}_t= \beta_{v,t}G(\theta^i)^{-1}(\tilde{\delta}^i_t\psi^i_t - \mathbb{E}[\tilde{\delta}^i_t\psi^i_t | \theta^i])$, $\xi^i_{t+1,4} = (G(\theta^i)^{-1} - G^{i^{-1}}_t )\tilde{\delta}^i_t\psi^i_t$, and $\xi^i_{t+1,5} = G(\theta^i)^{-1}\mathbb{E}[(\tilde{\delta}^i_{t,\theta} - \tilde{\delta}^i_t)\psi_{t}^{i} | \theta^i]$ respectively. From Theorem \ref{thm: critic_convergence_fisher} and \ref{thm: fisher_convergence_algo_1} both $\xi^i_{t+1,4}$ and $\xi^i_{t+1,5}$ are $o(1)$.  Note that  the sequence $\{\sum_{r=0}^{t-1} M^{3,i}_r\}$ is a convergent martingale sequence. Let $m_T = min \{m\geq p | \sum_{r=p}^m \beta_{v,r} \geq T \}$, then $\sum_{r=p}^{m_T} \beta_{v,r} G(\theta^i)^{-1}(\tilde{\delta}_r^i\psi_r^i - \mathbb{E}[\tilde{\delta}_r^i\psi_r^i | \theta^{i}_r]) \rightarrow 0$ almost surely as $p\rightarrow \infty$. Next, consider the ODE associated with Equation (\ref{eqn: adv_parameter_update_manac3_cost}),
\begin{equation}
\label{eqn: ode_w_rec_algo3}
    \dot w^{i}= - w^i - G(\theta^i)^{-1} \mathbb{E}[\tilde{\delta}^i_{t,\theta}\psi^i_t] =: g^{3,i}(w^i).
\end{equation}
Note that $g^{3,i}(w^i)$ is Lipschitz continuous in $w^i$, and hence the ODE (\ref{eqn: ode_w_rec_algo3}) is well-posed. Let $g^3_{\infty}(w^i) = lim_{r\rightarrow \infty} \frac{g^3(rw^i)}{r} =  - w^i$. 
For the ODE $\dot w^{i}= -w^i$ the origin is unique globally asymptotically stable equilibrium with $V_3(w^i) = w^{i^{\top}} w^i/2 $ as the associated Lyapunov function.  Similar to Theorem \ref{thm: adv_convergence} one can show that $w(\theta^i) = -G(\theta^i)^{-1} \mathbb{E}[\tilde{\delta}^i_{t,\theta} \psi^i_t] $ is an asymptotically stable attractor for the ODE given in Equation (\ref{eqn: ode_w_rec_algo3}). The proof follows from Theorem 2.2 of \cite{borkar2000ode}.

\end{proof}

\begin{theorem}
\label{thm: actor_convergence_advantage_parameter_fisher}
Under the assumptions X. \ref{ass: regularity} - X. \ref{ass: existence_local_min}  the sequence $\{\theta^i_t\}$ obtained from the actor step of the FIAP-MAN algorithm converges a.s. to asymptotically stable equilibrium set of the ODE
\begin{equation}
\label{eqn: actor_ode_advantage_parameter_fisher}
    \dot{\theta}^i = \hat{\Gamma}^i[-G(\theta^{i})^{-1} \mathbb{E}_{s_t\sim d_{\theta}, a_t\sim \pi_{\theta}} ( \tilde{\delta}^i_{t,\theta} \psi^i_{t,\theta})],~~~ \forall~i\in N.
\end{equation}
\end{theorem}
The proof of the above theorem is similar to the Theorem \ref{thm: actor_convergence_advantage_parameter}. We avoid writing it because of repetition.

\begin{remark}
\textcolor{black}{Though the ODEs corresponding to actor update in all MAN algorithms seem similar, we emphasize that they come from three different algorithms, each with a different critic update implicitly. Moreover, all the 3 MAN algorithms have their way of updating the Fisher information matrix or/and advantage parameters. Also, the objective function, $J(\theta)$, can have multiple stationary points and local optima. Thus all the three algorithms \textit{can potentially} attain different optima, and this was clearly illustrated in our comprehensive computational experiments in Section \ref{subsubsec: traffic_network_details}. Moreover, we also note that the optimal value obtained from FI-MAN and FIAP-MAN are close to each other, and AP-MAN and MAAC are also close. See also the discussion in  Sections \ref{subsec: algo_actor_comparison} and \ref{subsec: compare_variants}}. 
\end{remark}

To validate the usefulness of our proposed MAN algorithms, we implement them on a bi-lane traffic network and an abstract multi-agent RL model. The detailed computational experiments follow in the next section.

\section{Performance of algorithms in traffic network and abstract MARL models}
\label{sec: experiments}
This section provides comparative and comprehensive experiments in two different setups. Firstly, we model traffic network control as a multi-agent reinforcement learning problem. A similar model is available in \cite{shobhit} in a related but different context. Here, each traffic light is an agent, and the collective objective of these traffic lights is to minimize the average \emph{network congestion}. Another setup that we consider is an abstract multi-agent RL model with $15$ agents, $15$ states, and $2$ actions in each state. The model, algorithm parameters, rewards, and the features used in the linear function approximations are the same as given in \cite{dann2014policy,zhang2018fully}.

All the computations are done in python 3.8 on a machine equipped with 8 GB RAM and an Intel i5 processor. For the traffic network control, we use TraCI, i.e., ``Traffic Control Interface". TraCI uses a TCP-based client/server architecture to provide access to sumo-gui\footnote{The sumo-gui can be installed using: \url{https://www.eclipse.org/sumo/}} in python, thereby sumo act as a server \cite{krajzewicz2012recent}. For detailed documentation on TraCI, we refer to \url{https://sumo.dlr.de/docs/TraCI.html}. We now describe the traffic network and the arrival processes that we consider.

\subsection{Performance of algorithms for traffic network controls}
\label{subsubsec: traffic_network_details}
Consider the bi-lane traffic network as shown in Figure \ref{fig: traffic_network}. The network consists of $N_1, N_2, S_1,S_2, E_1, E_2, W_1$, and $W_2$ that act as the source and the destination nodes. $T_1, T_2, T_3$, and $T_4$ represents the traffic lights and act as agents. All the edges in the network are assumed to be of equal length. The agent's objective is to find a traffic signaling plan to minimize the overall network congestion. Note that the congestion to each traffic light is private information and hence not available to other traffic lights.
\begin{figure}[!ht]
	\centering
	\includegraphics[scale=0.65]{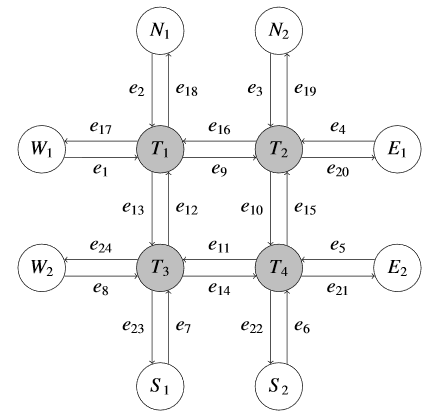}
	\caption{A bi-lane traffic network with four traffic lights $T_1,T_2,T_3,T_4$. All the other nodes $N_1, N_2, S_1,S_2, E_1, E_2, W_1$, and $W_2$ act as the source and destination nodes.}
	\label{fig: traffic_network}
\end{figure}

The sumo-gui requires the user to provide $T$, the total number of time steps the simulation needs to be performed, and $N_v$, the number of vehicles used in each simulation. As per the architecture of sumo-gui, vehicles arrive uniformly from the interval $\{1,2, \dots, T\}$. Once a vehicle arrives, it has to be assigned a source and a destination node. 
%At any time $t$, each node can observe high or low traffic. 
We assign the source node to each incoming vehicle according to various distributions. Different arrival patterns (ap) can be incorporated by considering different source-destination node assignment distributions. We first describe the assignment of the source node.
We use two different arrival patterns to capture high or low number of vehicles assigned to the source nodes in the network.  Let $p_{s,ap}$ be the probability that a vehicle is assigned a source node $s$ if arrival pattern is $ap$. Table \ref{table: prob_source node} gives probabilities, $p_{s,ap}$ for 2 arrival patterns ($ap \in \{1, 2\}$) that we consider. 
The destination node is sampled uniformly from the nodes except the source node. We assume that vehicles follow the shortest path from the source node to the destination node. However, if there are multiple paths with the same path length, then any one of them can be chosen with uniform probability. 

% \usepackage{multirow}
\begin{table}[h!]
\centering
\renewcommand{\arraystretch}{1.25}
\begin{tabular}{|c|c|c|}
\hline
\multirow{2}{*}{\textbf{Source node $(s)$}} & \multicolumn{2}{c|}{\textbf{Arrival pattern $(ap)$}} \\ \cline{2-3} 
 & \textbf{$ap=1$} & \textbf{$ap=2$} 
%  & \textbf{$m=2^{\prime}$}
\\ \hline \hline
$W_1$ & $\frac{3}{16}$ & $\frac{1}{28}$ 
% & $\frac{3}{14}$ 
\\ \hline
$W_2$ & $\frac{1}{16}$ & $\frac{1}{28}$ 
% & $\frac{3}{14}$
\\ \hline
$N_1$ & $\frac{1}{16}$ & $\frac{3}{14}$ 
% & $\frac{1}{28}$
\\ \hline
$N_2$ & $\frac{3}{16}$ & $\frac{3}{14}$ 
% & $\frac{1}{28}$ 
\\ \hline
$E_1$ & $\frac{3}{16}$ & $\frac{1}{28}$ 
% & $\frac{3}{14}$
\\ \hline
$E_2$ & $\frac{1}{16}$ & $\frac{1}{28}$ 
% & $\frac{3}{14}$
\\ \hline
$S_1$ & $\frac{1}{16}$ & $\frac{3}{14}$ 
% & $\frac{1}{28}$ 
\\ \hline
$S_2$ & $\frac{3}{16}$ & $\frac{3}{14}$ 
% & $\frac{1}{28}$ 
\\ \hline
\end{tabular}
\caption{Probability $p_{s,ap}$ for all the nodes and arrival patterns. $ap=1$ assigns high probabilities to the nodes, $N_2,S_2$ and $E_1,W_1$. $ap=2$ assigns high probabilities to all north and south nodes, $N_1,S_1,N_2,S_2$ and low probability to all east-west nodes, $E_1,E_2,W_1,W_2$.}
\label{table: prob_source node}
\end{table}

% one North-South, i.e., $N_2, S_2$ and one East-West, i.e., $E_1, W_1$ nodes with higher probabilities, thus putting highest congestion on traffic light $T_2$, almost equal congestion on $T_1,T_4$ and the least congestion on $T_3$. However, for $m=2$ all North-South nodes have high probability and all East-West nodes have low probability. Thus we expect that all the traffic lights will have almost the same congestion. $m=2^{\prime}$ is symmetric to $m=2$ and hence more vehicles will be assigned to East-West nodes.

In arrival pattern 1, i.e., $ap=1$, we have $p_{s,ap}$ higher for North-South nodes $N_2, S_2$, and East-West nodes $E_1,W_1$. Thus, we expect to see heavy congestion for traffic light $T_2$; almost same congestion for traffic lights $T_1$ and $T_4$; and the least congestion for traffic light $T_3$. For the arrival pattern 2, i.e., $ap=2$ more vehicles are assigned to all the North-South nodes.
% the arrival pattern $2^\prime$, i.e., $m=2^{\prime}$, which is symmetric to arrival pattern 2, has more vehicles assigned to all East and West nodes. 
So we expect that all the traffic lights will be equally congested.
% in both of these arrival patterns $(m=2,2^{\prime})$. 
For completeness, we now provide the distribution of the number of vehicles assigned to a source node $s$ at time $t$ for a given arrival pattern $ap$.

Let $N_t$ be the number of vehicles arrived at time $t$, and $N^s_t$ be the number of vehicles assigned to source node $s$ at time $t$. Thus, $N_t  = \sum_{s} N^s_t$.  Note that the arrivals are uniform in $\{0,1,\dots, T\}$, so $N_t$ is a binomial random variable with parameters $\left(N_v,\frac{1}{T} \right)$. Therefore, we have
\begin{align*}
\mathbb{P}(N_t = r) &= \binom{N_v}{r} \left(\frac{1}{T}\right)^r \left( 1-\frac{1}{T}\right)^{N_v-r},\quad \forall ~r=0,1,\cdots, N_v.
\end{align*}
Moreover, using the law of total probability, for all $ap~\in \{1,2\}$, we obtain
\begin{align}
\label{eqn: binomial}
\mathbb{P}(N^s_t = k ~|~ ap) &= \binom{N_v}{k} \left(\frac{p_{s,ap}}{T}\right)^k \left( 1-\frac{p_{s,ap}}{T}\right)^{N_v-k}, ~ \forall ~k=0,1,\cdots, N_v,
\end{align}
i.e., the distribution of $N_t^s$ for a given arrival pattern $ap$ is also binomial with parameters $\left(N_v,  \frac{p_{s,ap}}{T}\right)$. For more details see Appendix \ref{app: justification_eqn_40}.

In our experiments we take simulation time $T= 180000$ seconds which is divided into simulation cycle (called decision epoch) of $T_c = 120$ seconds each. Thus, there are 1500 decision epochs. The number of vehicles are taken as $N_v = 50000$. 
% (we also report our experimental results for lower number of vehicles, $N_v = 40000$. The details are available in \textcolor{black}{Appendix **}). 
We now describe the decentralized MARL framework for the traffic network shown in Figure \ref{fig: traffic_network}.

\subsubsection{Decentralized framework for traffic network control} 
% \textcolor{black}{apart from features, this subsection has many other details like signaling plans, etc. another title is better?}
%\textcolor{black}{Feature engineering for  MARL traffic network model??}
\label{subsubsub: marl_features}
In this section, we model the above traffic network control as a fully decentralized MARL problem with traffic lights as agents, $N = \{T_1,T_2, T_3, T_4\}$. Let $E_{in}$ denote the set of edges directed towards the traffic lights,
\begin{equation*}
E_{in} = \{e_1, e_2, e_{12}, e_{16}, e_3, e_4, e_9, e_{15}, e_8, e_7, e_{11}, e_{13}, e_5, e_6, e_{10}, e_{14}\}.
\end{equation*}

We assume that the maximum capacity of each lane in the network is $C = 50$. The state-space of the system consists of the number of vehicles in the lanes belonging to $E_{in}$.  Hence, the size of the state space is $50^{16}$.
% The total simulation time $T=180000 $ seconds is divided into simulation cycles of $T_c = 120$ seconds each. Thus, there are $1500$ decision epochs. 
At every decision epoch each traffic light follows one of the following traffic signal plans for the next $T_c = 120$ simulation steps.
\begin{enumerate}
	\item Equal green time of $\frac{T_c}{2}$ for both North-South and East-West lanes
	\item $\frac{3T_c}{4}$ green time for North-South and $\frac{T_c}{4}$ green time for East-West lanes
	\item $\frac{T_c}{4}$ green time for North-South and $\frac{3T_c}{4}$ green time for East-West lanes.
\end{enumerate}
Thus the total number of actions available at each traffic light is $3^4 = 81$. The rewards given to each agent is equal to the negative of the average number of vehicles stopped at its corresponding traffic light. Note that the rewards are privately available to each traffic light only. We aim to maximize the expected time average of the globally averaged rewards, which is equivalent to minimize the (time average of) number of stopped vehicles in the system. Since the state space is huge $(50^{16})$, we use the linear function approximation for the state value function and the reward function. The approximate state value for state $s$ is $V(s;v) = v^{\top} \varphi(s)$, where $\varphi(s) \in \mathbb{R}^{L},~ L << |\mathcal{S}|$, is the feature vector for the state $s$. Moreover, the reward function is approximated as ${R}(s,a;\lambda) = \lambda^{\top} f(s,a)$ where $f(s,a)\in \mathbb{R}^M,~ M<< |\mathcal{S}||\mathcal{A}|$ are the features associated with each state-action pair $(s,a)$. Next, we describe the features we are using in the experiments \cite{shobhit}. %\textcolor{black}{shobhit thesis, etc.}

Let $x_t^i$ denote the number of vehicles in lane $e_i \in E_{in}$ at time $t$. We normalize $x_t^i$ via maximum capacity of a lane, $C$ to obtain $z_t^i = x_t^i / C$. We define $\xi(s)$ as a vector having components containing $z_t^i$, as well as components with products of two or three $z_t^i$'s. The product terms are of the form $z_t^i z_t^j$ and $z_t^i z_t^j z_t^l$ where all terms in the product correspond to the same traffic light. 
\begin{align*}
\xi(s) &= \left(\begin{matrix}z_t^1, & z_t^2, & \cdots, & z_t^{16}, & z_t^1 z_t^2, & \cdots, & z_t^6 z_t^5, & z_t^1 z_t^2 z_t^{12}, & \cdots, & z_t^5 z_t^6 z_t^{10} \end{matrix} \right)
\end{align*}
The feature vector $\varphi(s)$ is defined as having all the components of $\xi(s)$ along with an additional bias component, $1$. Thus, $\varphi(s) = \left(\begin{matrix} 1, & \xi(s)\end{matrix} \right)^{\top}$. The length of $\xi(s)$ is $16+4\times (4^2 + 4^3) = 336$. Hence, the length of the feature vector, $\varphi(s)$ is $L=1+336=337$. For each agent $i\in N$, we parameterize the local policy $\pi^i(s,a^i)$ using the Boltzmann distribution as follows: 
\begin{equation}
\label{eqn: gibbs_policy}
\pi^i_{\theta^i}(s,a^i) = \frac{\exp(q^{\top}_{s,a^i} \cdot \theta^i)}{\sum_{b^i\in \mathcal{A}^i} \exp(q^{\top}_{s,b^i}\cdot\theta^i)},
\end{equation}
where $q_{s,b^i}\in \mathbb{R}^{m_i}$ is the feature vector of dimension same as $\theta^i$, for any $s \in \mathcal{S}$ and $b^i\in \mathcal{A}^i$, for all $i\in N$. 
% Define the binary variables $a^i_{c,d}$ corresponding to action $a^i$ as:
% \begin{align*}
% a^i_{c,d} &= \begin{cases}1 \quad \text{signal plan $d$ is selected for traffic light $c$ in action $a^i$}\\ 0 \quad \text{otherwise}\end{cases} \quad \forall c=1,2,3,4 \text{ and } d = 1,2,3
% \end{align*}
% \textcolor{red}{To write the features $q_{s,a^i}$. Eligibility traces?}
The feature vector $q_{s,a^i}$ for each $i\in N$ is given by:
\begin{align*}
q_{s,a^i} &= \left(\begin{matrix}1, & a^{i,1} \xi(s), & a^{i,2} \xi(s), & a^{i,3} \xi(s) \end{matrix}\right)^{\top},
\end{align*}
where $\xi(s)$ is defined as earlier, and $a^{i,j}$ are the binary numbers defined as:
\begin{align*}
a^{i,j} &= \begin{cases}1 \quad \text{if signal plan $j$ is selected in action $a^i$} \\ 0 \quad \text{otherwise},\end{cases} ~~ \forall~ i\in N.
\end{align*}
The length of $q_{s,a^i}$, i.e., $m_i = 3\times 336+1 = 1009$. For the Boltzmann policy function $\pi^i_{\theta^i}(s,a^i)$ we have
% have\cite{sutton1999policy} 
% \textcolor{black}{we can recall the  policy gradient theorem ** \ref{thm: pgt}, write  the mean adjusted interpretation and cite Rick Sutton et al paper? also, write that this $\psi(s, a) = \nabla_{\theta_i} \log \pi^i_{\theta_i}(s, a^i) = ... $ is compatible feature, etc.??} 
%\cite{bhatnagar2009natural},
\begin{equation*}
\nabla_{\theta^i} \log \pi^i_{\theta^i} (s, a^i)  = q_{s,a^i} - \sum_{b^i\in \mathcal{A}^i} \pi^i_{\theta^i} (s, b^i) q_{s,b^i},
\end{equation*}
where $\nabla_{\theta_i} \log \pi^i_{\theta_i}(s, a^i)$ are the compatible features as in the policy gradient Theorem \cite{sutton1999policy}. Note that the compatible features are same as the features associated to policy, except normalized to be mean zero for each state.
% \textcolor{black}{to add more? features are zero mean centered?}. 
The features $f(s,a)$ for the rewards function are taken in a  similar way as $q_{s,a^i}$ for each $i\in N$, thus $M = 4\times3\times336 + 1 = 4033$. 

We implement all the 3 MAN algorithms and compared the average network congestion with the MAAC algorithm. For all $i\in {N}$, the initial value of parameters $\mu^i_0$, 
$\tilde{\mu}^i_0$, $v^i_0$, $\tilde{v}^i_0$, $\lambda^i_0$, $\tilde{\lambda}^i_0$, $\theta^i_0$, $w^i_0$ are taken as zero vectors of appropriate dimensions. The Fisher information matrix inverse,  $G_0^{i^{-1}}$ is initialized to $I, ~ \forall~ i\in N$. The critic and actor step sizes are taken as $\beta_{v,t} = \frac{1}{(t+1)^{0.65}}$, and $\beta_{\theta,t} = \frac{1}{(t+1)^{0.85}}$ respectively. Note that these step sizes satisfy the conditions given in Equation (\ref{eqn: robbins_monro_conds}). 
We assume that the communication graph $\mathcal{G}_t$ is a complete graph at all time instances $t$ and the consensus matrix is constant with $c_t(i,j)=\frac{1}{4}$ for all pairs $i,j$ of agents. Although we do not use the eligibility traces in the convergence analysis, we use them ($\lambda = 0.25$ for TD($\lambda$) \cite{sutton2018reinforcement}) to provide better performance in case of function approximations.
We believe that the convergence of MAN  algorithms while incorporating eligibility traces is easy to follow, so we avoid them here. We now provide the performance of all the algorithms in both the arrival patterns.

\subsubsection{Performance of traffic network for arrival pattern 1}
\label{subsubsec: diss_com_results}

% We implement all the 3 MAN and MAAC algorithms for both the arrival patterns 
% as given in Table \ref{table: prob_source node}.
% We will first consider the arrival pattern 1 for assigning the source node to incoming vehicles.

% \subsubsection*{Arrival pattern 1; $T$ = 180000; and $N_v$ = 50000}

% For last 200 decision epochs
Recall, for arrival pattern 1 we assign high probability $p_{s,ap}$ to the source nodes $N_2,S_2$ and $E_1,E_2$ and low probability to other source nodes.  Table \ref{table: network_50K_pattern_1} provides the average network congestion (averaged over 10 runs, and round off to 5 decimal places), standard deviation and 95\% confidence interval. 

\begin{table}[h!]
\centering
\setlength{\tabcolsep}{2.5pt}
\begin{tabular}{|c|c|c|c|}
\hline
\textbf{Algorithms} & \multicolumn{1}{c|}{\textbf{\begin{tabular}[c]{@{}c@{}}Avg network \\ congestion\end{tabular}}} & \multicolumn{1}{c|}{\textbf{Standard Deviation}} & \multicolumn{1}{c|}{\textbf{Confidence Interval (95\%)}} \\ \hline \hline
\textbf{MAAC} & 14.01687 & 0.08405 & (13.96478, 14.06896) \\ \hline
\textbf{FI-MAN} & 12.02819 & 1.48071 & (11.11045, 12.94593) \\ \hline
\textbf{AP-MAN} & 14.07899 & 0.08266 & (14.02776, 14.13022) \\ \hline
\textbf{FIAP-MAN} & 11.28657 & 1.04137 & (10.64113, 11.93201) \\ \hline
\end{tabular}
\caption{Average network congestion, standard deviation and 95\% confidence interval for arrival pattern 1. FIAP-MAN has $\approx $ 18\%, and FI-MAN has $\approx$ 14\% less congestion than MAAC algorithm. The congestion for the AP-MAN algorithm is almost the same as the MAAC algorithm with high confidence.
% performance of AP-MAN is almost similar to MAAC 
%\textcolor{blue}{can end the sentence here? }
% because it never uses the Fisher information matrix.
% \textcolor{blue}{we should try saying something like `congestion values within the top $15 \%$ are in bold'. also, this $15 \%$ should be used throughout the paper -- hopefully, it will work for other experiments also} Expectedly, both MAAC and AP-MAN have reached steady state values faster than other two algorithms. \textcolor{blue}{i wrote this earlier; not sure, if it is to be in caption; the point, is that FI-MAN would have lower congestion had it reached steady-state -- the way sumo works, we can't simulate this as 2000 decision epochs with same 50k vehicles is completely another system as these 50k are now placed uniformly over the 2000 epochs } FI-MAN and FIAP-MAN almost (upto 1st decimal) are on decreasing trend w.r.t. average network congestion however this decay is very slow (0.1 in 200 epochs).
}
\label{table: network_50K_pattern_1}
\end{table}

We observe an $\approx$ 18\% reduction in average congestion for FIAP-MAN and $\approx$ 14\% reduction for FI-MAN algorithms compared to the MAAC algorithm. These observations are theoretically justified in Section \ref{subsec: algo_actor_comparison}.

\textcolor{black}{To show that these algorithms have attained the steady-state, we provide average congestion,  and the correction factor (CF), i.e., the 95\% confidence value  which is defined as $CF = 1.96 \times \frac{std~dvn}{\sqrt{10}}$ for last 200 decision epochs in Table \ref{table: 50K_each_light_pattern_1}  of  Appendix \ref{app: 50k_pattern_1}. We observe that average network congestion for the MAN algorithms are almost (up to 1st decimal) on decreasing trend w.r.t. network congestion; however, this decay is prolonged (0.1 fall in congestion in 200 epochs), suggesting that algorithms are converging to a local minimum. }

% The average congestion for the AP-MAN algorithm is on par with MAAC.
% Performance of AP-MAN did not improve, possibly because it does not explicitly estimate the Fisher information matrix inverse. 
Thus, we see that algorithms involving the natural gradients dominate those involving standard gradients. Figure \ref{fig: cum_reward_mode1_50k} shows the (time) average network congestion for single run (thus lower the better). 
\begin{figure}[h!]
    \centering
    \includegraphics[width = 0.6\textwidth]{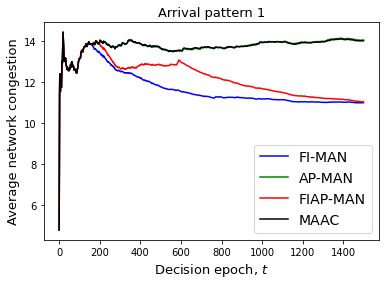}
\caption{(Time) average network congestion for all the algorithms with arrival pattern 1. The congestion is least for FIAP-MAN and FI-MAN algorithms. However, MAAC and AP-MAN algorithms have almost the same congestion.
% , which is $\approx$ 18\% more than FIAP-MAN and $\approx$ 14\% more than FI-MAN. 
Thus 2 (FI-MAN and FIAP-MAN) of the $3$ natural gradient-based algorithms find better optima, and the third one (AP-MAN) is on par with the MAAC algorithm. Moreover, for a few initial decision epochs $\approx$ 180, all the algorithms have almost the same performance, but afterward, they find different directions and end in different optima.}
\label{fig: cum_reward_mode1_50k}
\end{figure}

For almost $180$ decision epochs, all the algorithms have the same (time) average network congestion. However, after 180 decision epochs, FI-MAN and FIAP-MAN follow different paths and hence find different local minima as shown in Theorem \ref{thm: J_comp_t+1}.
We want to emphasize that the Theorem \ref{thm: J_comp_t+1} is for maximization framework. As mentioned above in Section \ref{subsubsub: marl_features}, we are also maximizing the globally average rewards, which is equivalent to minimizing the (time average of) number of stopped vehicles. 

\subsubsection*{Actor parameter comparison for arrival pattern 1 }

Recall, in Theorem \ref{thm: J_comp_t+1} we show that under some conditions, $J(\tilde{\theta}^N_{t+1}) \geq J(\tilde{\theta}^M_{t+1})$ for all $t\geq t_0$,
% if $\theta^N_t = \theta^M_t$. However, we observe that $\lambda_{\max}(G(\theta^N_t))<1$ for all $t\geq 2$. Therefore, we have $J(\tilde{\theta}^N_{t+1}) \geq J(\tilde{\theta}^M_{t+1})$ for all $t\geq 2$,
and hence at each iterate the average network congestion in FI-MAN, and FIAP-MAN algorithms are better than MAAC algorithm. To investigate this further, we plot the norm of difference of the actor parameter of all the 3 MAN algorithms with MAAC algorithm for each agent. For traffic light $T_1$ (or agent 1) these differences are shown in Figure \ref{fig: norm_diffe_agent_1} (for other agents see Figure \ref{fig: norm_diff_other_agents_50k_pattern_1} in Appendix \ref{app: 50k_pattern_1}).

\begin{figure}[h!]
\centering
\begin{subfigure}[b]{0.48\textwidth}
\includegraphics[width = \linewidth]{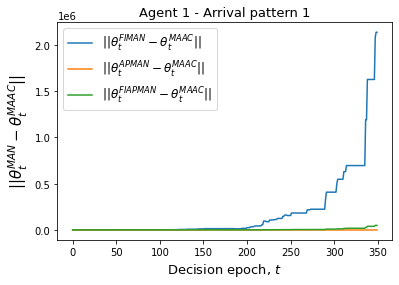}
\caption{}
\end{subfigure}
\begin{subfigure}[b]{0.49\textwidth}
\includegraphics[width = \linewidth]{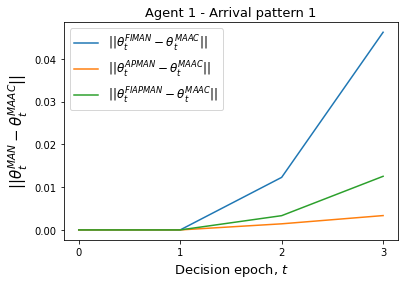}
\caption{}
\end{subfigure}
\caption{Norm of difference in the actor parameter of agent 1 for all the 3 MAN algorithms with MAAC algorithm for arrival pattern 1. Figure (a) is shown for 350 decision epochs; to show that the differences in the actor parameter are from decision epoch 2 itself, we zoom it in figure (b) in the left panel. However, the significant differences are observed only after $\approx$ 180 epochs. This illustrates Theorem \ref{thm: J_comp_t+1} and related discussions in Section \ref{subsec: algo_actor_comparison}.}
\label{fig: norm_diffe_agent_1}
\end{figure} 

We observe that all the 3 MAN algorithms pick up $\theta_2$ (i.e., the actor parameter at decision epoch 2) that is different from that of the MAAC scheme at varying degrees, with FI-MAN being a bit more `farther.' However, a significant difference is observed around decision epoch $\approx 180$. For better understanding, the same graphs are also shown in the logarithmic scale for agent 1 and agent 2 with arrival pattern 1 in Figure \ref{fig: norm_diff_log_scale_agent_1_2_mode_1}. 

\begin{figure}[h!]
\centering
\begin{subfigure}[b]{0.45\textwidth}
\includegraphics[width = \textwidth]{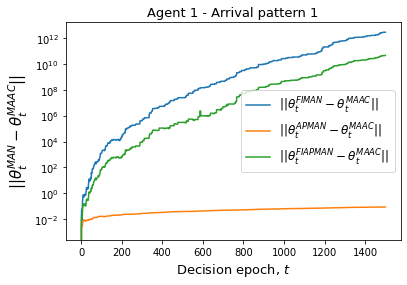}
\end{subfigure}
\begin{subfigure}[b]{0.45\textwidth}
\includegraphics[width = \textwidth]{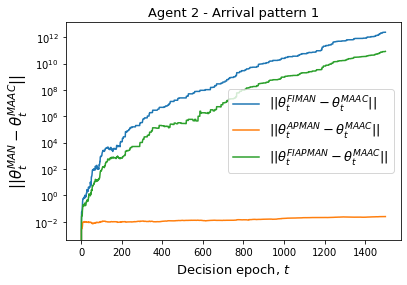}
\end{subfigure}
\caption{Norm of differences of the actor parameters for agents 1 and 2 with traffic arrival pattern 1 in logarithmic scale illustrating Theorem \ref{thm: J_comp_t+1} and related discussions in Section \ref{subsec: algo_actor_comparison}.}
% For details see Section \ref{subsec: algo_actor_comparison}.}
\label{fig: norm_diff_log_scale_agent_1_2_mode_1}
\end{figure}

We see that the norm difference is linearly increasing in FI-MAN and FIAP-MAN algorithms, whereas it is almost flat for the AP-MAN algorithm. So, the iterates of these $2$ algorithms are exponentially separating from those of the MAAC algorithm.
This again substantiates our analysis in Section \ref{subsec: algo_actor_comparison}. 

Though we aim to minimize the network congestion, in Table \ref{table: each_light_mode_1_last_epoch_only} we also provide the average congestion and the correction factor (CF) to each traffic light for last decision epoch (Table \ref{table: 50K_each_light_pattern_1} in Appendix \ref{app: 50k_pattern_1} shows these values for last 200 decision epochs). Expectedly, in all the algorithms, the average congestion for traffic light $T_2$ is highest; almost same for traffic lights $T_1,T_4$; and least for traffic light $T_3$.
% Because of congestion to each traffic light the overall network is also congested, which is reflected in Table \ref{table: network_50K_pattern_1}. 

% \usepackage{multirow}
\begin{table}[h!]
\centering
\begin{tabular}{|c|c|c|c|c|}
\hline
\multirow{2}{*}{\textbf{Algorithms}} & \multicolumn{4}{c|}{\textbf{Congestion (Avg $\pm$ CF)}} \\ \cline{2-5} 
 & \textbf{$T_1$} & \textbf{$T_2$} & \textbf{$T_3$} & \textbf{$T_4$} \\ \hline \hline
\textbf{MAAC} & 3.733  $\pm$  0.067 & 4.336 $\pm$ 0.050 & 2.249 $\pm$ 0.013 & 3.699 $\pm$ 0.040  \\ \hline
\textbf{FI-MAN} & 2.613  $\pm$  0.110 & 4.492 $\pm$ 0.868 & 2.359 $\pm$ 0.161 & 2.564 $\pm$ 0.038  \\ \hline
\textbf{AP-MAN} & 3.748  $\pm$  0.073 & 4.338 $\pm$ 0.046 & 2.247 $\pm$ 0.013 & 3.746 $\pm$ 0.039  \\ \hline
\textbf{FIAP-MAN} &  2.711  $\pm$  0.214 & 3.785 $\pm$ 0.371 & 1.907 $\pm$ 0.148 & 2.883 $\pm$ 0.147\\ \hline
\end{tabular}
\caption{The average congestion and correction factor (CF) for each traffic light for arrival pattern 1 (described in Section \ref{subsubsec: diss_com_results}). CF is defined as $1.96 \times \frac{std~dvn}{\sqrt{10}}$.}
\label{table: each_light_mode_1_last_epoch_only}
\end{table}

\subsubsection{Performance of traffic network for arrival pattern 2}
\label{subsubsec: arrival_pattern_2_performance}

Recall, in arrival pattern 1 nodes $N_2,S_2$ and $E_1,W_1$ are assigned more $p_{s,ap}$. We now change the $p_{s,ap}$'s and consider arrival pattern 2. Here the traffic origins $N_1, N_2, S_1$ and $S_2$ have higher probabilities of being assigned a vehicle. 
Since $p_{s,ap}$ for all these nodes is $\frac{3}{14}$, and for all other nodes it is $\frac{1}{28}$, we expect that all the traffic lights will have almost the same average congestion. This observation is reported in Appendix \ref{app: 50k_pattern_2}. Table \ref{table: network_50K_pattern_2} present the average network congestion (averaged over $10$ runs, and round off to 5 decimal places), standard deviation and 95\% confidence interval for arrival pattern 2.   

\begin{table}[h!]
\setlength{\tabcolsep}{2pt}
\centering
\begin{tabular}{|c|c|c|c|}
\hline
\textbf{Algorithms} & \textbf{\begin{tabular}[c]{@{}c@{}}Avg network \\ congestion\end{tabular}} & \textbf{Standard deviation} & \textbf{Confidence Interval  (95\%)} \\ \hline \hline
\textbf{MAAC} & 13.64571 & 0.19755 & (13.52327, 13.76815) \\ \hline
\textbf{FI-MAN} & 10.16988 & 0.11877 & (10.09627, 10.24349) \\ \hline
\textbf{AP-MAN} & 13.77573 & 0.18925 & (13.65843, 13.89303) \\ \hline
\textbf{FIAP-MAN} & 10.19858 & 0.21248 & (10.06689, 10.33027) \\ \hline
\end{tabular}
\caption{The average network congestion, standard deviation and 95\% confidence interval at last decision epoch for arrival pattern 2 (described in Section \ref{subsubsec: arrival_pattern_2_performance}). FI-MAN and FIAP-MAN algorithms has $\approx $ 25\% less average network congestion than MAAC algorithm. The performance of AP-MAN algorithm is almost similar to MAAC algorithm as shown in Theorem \ref{thm: J_comp_t+1} and in Section \ref{subsec: algo_actor_comparison}.}
\label{table: network_50K_pattern_2}
\end{table}

We observe an $\approx$ 25\% reduction in the average congestion with FI-MAN and FIAP-MAN algorithms as compared to the MAAC algorithm. AP-MAN is on par with the MAAC algorithm. This again shows the usefulness of the natural gradient-based algorithms. As opposed to arrival pattern $1$ where FIAP-MAN algorithm has slightly better performance than FI-MAN algorithm, in arrival pattern $2$, both algorithms have almost similar average network congestion. Moreover, the standard deviation in arrival pattern 1 is much higher than in arrival pattern 2. 
% This is possibly because with arrival pattern 2, all the algorithms will assign almost the same congestion to each traffic light with high confidence \textcolor{red}{(more details in appendix ??)}.
Figure \ref{fig: cum_reward_mode2_50k} shows the (time) average congestion for single simulation run of all the algorithms.

\begin{figure}[h!]
    \centering
    \includegraphics[width = 0.6\textwidth]{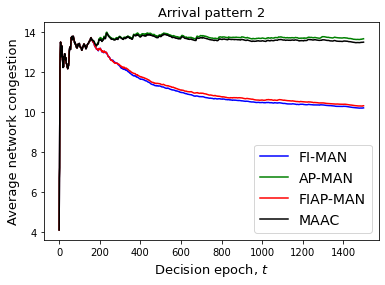}
\caption{(Time) average network congestion for arrival pattern 2. Again, for few initial decision epochs $\approx$ 180 all the algorithms have almost same performance, but afterwards they find different directions, and hence ends in different optima.}
\label{fig: cum_reward_mode2_50k}
\end{figure}

\subsubsection*{Actor parameter comparison for arrival pattern 2 }

Similar to arrival pattern 1, in Figure \ref{fig: norm_diffe_agent_1_mode_2} we plot the norm of the difference for traffic light $T_1$ (for other traffic lights (agents) see Figure \ref{fig: norm_diff_other_agents_50k_pattern_2} of Appendix \ref{app: 50k_pattern_2}). For better understanding, the same graphs are also shown in the logarithmic scale for agent 1 and agent 2 in Figure \ref{fig: norm_diff_log_scale_agent_1_2_mode_2}. 
% (for other agents see \textcolor{red}{Appendix}). 
Again, we see that the norm difference is linearly increasing in FI-MAN and FIAP-MAN algorithms, whereas it is almost flat for the AP-MAN algorithm. So, the iterates of these $2$ algorithms are exponentially separating from the MAAC algorithm.
This again substantiates our analysis in Section \ref{subsec: algo_actor_comparison}. Moreover, we also compute the average congestion and (simulation) correction factor for each traffic light. Table \ref{table: each_light_mode_2_last_epoch_only} shows these values for last decision epoch (See Table \ref{table: 50K_each_light_pattern_2} for last 200 decision epochs). As expected, the average congestion to each traffic light is almost the same.  

% \usepackage{multirow}
\begin{table}[h!]
\centering
\begin{tabular}{|c|c|c|c|c|}
\hline
\multirow{2}{*}{\textbf{Algorithms}} & \multicolumn{4}{c|}{\textbf{Congestion (Avg $\pm$ CF)}} \\ \cline{2-5} 
 & \textbf{$T_1$} & \textbf{$T_2$} & \textbf{$T_3$} & \textbf{$T_4$} \\ \hline \hline
\textbf{MAAC} & 3.481 $\pm$ 0.070 & 3.379 $\pm$ 0.034 & 3.348 $\pm$ 0.032 & 3.438 $\pm$ 0.033 \\ \hline
\textbf{FI-MAN} & 2.532 $\pm$ 0.047 & 2.530 $\pm$ 0.040 & 2.519 $\pm$ 0.045 & 2.589 $\pm$ 0.045 \\ \hline
\textbf{AP-MAN} & 3.511 $\pm$ 0.064 & 3.404 $\pm$ 0.032 & 3.376 $\pm$ 0.034 & 3.485 $\pm$ 0.033 \\ \hline
\textbf{FIAP-MAN} &   2.566 $\pm$ 0.041 & 2.534 $\pm$ 0.034 & 2.543 $\pm$ 0.041 & 2.556 $\pm$ 0.036 \\ \hline
\end{tabular}
\caption{The average congestion and correction factor (CF) for each traffic light for arrival pattern 2 (described in Section \ref{subsubsec: arrival_pattern_2_performance}). CF is defined as $1.96 \times \frac{std~dvn}{\sqrt{10}}$. For more details see Table \ref{table: 50K_each_light_pattern_2}.}
\label{table: each_light_mode_2_last_epoch_only}
\end{table}

% \textcolor{black}{we can have another subsection and talk about the congestion at the individual traffic lights, for both traffic intensity, $N_v$, as minimizing individual agent's congestion is not the aim of these algos. this is interesting, but, we need to present it carefully. }

\begin{figure}[h!]
\centering
\begin{subfigure}[b]{0.48\textwidth}
\includegraphics[width = \linewidth]{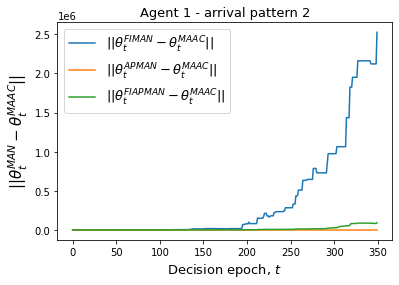}
\caption{}
\end{subfigure}
\begin{subfigure}[b]{0.49\textwidth}
\includegraphics[width = \linewidth]{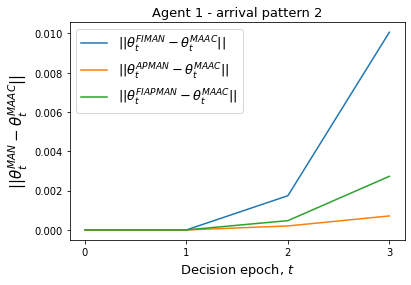}
\caption{}
\end{subfigure}
\caption{Norm of difference in the actor parameter of agent 1 for all the 3 MAN algorithms with MAAC algorithm. Figure (a) is shown for 350 decision epochs; to show that the differences in the actor parameter are from decision epoch 2 itself, we zoom it in Figure (b). However, the significant differences are observed only after $\approx$ 180 epochs. This illustrates Theorem \ref{thm: J_comp_t+1} and related discussions in Section \ref{subsec: algo_actor_comparison}.}
\label{fig: norm_diffe_agent_1_mode_2}
\end{figure}

\begin{figure}[h!]
\centering
\begin{subfigure}[b]{0.49\textwidth}
\includegraphics[width = \textwidth]{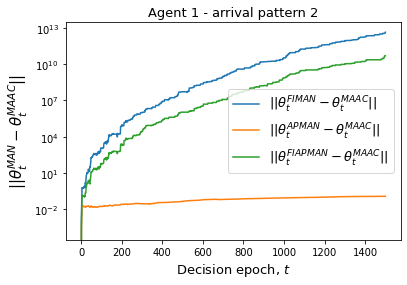}
\end{subfigure}
\begin{subfigure}[b]{0.49\textwidth}
\includegraphics[width = \textwidth]{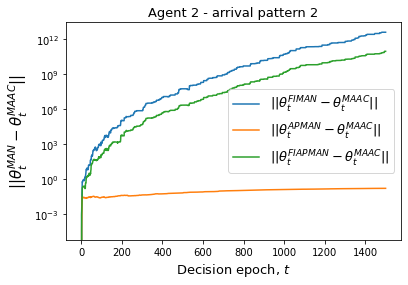}
\end{subfigure}
\caption{Norm of differences of the actor parameter for agents 1 and 2 with traffic arrival pattern 2 in logarithmic scale illustrating Theorem \ref{thm: J_comp_t+1} and related discussions in Section \ref{subsec: algo_actor_comparison}.}
% For details see Section \ref{subsec: algo_actor_comparison}.}
\label{fig: norm_diff_log_scale_agent_1_2_mode_2}
\end{figure}

% Note that we also consider another arrival pattern $2^\prime$ and \textcolor{red}{Appendix **} reports the congestion levels of all the 4 algorithms for arrival pattern $2^{\prime}$. The arrival pattern $2^{\prime}$ is such that vehicles are being assigned now with the same higher probabilities, but  along East-West directions. In view of these symmetries, the congestion levels in the network and to each traffic lights for arrival pattern $2^{\prime}$ are almost same as those of Arrival pattern 2. 

% -------

% \textcolor{black}{FIAP-MAN and FI-MAN have essentially similar performance when the number of vehicles $N_v = 40000$; congestion induced by FIMAN-MAN is  marginally lower, differing in just second decimal places.}

% \textcolor{red}{A small remark can be written for 40000 vehicles, and a table of average rewards can also be given??}
% We also provide the average congestion for $N_v = 40000, T=180000$ seconds in Appendix **. We expect that average congestion will be smaller for all algorithms. 

We now present another computational experiment where we consider an abstract MARL with $n=15$ agents. The model, algorithm parameters, including transition probabilities, rewards, and features for state value function, and rewards are the same as given in \cite{dann2014policy,zhang2018fully}

% \textcolor{black}{FIAP-MAN has least congestion when the number of vehicles is higher at $N_v = 50000$; lower than that of FI-MAN by  $\approx 0.7$ or $ \%$; first setting when FIAP-MAN is doing better than FI-MAN?? ALso, at this volume of $N_v = 50000$ vehicles in the system, network needs more than 1500 epochs to reach steady-state?? }

% \textcolor{black}{Even though FIAP-MAN does better than FI-MAN as far as system wide total congestion is concerned, it is not that FIAP-MAN does better than FI-MAN at each traffic light $T_i, i \in \{1,\cdots, 4\}$ at this $N_v = 50000$ vehicles; while congestion at traffic lights $T_3$ and at the bottle-neck node $T_2$ has improved from that offered by FI-MAN, it has increased marginally at $T_1$ and $T_4$.}

\subsection{Performance of algorithms in abstract MARL model %\textcolor{black}{Garnet?}
}
\label{subsec: garnet_experiments}
The abstract MARL model that we consider consists of $n=15$ agents and $|\mathcal{S}| = 15$ states. Each agent $i\in N$ is endowed with the binary valued actions $\mathcal{A}^i \in \{0,1\}$. Therefore, the total number of actions are $2^{15}$. Each element of the transition probability is a random number uniformly generated from the interval $[0,1]$. These values are normalized to incorporate the stochasticity. To ensure the ergodicity we add a small constant $10^{-5}$ to each entry of the transition matrix. The mean reward $R^i(s,a)$ are sampled uniformly from the interval $[0,4]$ for each agent $i\in N$, and for each state-action pair $(s,a)$. The instantaneous rewards $r^i_t$ are sampled uniformly from the interval $[R^i(s,a)-0.5, R^i(s,a) + 0.5]$. We parameterize the policy using the Boltzmann distribution as given in Equation (\ref{eqn: gibbs_policy}) with $m_i = 5, ~ \forall ~i\in N$.
All the feature vectors (for the state value and the reward functions) are sampled uniformly from the set $[0,1]$ of suitable dimensions. More details of the multi-agent model are available in Appendix \ref{supp: MARL_Garnet}.

We compare all the 3 MAN actor-critic algorithms with the MAAC algorithm. We observe that the globally averaged return from all the algorithms is almost close to each other (for details, see Table \ref{table: avg_rewards_garnet} of Appendix \ref{supp: MARL_Garnet}). To provide more details, we compute the relative $V$ values for each agent $i\in N$ that is defined as $V_{\theta}(s;v^i) = v^{i^\top} \varphi(s)$; this way of comparison of MARL algorithms was earlier  used by \cite{zhang2018fully}. Thus, the higher the value, the better is the algorithm. The relative values corresponding to the FI-MAN and FIAP-MAN algorithms are more than the MAAC algorithm almost for all agents, suggesting again that the natural gradient-based algorithms are better than the standard gradient ones (as far as the relative values are concerned). The detailed observations are available in Appendix \ref{supp: MARL_Garnet}.

\section{Related work}
\label{sec: related_work_manac}
Reinforcement learning has been extensively studied and explored by researchers because of its varied applications and usefulness in many real-world applications  \cite{fax2004information,corke2005networked,dall2013distributed}. Single-agent  reinforcement learning models are well explained in many works including \cite{bertsekas1995dynamic,sutton2018reinforcement,bertsekas2019reinforcement}. The average reward reinforcement learning algorithms are available in \cite{mahadevan1996average,dewanto2020average}.

Various algorithms to compute the optimal policy for single-agent RL are available; these are mainly classified as off-policy and on-policy algorithms in literature \cite{sutton2018reinforcement}. Moreover, because of the large state and action space, it is often helpful to consider the function approximations of the state value functions \cite{sutton1999policy}. To this end, actor-critic algorithms with function approximations are presented in \cite{konda2000actor}. In actor-critic algorithms, the actor updates the policy parameters, and the critic evaluates the policy's value for the actor parameters until convergence. The convergence of linear architecture in actor-critic methods is known. The algorithm in \cite{konda2000actor} uses the standard gradient while estimating the objective function. However, as mentioned in Section \ref{sec: MARL_framework}, we outlined some drawbacks of using standard gradients \cite{ratliff2013information,martens2020new}. 

To the best of our knowledge, the idea of natural gradients stems from the work of \cite{amari1998natural}. Afterward, it has been expanded to learning setup in \cite{natural_amari_in_learning}. For recent developments and work on natural gradients we refer to \cite{martens2020new}. Some recent  overviews of natural gradients are available in \cite{bottou2018optimization} and lecture slides by Roger Grosse\footnote{\url{https://csc2541-f17.github.io/slides/lec05a.pdf}}. The policy gradient theorem involving the natural gradients is explored in \cite{kakade2001natural}. For the discounted reward  \cite{agarwal2021theory,mei2020global} 
% \textcolor{blue}{off-policy setup??}
recently showed that despite the non-concavity in the objective function, the policy gradient methods under tabular setting with softmax policy characterization find the global optima. However, such a result is not available for average reward criteria with actor-critic methods and the general class of policy. Moreover, we also see in our computations in Section \ref{subsubsec: traffic_network_details} that MAN algorithms are stabilizing at different local optima. Actor-critic methods involving the natural gradients for single-agent are available in  \cite{peters2008natural,bhatnagar2009natural}. On the contrary, we deal with the multi-agent setup where all agents have private rewards but have a common objective. For a comparative survey of the MARL algorithms, we refer to \cite{busoniu2008comprehensive,tuyls2012multiagent,zhang2019multi}. 

The MARL algorithms given in \cite{zhang2019multi} are majorly centralized, and hence relatively slow. However, in many situations \cite{corke2005networked,dall2013distributed} deploying a centralized agent is inefficient and costly. Recently, \cite{zhang2018fully} gave two different actor-critic algorithms in a fully decentralized setup; one based on approximating the state-action value function and the other approximating the state value function. Another work in the same direction is available in \cite{heredia2019distributed,suttle2020multi,zhang2021decentralized,mathkar2016nonlinear}. In particular, for distributed stochastic approximations, authors in \cite{mathkar2016nonlinear} introduced and analyzed a non-linear gossip-based distributed stochastic approximation scheme. We use some proof techniques as part of consensus updates from it. We build on algorithm 2 of the \cite{zhang2018fully} and incorporate the natural gradients into it. The algorithms that we propose use the natural gradients as in  \cite{bhatnagar2009natural}. We propose three algorithms incorporating natural gradients into multi-agent RL based on Fisher's information matrix inverse, approximation of advantage parameters, or both. Using the ideas of stochastic approximation available in  \cite{borkar2000ode,borkar2009book,kushner2012stochastic}, we prove the convergence of all the proposed algorithms. 

\section{Discussion}
\label{sec: discussion}

This paper proposes three multi-agent natural actor-critic (MAN) reinforcement learning algorithms. Instead of using a central controller for taking action, our algorithms use the consensus matrix and are fully decentralized. These MAN algorithms majorly use the Fisher information matrix and the advantage function approximations. We show the convergence of all the three MAN algorithms, possibly to different local optima. 
 
We prove that a deterministic equivalent of the natural gradient-based algorithm dominates that of the MAAC algorithm under some conditions. It follows by leveraging a fundamental property of the Fisher information matrix that we show: the minimum singular value is within the reciprocal of the dimension of the policy parameterization space.

The Fisher information matrix in the natural gradient-based algorithms captures the KL divergence curvature between the policies at consecutive iterates. Indeed, we show that the KL divergence is proportional to the objective function's gradient. The use of natural gradients offered a \emph{new representation} to the objective function's gradient in the prediction space of policy distributions which improved the search for better policies.

To demonstrate the usefulness of our algorithm, we empirically evaluate them on a bi-lane traffic network model. The goal is to minimize the overall congestion in the network in a fully decentralized fashion. Sometimes the MAN algorithms can reduce the network congestion by almost $\approx 25\%$ compared to the MAAC algorithm. One of our natural gradient-based algorithms, AP-MAN, is on par with the MAAC algorithm. Moreover, we also consider an abstract MARL with $n=15$ agents; again, the MAN algorithms are at least as good as the MAAC algorithm with high confidence.

We now mention some of the further possibilities. Firstly, some assumptions on the communication network can be relaxed \cite{thoppe2021law}. A careful study of the trade-off between extra per iterate computation versus the gain in the objective function value of the learned MARL policy obtained by these MAN algorithms would be useful. It is in the context of a similar phenomenon in optimization algorithms \cite{bottou2018optimization},  \cite{nocedal2006numerical} and other computational sciences.

Moreover, further understanding of the natural gradients and its key ingredient, the Fisher information matrix $G(\theta)$, is needed in their role as learning representations. Our uniform bound on the smallest singular value of $G(\theta)$ and its role in the dominance of deterministic MAN algorithms are initial results in these directions. More broadly, various learning representations for RL like natural gradients and others are further possibilities.

\bibliography{References}
\bibliographystyle{apalike}

\newpage

\begin{appendix}

\section{Proof of theorems and lemmas}
\label{app: proofs}
In this section, we provide proof of some important theorems and lemmas stated in the main paper.

\subsection{Proof of Theorem \ref{thm: pgt} \cite{zhang2018fully}}
\label{app: pgt}
Recall Theorem \ref{thm: pgt}: Under assumption X. \ref{ass: regularity}, for any $\theta \in \Theta$, and each agent $i\in N$, the gradient of $J(\theta)$ with respect to $\theta^i$ is given by 
\begin{equation*}
\begin{aligned}
    \nabla_{\theta^i}J(\theta) &=& \mathbb{E}_{s\sim d_{\theta}, a\sim \pi_{\theta}}[\nabla_{\theta^i} \log \pi^{i}_{\theta^i}(s,a^i) \cdot A_{\theta}(s,a)]
    \\
    &=& \mathbb{E}_{s\sim d_{\theta}, a\sim \pi_{\theta}}[\nabla_{\theta^i} \log \pi^{i}_{\theta^i}(s,a^i) \cdot A^i_{\theta}(s,a)].
\end{aligned}
\end{equation*}

\begin{proof}
\label{proof: pgt}
The proof is a suitable extension of the policy gradient theorem for single agent RL \cite{sutton1999policy} which says
\begin{eqnarray}
    \nabla_{\theta} J(\theta) &=& \mathbb{E}_{s\sim d_{\theta},a \sim \pi_{\theta}}[\nabla_{\theta} \log \pi_{\theta}(s,a) \cdot Q_{\theta}(s,a)], \nonumber
    \\
    &=& \sum_{s\in \mathcal{S}}d_{\theta}(s) \sum_{a \in \mathcal{A}} \pi_{\theta}(s,a) \left[\nabla_{\theta} \sum_{i\in N} \log \pi^i_{\theta^i}(s,a^i)\right] \cdot Q_{\theta}(s,a).
\end{eqnarray}
The second equality holds because for the multi-agent setup, $\pi_{\theta} = \prod_{i\in N} \pi_{\theta^i}^i(s,a^i)$. Therefore, the gradient of the objective function $J(\theta)$ with respect to $\theta^i$ is given by
\begin{equation}
    \nabla_{\theta^i} J(\theta) = \sum_{s\in \mathcal{S}}d_{\theta}(s) \sum_{a \in \mathcal{A}} \pi_{\theta}(s,a) \left[\nabla_{\theta^i} \log \pi^i_{\theta^i}(s,a^i)\right] \cdot Q_{\theta}(s,a).
\end{equation}
Moreover, we have $\sum_{a^i\in \mathcal{A}^i} \pi^i_{\theta^i} (s,a^i) = 1$, and hence $\nabla_{\theta^i} \left[\sum_{a^i\in \mathcal{A}^i} \pi^i_{\theta^i} (s,a^i)\right] = 0$. This implies, $\nabla_{\theta^i} J(\theta) = \mathbb{E}_{s\sim d_{\theta},a \sim \pi_{\theta}}[\nabla_{\theta^i} \log \pi^i_{\theta^i}(s,a^i) \cdot Q_{\theta}(s,a)] $.

Now let $a^{-i}= (a^1, \dots, a^{i-1}, a^{i+1}, \dots, a^n)$ be the joint action of all the agents except agent $i\in N$, and let $\mathcal{A}^{-1} = \prod_{j\neq i} \mathcal{A}^j$. Thus, for any function $\Lambda:\mathcal{S} \times \mathcal{A}^{-i} \rightarrow \mathbb{R}$ that doesn't depend on $a^i$, we have
\begin{equation}
\begin{aligned}
    \sum_{a \in \mathcal{A}} \pi_{\theta}(s,a) \left[\nabla_{\theta^i} \log \pi^i_{\theta^i}(s,a^i)\right] \cdot \Lambda(s,a^{-i}) & = \sum_{a^{-i} \in \mathcal{A}^{-i}} \Lambda(s,a^{-i}) \cdot \left[ \prod_{j\in N, j\neq i} \pi^j_{\theta^j}(s,a^j) \right] \cdot \left[ \sum_{a^i\in \mathcal{A}^i} \nabla_{\theta^i} \pi^i_{\theta^i}(s,a^i)\right]
    \\
    & = 0.
\end{aligned}
\end{equation}
Thus, replacing $\Lambda$ in above equation by value function $V_{\theta}(s)$ we have first equality of theorem, and replacing it by $\tilde{V}_{\theta}^i(s,a^{-i})$ we obtain the second equality.

\end{proof}

\subsection{Stationary points of optimization problems in the reward function parameterization}
\label{app: derivative_reward_optimizations}

Recall the two optimization problems are:
\begin{equation}
\label{eqn: OP_1}
\tag{OP 1}
    min_{\lambda} ~ \sum_{s\in \mathcal{S}, a\in \mathcal{A}} \tilde{d}_{\theta}(s,a) [ \bar{R}(s,a) - \bar{R}(s,a; \lambda)]^2.
\end{equation}

\begin{equation}
\label{eqn: OP_2}
\tag{OP 2}
min_{\lambda} ~ \sum_{i\in {N}} \sum_{s\in \mathcal{S}, a\in \mathcal{A}} \tilde{d}_{\theta}(s,a) [ {R}^i(s,a) - \bar{R}(s,a; \lambda) ]^2,
\end{equation}

Taking the first derivative of the objective function in optimization problem (\ref{eqn: OP_1}) w.r.t. $\lambda$, we have: 
\begin{eqnarray*}
    & & -2 \times \sum_{s\in \mathcal{S}, a\in \mathcal{A}} \tilde{d}_{\theta}(s,a) [ \bar{R}(s,a) - \bar{R}(s,a; \lambda)] \times \nabla_{\lambda} \bar{R}(s,a; \lambda),
    \\
    &=& -2 \times \sum_{s\in \mathcal{S}, a\in \mathcal{A}} \tilde{d}_{\theta}(s,a) \left[ \frac{1}{n}\sum_{i\in {N}} R^i(s,a) - \bar{R}(s,a; \lambda)\right] \times \nabla_{\lambda} \bar{R}(s,a; \lambda),
    \\
    &=& -\frac{2}{n} \times \sum_{s\in \mathcal{S}, a\in \mathcal{A}} \tilde{d}_{\theta}(s,a) \left[ \sum_{i\in {N}} R^i(s,a) - n\cdot\bar{R}(s,a; \lambda)\right] \times \nabla_{\lambda} \bar{R}(s,a; \lambda),
    \\
    &=& -\frac{2}{n} \times \sum_{s\in \mathcal{S}, a\in \mathcal{A}} \tilde{d}_{\theta}(s,a) \left[ \sum_{i\in {N}} \left( R^i(s,a) - \bar{R}(s,a; \lambda)\right)\right] \times \nabla_{\lambda} \bar{R}(s,a; \lambda).
\end{eqnarray*}
Ignoring the factor $\frac{1}{n}$, the above equation we exactly have the first derivative of the objective function in \ref{eqn: OP_2}. Thus, both the optimization problems have the same stationary points. Hence, \ref{eqn: OP_1} is an \textit{equivalent characterization} of \ref{eqn: OP_2}.

\subsection{Proof of Lemma \ref{lemma: sigma_min_less_1_m}}
\label{app: sigma_min_less_1_m}
Recall Lemma \ref{lemma: sigma_min_less_1_m}: For the Fisher information matrix, $G(\theta) = \mathbb{E}[\psi \psi^{\top}]$, such that $|| \psi ||\leq 1$, the minimum singular value, $\sigma_{\min}(G(\theta))$ is upper bounded by $\frac{1}{m}$, i.e.,
\begin{equation}
\sigma_{\min}(G(\theta)) \leq \frac{1}{m}.
\end{equation}
% Recall the Lemma \ref{lemma: sigma_min_less_1_m}: Let the features associated to the parameterized policies are uniformly bounded by 1. Moreover, $\psi_t \in \mathbb{R}^{m}$. Then the smallest singular value of the Fisher information matrix $G(\theta)$ is upper bounded by $\frac{1}{m}$, i.e,
% \begin{equation}
% \label{eqn: sgima_min_less_than_frac}
%     \sigma_{\min}(G(\theta)) \leq \frac{1}{m}.
% \end{equation}

\begin{proof}
\label{proof: sigma_min_less_1_m}
% The proof is easy to follow:
First recall that the Fisher matrix is defined as $G(\theta) = \mathbb{E}[\psi \psi^{\top}]$. Also, the matrix $G(\theta)$ is symmetric and positive definite. Let $\lambda_{\min}(G(\theta))$ be the minimum eigenvalue of $G(\theta)$. Thus, $\sigma_{\min}(G(\theta)) = \sqrt{\lambda_{\min}(G(\theta) G(\theta)^{\top})} = \sqrt{\lambda_{\min}(G^2(\theta))} = \sqrt{\lambda_{\min}^2(G(\theta))} =  \lambda_{\min}(G(\theta))$. 
\begin{eqnarray*}
    \lambda_{\min}(G(\theta)) &\leq& \frac{1}{m} \sum_{j=1}^{m} \lambda_j (G(\theta)) = \frac{1}{m} tr(G(\theta))
    \\
    &=& \frac{1}{m} tr(\mathbb{E}[\psi \psi^{\top}]) =  \frac{1}{m} \mathbb{E}[tr(\psi \psi^{\top})]
    \\
    &=& \frac{1}{m} \mathbb{E}[||\psi||^{2}]
    \\
    &\leq& \frac{1}{m}.
\end{eqnarray*}
This ends the proof.

\end{proof}

\subsection{Proof of Lemma \ref{lemma: KL_boltzmann}}
\label{app: KL_boltzmann}

Recall the Lemma \ref{lemma: KL_boltzmann}: For the Boltzmann policy as given in Equation (\ref{eqn: boltzmann}) the KL divergence between policy parameterized by $\theta_t$ and $\theta_t + \Delta \theta_t$ is given by
\begin{equation}
    KL(\pi_{\theta_t}(s,a)|| \pi_{\theta_t + \Delta \theta_t} (s,a)) = \mathbb{E} \left[\log \left( \sum_{b\in \mathcal{A}} \pi_{\theta_t}(s,b)  \exp(\Delta q^{\top}_{s,ba} \Delta \theta_t)\right) \right],
\end{equation}
where $\Delta q^{\top}_{s,ba} = q^{\top}_{s,b} - q^{\top}_{s,a}$.
\begin{proof}
Recall, KL divergence is defined as
\begin{eqnarray*}
	KL(\pi_{\theta_t}(\cdot, \cdot)|| \pi_{\theta_t + \Delta \theta_t} (\cdot, \cdot)) &=& \mathbb{E} \left[\log \left(\frac{\pi_{\theta_t}(s,a)}{\pi_{\theta_t + \Delta \theta_t}(s,a)}\right) \right] = \mathbb{E} \left[ \log \left( \frac{\frac{\exp(q^{\top}_{s,a} \theta_t)}{\sum_{c\in\mathcal{A}} \exp(q^{\top}_{s,c} \theta_t)}}{\frac{\exp(q^{\top}_{s,a} (\theta_t + \Delta \theta_t))}{\sum_{b\in\mathcal{A}} \exp(q^{\top}_{s,b} (\theta_t + \Delta \theta_t))}} \right) \right]
	\\
	&=& \mathbb{E} \left[ \log \left( \exp(q^{\top}_{s,a} (\theta_t - (\theta_t + \Delta \theta_t)) \cdot \frac{\sum_{b\in\mathcal{A}} \exp(q^{\top}_{s,b} (\theta_t + \Delta \theta_t))}{\sum_{c\in\mathcal{A}} \exp(q^{\top}_{s,c} \theta_t)}  \right) \right]
	\\
	&=& \mathbb{E} \left[ -q^{\top}_{s,a} \Delta \theta_t + \log \left( \sum_{b\in \mathcal{A}} \left( \frac{\exp(q^{\top}_{s,b} \theta_t) \cdot \exp(q^{\top}_{s,b} \Delta \theta_t)}{\sum_{c\in\mathcal{A}} \exp(q^{\top}_{s,c} \theta_t)} \right) \right) \right]
	\\
	&=&\mathbb{E} \left[ -q^{\top}_{s,a} \Delta \theta_t + \log \left( \sum_{b\in \mathcal{A}} \pi_{\theta_t}(s,b) \cdot \exp(q^{\top}_{s,b} \Delta \theta_t) \right) \right]
	\\
	&=& \mathbb{E} \left[\log \left( \sum_{b\in \mathcal{A}} \pi_{\theta_t}(s,b) \cdot \frac{\exp(q^{\top}_{s,b} \Delta \theta_t)}{ \exp(q^{\top}_{s,a} \Delta \theta_t)} \right) \right]
	\\
	&=& \mathbb{E} \left[\log \left( \sum_{b\in \mathcal{A}} \pi_{\theta_t}(s,b) \cdot \exp((q^{\top}_{s,b} - q^{\top}_{s,a}) \Delta \theta_t)\right) \right]
	\\
	&=& \mathbb{E} \left[\log \left( \sum_{b\in \mathcal{A}} \pi_{\theta_t}(s,b) \cdot \exp(\Delta q_{s,ba}^{\top} \Delta \theta_t)\right) \right].
\end{eqnarray*}
This ends the proof.

\end{proof}

\subsection{Proof of Lemma \ref{lemma: grad_KL_softmax}}
\label{app: grad_KL_softmax}

Recall the Lemma \ref{lemma: grad_KL_softmax}: For the Boltzmann policy as given in Equation (\ref{eqn: boltzmann}) we have 
\begin{equation}
    \nabla KL(\pi_{\theta_t}(\cdot,\cdot)|| \pi_{\theta_t + \Delta \theta_t}(\cdot,\cdot)) = -\mathbb{E}[\nabla \log \pi_{\theta_t + \Delta \theta_t} (s,a)].
\end{equation}
\begin{proof}
Recall, $KL(\pi_{\theta_t}(\cdot,\cdot)|| \pi_{\theta_t + \Delta \theta_t}(\cdot,\cdot))$ for the Boltzmann policy is given in Lemma \ref{lemma: KL_boltzmann}. Thus, $\nabla KL(\pi_{\theta_t}(\cdot,\cdot)|| \pi_{\theta_t + \Delta \theta_t}(\cdot,\cdot)) $ is
\begin{eqnarray*}
	&=& \nabla \mathbb{E} \left[\log \left( \sum_{b\in \mathcal{A}} \pi_{\theta_t}(s,b) \cdot \exp(\Delta q_{s,ba}^{\top} \Delta \theta_t)\right) \right]
	\\
	&=& \mathbb{E}\left[ \nabla \log \left( \sum_{b\in \mathcal{A}} \pi_{\theta_t}(s,b) \cdot \exp(\Delta q_{s,ba}^{\top} \Delta \theta_t)\right) \right]~~(\because ~\mathcal{A} ~is~finite)
	\\
	&=& \mathbb{E}\left[ \frac{1}{\sum_{b\in \mathcal{A}} \pi_{\theta_t}(s,b) \cdot \exp(\Delta q_{s,ba}^{\top} \Delta \theta_t)} \cdot \sum_{b\in \mathcal{A}} \pi_{\theta_t}(s,b) \cdot \Delta q_{s,ba} \cdot \exp(\Delta q_{s,ba}^{\top} \Delta \theta_t) \right]
	\\
	&=& \mathbb{E}\left[ \sum_{b\in \mathcal{A}} \pi_{\theta_t}(s,b) \cdot \Delta q_{s,ba} \cdot \left( \frac{\exp(\Delta q_{s,ba}^{\top} \Delta \theta_t)}{\sum_{c\in \mathcal{A}} \pi_{\theta_t}(s,c) \cdot \exp(\Delta q_{s,ca}^{\top} \Delta \theta_t)} \right) \right]
	\\
	&=& \mathbb{E}\left[  \sum_{b\in \mathcal{A}} \Delta q_{s,ba} \cdot \left( \frac{\pi_{\theta_t}(s,b) \cdot \exp(\Delta q_{s,ba}^{\top} \Delta \theta_t)}{\sum_{c\in \mathcal{A}} \pi_{\theta_t}(s,c) \cdot \exp(\Delta q_{s,ca}^{\top} \Delta \theta_t)} \right) \right]
	\\
	&=& \mathbb{E}\left[ \sum_{b\in \mathcal{A}} \Delta q_{s,ba} \cdot \left( \frac{\frac{\exp(q^{\top}_{s,b} \theta_t)}{\sum_{e\in\mathcal{A}} \exp(q^{\top}_{s,e} \theta_t)} \cdot \frac{\exp(q^{\top}_{s,b} \Delta \theta_t)}{\exp( q^{\top}_{s,a} \Delta \theta_t)}}{\sum_{c\in \mathcal{A}} \frac{\exp(q^{\top}_{s,c} \theta_t)}{\sum_{d\in\mathcal{A}} \exp(q^{\top}_{s,d} \theta_t)} \cdot \frac{\exp(q^{\top}_{s,c} \Delta \theta_t)}{\exp( q^{\top}_{s,a} \Delta \theta_t)}} \right) \right]
	\\
	&=& \mathbb{E}\left[ \sum_{b\in \mathcal{A}} \Delta q_{s,ba} \cdot \left( \frac{\exp(q^{\top}_{s,b} (\theta_t + \Delta \theta_t))}{\sum_{c\in \mathcal{A}} \exp(q^{\top}_{s,c} (\theta_t + \Delta \theta_t))} \right) \right]
	\\
	&=& \mathbb{E}\left[ \sum_{b\in \mathcal{A}} (q^{\top}_{s,b} - q^{\top}_{s,a}) \cdot \pi_{\theta_t + \Delta \theta_t} (s,b) \right]
	\\
	&=& -\mathbb{E}\left[ q^{\top}_{s,a} - \sum_{b\in \mathcal{A}} q^{\top}_{s,b} \cdot \pi_{\theta_t + \Delta \theta_t} (s,b)  \right]
	\\
	&=& -\mathbb{E}[\nabla \log \pi_{\theta_t + \Delta \theta_t} (s,a)].
\end{eqnarray*}
This ends the proof.

\end{proof}

\subsection{Proof of Theorem \ref{thm: critic_convergence_fisher}}
\label{app: critic_convergence_proof}

Recall Theorem \ref{thm: critic_convergence_fisher}: Under assumptions X. \ref{ass: regularity}, X. \ref{ass: comm_matrix}, and X. \ref{ass: feature_conds}, for any policy $\pi_{\theta}$, with sequences $\{\lambda^i_t\}, \{\mu^i_t\}, \{v^i_t\}$, we have $lim_t~ \mu^i_t = J(\theta),~ lim_t~ \lambda^i_t = \lambda_{\theta}$, and $lim_t~v^i_t = v_{\theta}$ a.s. for each agent $i\in N$, where $J(\theta),~ \lambda_{\theta}$, and $v_{\theta}$ are unique solutions to 
\begin{eqnarray*}
    F^{\top}D_{\theta}^{s,a}(\bar{R} - F\lambda_{\theta}) &=& 0,
    \\
    \varPhi^{\top} D_{\theta}^s[T_{\theta}^V(\varPhi v_{\theta}) - \varPhi v_{\theta}] &=& 0.
\end{eqnarray*}

\begin{proof}
\label{proof: critic_convergence_fisher}
% The proof is similar to that of Theorem 4.9 of \cite{zhang2018fully}. 
For all $i\in N$, define the following $z^i_t = [\mu^i_t, (\lambda^i_t)^{\top}, (v^i_t)^{\top}]^{\top} \in \mathbb{R}^{1+M+L}$. Firstly, we recall the following lemma from \cite{zhang2018fully} to give the bounds on $z^i_t$ (for proof see \cite{zhang2018fully}).

\begin{lemma}
\label{lemma: z_boundedness}
Under assumptions X. \ref{ass: regularity}, X. \ref{ass: comm_matrix}, and X. \ref{ass: feature_conds} the sequence $\{z^i_t\}$ satisfy $sup_t~||z^i_t||<\infty$ a.s., for all $i\in {N}$.
\end{lemma}
Now consider the actor step as given in Equation (\ref{eqn: actor_update_manac1})
\begin{eqnarray}
    \theta^i_{t+1} &=& \theta^i_t + \beta_{\theta,t} \cdot G_t^{i^{-1}} \cdot \tilde{\delta}^i_t \cdot \psi^i_t \nonumber
    \\
    &=& \theta^i_t + \beta_{v,t} \cdot \frac{\beta_{\theta,t}}{\beta_{v,t}} \cdot G_t^{i^{-1}} \cdot \tilde{\delta}^i_t \cdot \psi^i_t.
\label{eqn: actor_step_fisher}
\end{eqnarray}
By assumption, $\frac{\beta_{\theta,t}}{\beta_{v,t}} \rightarrow 0$. Also note that $\tilde{\delta}^i_t \cdot \psi^i_t$ is bounded by assumptions and Lemma \ref{lemma: z_boundedness}. Moreover, $G_t^{i^{-1}}$ is bounded by assumption X. \ref{ass: g_inv}. 
Thus, the actor update in Equation (\ref{eqn: actor_step_fisher}) can be traced via an ODE $\dot \theta^i = 0$. Hence we can fix the value of $\theta_t$ as constant $\theta$ while analyzing the critic-step at the faster time scale. Note that the above argument of taking $\theta_t \equiv \theta$ while analyzing the critic step is common to all the 3 MAN algorithms. The only difference is in actor update Equation (\ref{eqn: actor_step_fisher}), nonetheless a similar argument concludes that  $\theta_t \equiv \theta$. 

Let $\mathcal{F}_{t,4} = \sigma(r_{\tau}, \mu_{\tau}, \lambda_{\tau},v_{\tau}, s_{\tau},a_{\tau},C_{\tau - 1}, \tau \leq t)$ be the filtration which is an increasing $\sigma$-algebra over time $t$. Define the following for notation convenience. Let $r_t = [r^1_t, \dots, r^n_t]^{\top}, ~v_t = [(v^1_t)^{\top}, \dots, (v^n_t)^{\top}]^{\top},~ \delta_t = [(\delta^1_t)^{\top}, \dots, (\delta^n_t)^{\top}]^{\top}$, and $z_t = [(z^1_t)^{\top}, \dots, (z^n_t)^{\top}]^{\top} \in \mathbb{R}^{n(1+M+L)}$. Moreover, let $A \otimes B$ represents the Kronecker product of matrices $A$ and $B$. Let $y_t = [(y^1_t)^{\top}, \dots, (y^n_t)^{\top}]^{\top}$, where $y^i_{t+1} = [r^i_{t+1} - \mu^i_t, (r^i_{t+1} - f_t^{\top}\lambda^i_t)f_t^{\top}, \delta^i_t \varphi_t^{\top}]^{\top}$. Recall, $f_t = f(s_t,a_t)$, and $\varphi_t = \varphi(s_t)$. Let $I$ be the identity matrix of the dimension $(1+M+L) \times (1+M+L)$. Then update of $z_t$ can be written as 
\begin{equation}
\label{eqn: z_recursion}
z_{t+1} = (C_t \otimes I)(z_t + \beta_{v,t}\cdot y_{t+1}).
\end{equation}
Let $\mathbbm{1} = (1,\dots, 1)$ represents the vector of all 1's. We define the operator $\left\langle z \right\rangle = \frac{1}{n} (\mathbbm{1}^{\top} \otimes I) z = \frac{1}{n} \sum_{i\in {N}} z^i $. This $\langle z \rangle \in \mathbb{R}^{1+M+L}$ represents the average of the vectors in $\{z^1, z^2, \dots, z^n \}$. For notational simplicity let $k=1+M+L$. 
Moreover, let $\mathcal{J} = (\frac{1}{n} \mathbbm{1}\mathbbm{1}^{\top})\otimes I \in \mathbb{R}^{nk \times nk}$ is the projection operator that projects a vector into the consensus subspace $\{\mathbbm{1}\otimes u: u\in \mathbb{R}^{k}\}$. Thus $\mathcal{J}z = \mathbbm{1}\otimes \langle z \rangle$. Now define the 
disagreement vector $z_{\perp} = \mathcal{J}_{\perp}z = z- \mathbbm{1}\otimes \langle z \rangle$, where $\mathcal{J}_{\perp} = I-\mathcal{J}$. Here $I$ is $nk \times nk$ dimensional identity matrix. The iteration $z_t$ can be separated as the sum of a vector in disagreement space and a vector in consensus space, i.e., $z_t = z_{\perp,t} + \mathbbm{1} \otimes \langle z_t \rangle$. The proof of critic step convergence consists of two steps.

% Step 1 shows a.s. convergence of $\{z_{\perp,t}\}$ to zero, and step 2 shows $\{\mathbbm{1} \otimes \langle z_t \rangle\}$ convergence to equilibrium such that $\langle z_t \rangle$ satisfy the equations  
% \begin{eqnarray}
%     F^{\top}D_{\theta}^{s,a}(\bar{R} - F\lambda_{\theta}) &=& 0,
%     \\
%     \Phi^{\top} D_{\theta}^s[T_{\theta}^V(\Phi v_{\theta}) - \Phi v_{\theta}] &=& 0
% \end{eqnarray}

\textbf{Step 01:} To show $lim_t~ z_{\perp,t} = 0$ a.s. From Lemma \ref{lemma: z_boundedness} we have $\mathbb{P}[sup_t ||z_t|| < \infty] = 1$, i.e., $\mathbb{P}[\cup_{K\in \mathbb{Z}^{+}}~\{sup_t ||z_t|| < K\}] = 1$. It suffices to show that $lim_t~ z_{\perp, t} \mathbbm{1}_{\{sup_t ||z_t|| < K\}} = 0$ for any $K\in \mathbb{Z}^+$. Lemma B.5 in \cite{zhang2018fully} proves the boundedness of  $\mathbb{E} \left[||\beta_{v,t}^{-1}z_{\perp,t}||^2\right]$ over the set  $\{sup_t ||z_t|| \leq K\}$, for any $K>0$. We state the lemma here. 
\begin{lemma}
\label{lemma: sup_algo1}
Under assumptions X. \ref{ass: comm_matrix}, and X. \ref{ass: feature_conds} for any $K>0$, we have 
\begin{equation*}
    sup_t ~\mathbb{E}[||\beta_{v,t}^{-1}z_{\perp,t}||^2 \mathbbm{1}_{\{sup_t ||z_t|| \leq K\}}] < \infty.
\end{equation*}
\end{lemma}
From Lemma \ref{lemma: sup_algo1} we obtain that for any $K>0, ~\exists ~ K_1< \infty$ such that for any $t\geq 0,~ \mathbb{E}[||z_{\perp,t}||^2] < K_1 \beta_{v,t}^2 $ over the set $\{sup_t~ ||z_t|| < K\}$. Since $\sum_t \beta_{v,t}^2 < \infty$, by Fubini's theorem we have $\sum_t \mathbb{E}(||z_{\perp, t}||^2 \mathbbm{1}_{\{sup_{t} \ ||z_t|| < K\}}) < \infty$. Thus, $\sum_t ||z_{\perp, t}||^2 \mathbbm{1}_{\{sup_{t} \   ||z_t|| < K\}} < \infty$ a.s. Therefore, $lim_t ~ z_{\perp, t} \mathbbm{1}_{\{sup_t ||z_t|| < K\}} =0 $ a.s. Since $\{sup_t ||z_t||  < \infty\}$ with probability 1, thus $lim_t~ z_{\perp,t} = 0$ a.s. This ends the proof of Step 01.

\textbf{Step 02:} To show the convergence of the consensus vector $\mathbbm{1}\otimes \langle z_t \rangle$ first, note that the iteration of  $\langle z_t \rangle$ (Equation (\ref{eqn: z_recursion})) can be written as
\begin{eqnarray}
    \langle z_{t+1} \rangle  &=& \frac{1}{N} (\mathbbm{1}^{\top} \otimes I)(C_t \otimes I)(\mathbbm{1}\otimes \langle z_t \rangle + z_{\perp, t} + \beta_{v,t} y_{t+1}) \nonumber
    \\
    &=& \langle z_t \rangle + \beta_{v,t} \langle (C_t \otimes I)(y_{t+1} + \beta_{v,t}^{-1}z_{\perp,t}) \rangle \nonumber
    \\
    &=&  \langle z_{t} \rangle + \beta_{v,t} \mathbb{E}( \langle y_{t+1} \rangle | \mathcal{F}_{t,4}) + \beta_{v,t}\xi_{t+1}, \label{eqn: z_update}
\end{eqnarray}
where 
\begin{eqnarray*}
\xi_{t+1} &=& \langle (C_t \otimes I)(y_{t+1} + \beta_{v,t}^{-1}z_{\perp,t}) \rangle - \mathbb{E}( \langle y_{t+1} \rangle| \mathcal{F}_{t,4}), ~~ and
\\
\label{y}
\langle y_{t+1} \rangle &=& [\bar{r}_{t+1} - \langle \mu_t \rangle, (\bar{r}_{t+1} - f_t^{\top} \langle \lambda_t \rangle)f_t^{\top}, \langle \delta_t \rangle \varphi_t^{\top}]^{\top}.
\end{eqnarray*}
Here $\langle \delta_{t} \rangle = \bar{r}_{t+1} - \langle \mu_t \rangle + \varphi^{\top}_{t+1} \langle v_t \rangle - \varphi^{\top}_t \langle v_t \rangle$. Note that $\mathbb{E}(\langle y_{t+1} \rangle | \mathcal{F}_{t,4})$ is Lipschitz continuous in $\langle z_t \rangle = (\langle \mu_t \rangle, \langle \lambda_t \rangle^{\top}, \langle v_t \rangle^{\top})^{\top}$. Moreover, $\xi_{t+1}$ is a martingale difference sequence and satisfies 
\begin{equation}
\label{eqn: xi-martingale}
\mathbb{E}[||\xi_{t+1}||^2 ~|~ \mathcal{F}_{t,4}] \leq  \mathbb{E}[||y_{t+1} + \beta_{v,t}^{-1}z_{\perp,t}||^2_{S_t}~|~\mathcal{F}_{t,4}] +  || \mathbb{E}(\langle y_{t+1} \rangle ~|~ \mathcal{F}_{t,4}) ||^2,
\end{equation}
where $S_t = \frac{C_t^{\top} \mathbbm{1} \mathbbm{1}^{\top}C_t \otimes I}{ n^{2}}$ has bounded spectral norm. 
% The first term is bounded over a set $\{sup_t ||z_t|| \leq M\}$, for any $M>0$, i.e., there exists $K_3,K_4 < \infty$ such that 
% \begin{eqnarray*}
%     \mathbb{E}[||y_{t+1} + \beta_{v,t}^{-1}z_{\perp,t}||^2_{S_t}|\mathcal{F}_{t,1}] \mathbbm{1}_{\{sup_t ||z_t|| \leq M\}}
%     &\leq& K_3 \mathbb{E}[||y_{t+1}||^2 +||\beta_{v,t}^{-1}z_{\perp,t}||^2_{S_t}||^2 ~|\mathcal{F}_{t,1}] \mathbbm{1}_{\{sup_t ||z_t|| \leq M\}}
% \end{eqnarray*}
% Using Eq(\ref{y}) the second term can be bounded by 
% \begin{eqnarray}
% ||\mathbb{E}(\langle y_{t+1} \rangle | \mathcal{F}_{t,1}) ||^2 \leq  \mathbb{E}( || \langle y_{t+1} \rangle ||^2 ~| \mathcal{F}_{t,1}) &\leq&  K_5 (1+ ||\langle \mu_t \rangle||^2 + ||\langle \lambda_t \rangle||^2 + ||\langle v_t \rangle||^2) 
% \\
% &=& K_5(1+||\langle z_t \rangle||^2)
% \end{eqnarray}
Bounding first and second terms in RHS of Equation (\ref{eqn: xi-martingale}) we have, for any $K>0$
\begin{equation*}
  \mathbb{E}(||\xi_{t+1}||^2 | \mathcal{F}_{t,4}) \leq K_2 (1+ ||\langle z_t \rangle||^2),
\end{equation*}
over the set  $\{sup_t ~ ||z_t||\leq K\}$ for some $K_2 < \infty$. The ODE associated with the Equation (\ref{eqn: z_update}) has the form
\begin{equation}
\label{eqn: z_ode}
\langle \dot z \rangle 
=
\begin{pmatrix}  \langle \dot \mu \rangle \\ \langle \dot \lambda \rangle \\ \langle \dot v \rangle
\end{pmatrix}
= 
\begin{pmatrix}
-1 & 0 & 0 \\
0 & -F^{\top}D_{\theta}^{s,a}F & 0 \\
-\Phi^{\top}D_{\theta}^s\mathbbm{1} & 0 & \Phi^{\top}D_{\theta}^s(P^{\theta}-I)\Phi
\end{pmatrix}
\begin{pmatrix}  \langle \mu \rangle \\ \langle \lambda \rangle \\ \langle v \rangle
\end{pmatrix}
+
\begin{pmatrix}
J(\theta)
\\
F^{\top}D_{\theta}^{s,a}\bar{R}
\\
\Phi^{\top}D_{\theta}^s\bar{R}_{\theta}
\end{pmatrix}.
\end{equation}

Let the RHS of equation (\ref{eqn: z_ode}) be $h(\langle z \rangle)$. Note that $h(\langle z \rangle)$ is Lipschitz  continuous in $\langle z \rangle$. Also, recall that $D_{\theta}^{s,a} = diag[d_{\theta}(s) \cdot \pi_{\theta}(s,a), s\in\mathcal{S}, a\in \mathcal{A}]$. Using assumption X. \ref{ass: regularity}, and Perron-Frobenius theorem the stochastic matrix $P^{\theta}$ has a eigenvalue of 1, and remaining eigenvalues have real part less than 1. Hence $(P^{\theta} -I)$ has one eigenvalue zero, and all other eigenvalues with negative real parts. Same follows for the matrix $\Phi^{\top}D_{\theta}^s(P^{\theta}-I)\Phi$, since $\Phi$ is full rank column matrix. The simple eigen-vector $\nu$ of the matrix satisfies $\Phi \nu = \alpha\mathbbm{1}$ for some $\alpha\neq 0$. However, this will not happen because of the assumption X. \ref{ass: feature_conds}. Hence the ODE given in Equation (\ref{eqn: z_ode}) has unique globally asymptotically stable equilibrium $[J(\theta), \lambda_{\theta}^{\top}, v_{\theta}^{\top}]^{\top}$ satisfying 
\begin{eqnarray*}
    - \langle \mu \rangle + J(\theta) = 0;~~F^{\top}D_{\theta}^{s,a}(\bar{R} - F\lambda_{\theta}) = 0;~~\Phi^{\top} D_{\theta}^s[T_{\theta}^V(\Phi v_{\theta}) - \Phi v_{\theta}] = 0.
\end{eqnarray*}
% where $[J(\theta), \lambda_{\theta}^{\top}, v_{\theta}^{\top}]^{\top}$ are the unique equilibrium solutions to the above equations. 
Moreover, from Lemma \ref{lemma: z_boundedness}, and Lemma \ref{lemma: sup_algo1}, the sequence $\{z_t\}$ is bounded almost surely, so is the sequence $\{\langle z_t \rangle\}$. Specializing Corollary 8 and Theorem 9 on page 74-75 in \cite{borkar2009book} we have $lim_t~ \langle \mu_t \rangle= J(\theta),~ lim_t~ \langle \lambda_t \rangle = \lambda_{\theta}$, and $lim_t~ \langle v_t \rangle = v_{\theta}$ a.s. over the set $\{sup_t~||z_t|| \leq K\}$ for any $K>0$. This concludes the proof of Step 02. 

The proof of the theorem thus follows from Lemma \ref{lemma: z_boundedness} and results from Step 01. Thus, we have $lim_t~ \mu^i_t = J(\theta),~ lim_t~ \lambda^i_t = \lambda_{\theta}$, and $lim_t~ v^i_t = v_{\theta}$ a.s. for each $i\in {N}$.
\end{proof}

% \subsection{Proof of Theorem \ref{thm: J_comp_t+1}}
% \label{app: J_comp_t+1}
% Recall the Lemma: Let $\theta^M_t$ and $\theta^N_t$ are the actor parameters at time $t$ for FI-MAN and MAAC algorithms respectively such that $\theta^M_t = \theta^N_t$, so $J(\theta^M_t) = J(\theta^N_t)$, and $\nabla J(\theta^M_t) = \nabla J(\theta^N_t)$. Then
% \begin{eqnarray*}
%     J(\tilde{\theta}^N_{t+1}) &\geq& J(\tilde{\theta}^M_{t+1})~~ if~~ \sigma_{\min}(G(\theta^N_t)) \leq 1,~ and
%     \\
%     J(\tilde{\theta}^N_{t+1}) &\leq& J(\tilde{\theta}^M_{t+1})~~ if~~ \sigma_{\max}(G(\theta^N_t)) \geq 1,
% \end{eqnarray*}
% where $\sigma_{\max}(G(\theta^N_t))$, and $\sigma_{\min}(G(\theta^N_t))$ are maximum and minimum singular values of the Fisher information matrix $G(\theta^N_t)$.

% Moreover, if the sequences $\{\sigma_{\min}(G(\theta^N_t))\}_{t\geq0}$ and $\{\sigma_{\max}(G(\theta^N_t))\}_{t\geq0}$ are monotonically decreasing, then
% \begin{eqnarray*}
%     J(\theta^{N^{\star}}) &\geq& J(\theta^{M^{\star}})~~ if~~ \sigma_{\min}(G(\theta^{N^{\star}})) \leq 1,
%     \\
%     J(\theta^{N^{\star}}) &\leq& J(\theta^{M^{\star}})~~ if~~ \sigma_{\max}(G(\theta^{N^{\star}})) \geq 1,
% \end{eqnarray*}
% where $J(\theta^{N^{\star}})$ and $J(\theta^{M^{\star}})$ are the asymptotic optimal values of the objective function for FI-MAN and MAAC algorithms respectively, and $\sigma_{\min}(G(\theta^{N^{\star}}))$ and $\sigma_{\max}(G(\theta^{N^{\star}}))$ are the smallest and the largest singular values of the limiting Fisher information matrix respectively.

\section{Further details on computational experiments}
\label{app: experiments_detailed}
This section will provide further details of the computational experiments for traffic network control and the abstract multi-agent reinforcement learning setup.

\subsection{Experiments for abstract multi-agent reinforcement learning model}
\label{supp: MARL_Garnet}

We will provide the details of the experiments for an abstract multi-agent RL model we have considered in Section \ref{subsec: garnet_experiments}. 

The abstract MARL model that we consider consists of $n=15$ agents and $|\mathcal{S}| = 15$ states. Each agent $i\in N$ is endowed with the binary valued actions, $\mathcal{A}^i = \{0,1\}$. Therefore, total number of actions are $2^{15}$. Each element of the transition probability is a random number uniformly generated from the interval $[0,1]$. These values are normalized to be stochastic. To ensure the ergodicity we add a small constant $10^{-5}$ to each entry of the transition matrix. The mean reward $R^i(s,a)$ is sampled uniformly from the interval $[0,4]$ for each agent $i \in N$, and each state-action pair $(s,a)$. The instantaneous rewards $r^i_t$ are sampled uniformly from the interval $[R^i(s,a)-0.5, R^i(s,a) + 0.5]$. We parameterize the policy using the Boltzmann distribution as 
\begin{equation}
\label{eqn: gibbs_policy_2}
\pi^i_{\theta^i}(s,a^i) = \frac{\exp(q_{s,a^i}^{\top} \cdot \theta^i)}{\sum_{b^i\in \mathcal{A}^i} \exp(q_{s,b^i}^{\top}\cdot \theta^i)},
\end{equation}
where $q_{s,b^i}\in \mathbb{R}^{m_i}$ is the feature vector of dimension same as $\theta^i$, for any $s \in \mathcal{S}$, and $b^i \in \mathcal{A}^i$, for all $i\in N$. 
We set $m_i = 5,~ \forall i\in N$. All the elements of $q_{s,b^i}$ are also uniformly sampled from $[0,1]$. For the Boltzmann policy function we have \cite{sutton1999policy}
\begin{equation*}
    \nabla_{\theta^i} \log \pi^i_{\theta^i} (s, a^i)  = q_{s,a^i} - \sum_{b^i\in \mathcal{A}^i} \pi^i_{\theta^i} (s, b^i) q_{s,b^i}.
\end{equation*}
% All the feature vectors (value function, and reward function) are sampled uniformly from the set $[0,1]$ of corresponding dimensions. 
The features for state value function $\varphi(s) \in \mathcal{R}^L$ are uniformly sampled from $[0,1]$ of dimension $L = 5 << |\mathcal{S}|$, whereas the features for the reward function $f(s,a) \in \mathbb{R}^M$ are uniformly sampled from $[0,1]$ of dimension $M = 10 << |\mathcal{S}||\mathcal{A}|$. The communication network $\mathcal{G}_t$ is generated randomly at each time $t$, such that the links in the network are formed with the connectivity ratio\footnote{Ratio of total degree of the graph and the degree of complete graph, i.e., $\frac{2E}{n(n-1)}$, where $E$ is the number of edges in the graph.} $4/n$. The step sizes are taken as $\beta_{v,t} = \frac{1}{t^{0.65}}$, and $\beta_{\theta,t} = \frac{1}{t^{0.85}}$ respectively. Initial values of parameters $\mu^i_0, \tilde{\mu}^i_0, v^i_0, \tilde{v}^i_0, \lambda^i_0, \tilde{\lambda}^i_0, \theta^i_0, w^i_0$ are taken as zero vectors of appropriate dimension, $\forall ~i\in {N}$. The Fisher information matrix inverse is initialized to $G_0^{i^{-1}} = 1.5\times I, ~ \forall~ i\in N$. We compared the MAAC algorithm with the multi-agent natural actor-critic algorithms FI-MAN, AP-MAN, and FIAP-MAN. The globally averaged reward, standard deviation and 95\% confidence intervals (averaged over 25 iterations) for all multi-agent natural actor-critic algorithms are given in Table \ref{table: avg_rewards_garnet}.

\begin{table}[!ht]
\centering
\begin{tabular}{|c|c|c|c|}
\hline
\textbf{Algorithm} & \textbf{Avg Rewards} & \textbf{Std Dvn} & \textbf{Confidence interval} \\ \hline \hline
\textbf{MAAC} & 1.993280 & 0.066421 &  (1.967243, 2.019316) \\ \hline
\textbf{FI-MAN} & 2.008412 & 0.055538 & (1.986642, 2.030183) \\ \hline
\textbf{AP-MAN} & 1.982451 & 0.079404 & (1.951325, 2.013576) \\ \hline
\textbf{FIAP-MAN} & 1.981089 & 0.093754 &  (1.944338, 2.017839) \\ \hline

\end{tabular}
\caption{Globally averaged rewards, standard deviation and 95\% confidence for all the algorithms for the abstract multi-agent RL problem. We observe that globally averaged rewards and standard deviation are almost same for all the algorithms with high confidence. All the values are averaged over $25$ runs.}
\label{table: avg_rewards_garnet}
\end{table}

We see that the globally averaged returns from all the algorithms are almost close with high confidence. Since the performance of all the algorithms are same, we have compared the relative $V$ values, defined as $V(s;v^i) = v^{i^{\top}}\varphi(s)$ for each agent $i\in N$ averaged over 25 runs with each run involving $12000$ iterations. These values are available in Table \ref{table: rel_v_values} (maximum values in each row are bold).
\begin{table}[!ht]
\centering
\begin{tabular}{|c|c|c|c|c|}
\hline
\multicolumn{1}{|l|}{\multirow{2}{*}{\textbf{Agents}}} & \multicolumn{4}{c|}{\textbf{Algorithms}} \\ \cline{2-5} 
\multicolumn{1}{|l|}{} & \textbf{FI-MAN} & \textbf{ AP-MAN} & \textbf{FIAP-MAN} & \textbf{MAAC} \\ \hline \hline
\multicolumn{1}{|c|}{1} & \textbf{0.099790} & -0.081895 & 0.096078 & -0.011488 \\ \hline
\multicolumn{1}{|c|}{2} & \textbf{0.082189} & -0.006449 & 0.030639 & 0.012191 \\ \hline
\multicolumn{1}{|c|}{3} & 0.055486 & -0.078218 & \textbf{0.082690} & -0.082770 \\ \hline
\multicolumn{1}{|c|}{4} & \textbf{0.076051} & -0.079353 & 0.012621 & 0.015753 \\ \hline
\multicolumn{1}{|c|}{5} & \textbf{0.083763} & -0.063922 & 0.061808 & 0.002050 \\ \hline
\multicolumn{1}{|c|}{6} & 0.020304 & -0.052126 & \textbf{0.108242} & -0.008364 \\ \hline
\multicolumn{1}{|c|}{7} & 0.056736 & -0.105311 & \textbf{0.087904} & 0.042844 \\ \hline
\multicolumn{1}{|c|}{8} & 0.036493 & -0.039098 & -0.004051 & \textbf{0.052373} \\ \hline
\multicolumn{1}{|c|}{9} & 0.101874 & -0.026014 & \textbf{0.153574} & 0.030972 \\ \hline
\multicolumn{1}{|c|}{10} & 0.040236 & -0.041242 & \textbf{0.045329} & 0.016195 \\ \hline
\multicolumn{1}{|c|}{11} & 0.024138 & -0.035986 & \textbf{0.101919} & -0.006431 \\ \hline
\multicolumn{1}{|c|}{12} & \textbf{0.106377} & -0.054188 & -0.018450 & 0.012857 \\ \hline
\multicolumn{1}{|c|}{13} & 0.058369 & -0.049010 & \textbf{0.078504} & 0.003431 \\ \hline
\multicolumn{1}{|c|}{14} & 0.049588 & -0.090194 & \textbf{0.056069} & -0.003832 \\ \hline
\multicolumn{1}{|c|}{15} & \textbf{0.090342} & -0.057247 & 0.066245 & 0.013376 \\ \hline
\end{tabular}
\caption{Relative $V$ values for each agent for all algorithms averaged over 25 runs. The values in bold represent the maximum in each row, i.e., for each agent. The relative $V$ values are maximum either for FI-MAN or FIAP-MAN algorithm for all the agents except agent $8$, suggesting that natural actor-critic algorithms have better relative $V$ value than standard gradient-based MAAC algorithm. However, the globally averaged returns are almost the same for all the algorithms as shown in Table \ref{table: avg_rewards_garnet}.}
\label{table: rel_v_values}
\end{table}

Though the globally averaged returns are almost the same for all the algorithms, the relative $V$ values are maximum for all but agent $8$ in the FI-MAN or FIAP-MAN algorithms showing the usefulness of multi-agent natural actor-critic algorithms.

% \textcolor{red}{Write that similar thing is observed in \cite{bhatnagar2009natural}}.

% Following are the observations from the above figures:
% \begin{itemize}
%     \item The globally averaged return from all the algorithms is almost the same.
%     \item The relative $V$ values for each agent $i\in N$ is defined as $V(s,v^i) = v^{i^\top} \varphi(s)$. Thus higher is better
%     \item The approximate of the global value function in algorithm \ref{alg: MANAC-fisher}, \ref{alg: MANAC-advantage-parameters} and \ref{alg: MANAC-advantage-parameters-fisher} reach consensus much faster than the algorithm \ref{alg: MARL-algo2} converges. 
%     \item Moreover, the relative values corresponding to natural gradient based algorithms are much more than the MARL algorithm \ref{alg: MARL-algo2}.
% \end{itemize}

\subsection{Experiments for traffic network control}
\label{app: experiments_traffic_network}

Here we provide some more details about the traffic network congestion model for both arrival patterns.

\subsubsection{More performance details for arrival pattern 1}
\label{app: 50k_pattern_1}

Though we aim to minimize the network congestion only, in Table \ref{table: 50K_each_light_pattern_1}, we also provide the congestion to each traffic light for the last 200 decision epochs (all the values are round off to 3 decimal places). Because of the $p_{s,ap}$'s for arrival pattern 1, as expected the traffic light $T_2$ is heavily congested; $T_1, T_4$ are almost equally congested and $T_3$ is the least congested. We also present the correction factor (CF) that is calculated as $1.96 \times \frac{std~dvn}{\sqrt{10}}$. It captures the 95\% confidence about the average congestion of each traffic light and the overall network congestion. 

% \input{Tables-Traffic Light/Mode 1/50k_new}

% \usepackage{multirow}
\begin{table}[h!]
\centering
\setlength{\tabcolsep}{1.5pt}
\begin{tabular}{|c|c|c|c|c|c|c|}
\hline
\multirow{2}{*}{\textbf{Algos}} & \multirow{2}{*}{\textbf{\begin{tabular}[c]{@{}c@{}}Decision \\ Epochs\end{tabular}}} & \multicolumn{5}{c|}{\textbf{Congestion (Avg  $\pm$  CF)}} \\ \cline{3-7} 
 &  & \textbf{$T_1$} & \textbf{$T_2$} & \textbf{$T_3$} & \textbf{$T_4$} & \textbf{Network} \\ \hline \hline
\multirow{5}{*}{\textbf{MAAC}} & 1300 & 3.742  $\pm$  0.071 & 4.344 $\pm$ 0.049 & 2.241 $\pm$ 0.012 & 3.716 $\pm$ 0.045 & 14.042 $\pm$ 0.053 \\ \cline{2-7} 
 & 1350 & 3.737  $\pm$  0.072 & 4.345 $\pm$ 0.050 & 2.239 $\pm$ 0.011 & 3.717 $\pm$ 0.041 & 14.037 $\pm$ 0.054 \\ \cline{2-7} 
 & 1400 & 3.742  $\pm$  0.065 & 4.344 $\pm$ 0.049 & 2.242 $\pm$ 0.011 & 3.708 $\pm$ 0.042 & 14.035 $\pm$ 0.043 \\ \cline{2-7} 
 & 1450 & 3.741  $\pm$  0.061 & 4.347 $\pm$ 0.053 & 2.246 $\pm$ 0.011 & 3.705 $\pm$ 0.042 & 14.039 $\pm$ 0.046 \\ \cline{2-7} 
 & 1500 & 3.733  $\pm$  0.067 & 4.336 $\pm$ 0.050 & 2.249 $\pm$ 0.013 & 3.699 $\pm$ 0.040 & 14.017 $\pm$ 0.052 \\ \hline \hline
\multirow{5}{*}{\textbf{FI}} & 1300 & 2.652  $\pm$  0.114 & 4.515 $\pm$ 0.838 & 2.359 $\pm$ 0.159 & 2.598 $\pm$ 0.041 & 12.124 $\pm$ 0.896 \\ \cline{2-7} 
 & 1350 & 2.638  $\pm$  0.112 & 4.516 $\pm$ 0.851 & 2.358 $\pm$ 0.159 & 2.589 $\pm$ 0.042 & 12.100 $\pm$ 0.903 \\ \cline{2-7} 
 & 1400 & 2.633  $\pm$  0.110 & 4.514 $\pm$ 0.856 & 2.358 $\pm$ 0.160 & 2.579 $\pm$ 0.039 & 12.084 $\pm$ 0.910 \\ \cline{2-7} 
 & 1450 & 2.623  $\pm$  0.111 & 4.515 $\pm$ 0.871 & 2.361 $\pm$ 0.161 & 2.569 $\pm$ 0.039 & 12.068 $\pm$ 0.921 \\ \cline{2-7} 
 & 1500 & 2.613  $\pm$  0.110 & 4.492 $\pm$ 0.868 & 2.359 $\pm$ 0.161 & 2.564 $\pm$ 0.038 & 12.028 $\pm$ 0.918 \\ \hline \hline
\multirow{5}{*}{\textbf{AP}} & 1300 & 3.754  $\pm$  0.077 & 4.349 $\pm$ 0.045 & 2.236 $\pm$ 0.013 & 3.758 $\pm$ 0.041 & 14.097 $\pm$ 0.048 \\ \cline{2-7} 
 & 1350 & 3.747  $\pm$  0.078 & 4.346 $\pm$ 0.046 & 2.235 $\pm$ 0.011 & 3.761 $\pm$ 0.036 & 14.089 $\pm$ 0.047 \\ \cline{2-7} 
 & 1400 & 3.755  $\pm$  0.072 & 4.345 $\pm$ 0.045 & 2.238 $\pm$ 0.011 & 3.751 $\pm$ 0.037 & 14.091 $\pm$ 0.039 \\ \cline{2-7} 
 & 1450 & 3.756  $\pm$  0.068 & 4.348 $\pm$ 0.048 & 2.243 $\pm$ 0.011 & 3.752 $\pm$ 0.041 & 14.099 $\pm$ 0.041 \\ \cline{2-7} 
 & 1500 & 3.748  $\pm$  0.073 & 4.338 $\pm$ 0.046 & 2.247 $\pm$ 0.013 & 3.746 $\pm$ 0.039 & 14.079 $\pm$ 0.051 \\ \hline \hline
\multirow{5}{*}{\textbf{FIAP}} & 1300 & 2.757  $\pm$  0.245 & 3.855 $\pm$ 0.376 & 1.929 $\pm$ 0.148 & 2.939 $\pm$ 0.165 & 11.479 $\pm$ 0.691 \\ \cline{2-7} 
 & 1350 & 2.743  $\pm$  0.236 & 3.833 $\pm$ 0.372 & 1.923 $\pm$ 0.149 & 2.926 $\pm$ 0.161 & 11.424 $\pm$ 0.680 \\ \cline{2-7} 
 & 1400 & 2.735  $\pm$  0.229 & 3.827 $\pm$ 0.383 & 1.916 $\pm$ 0.147 & 2.910 $\pm$ 0.156 & 11.388 $\pm$ 0.675 \\ \cline{2-7} 
 & 1450 & 2.723  $\pm$  0.222 & 3.809 $\pm$ 0.380 & 1.913 $\pm$ 0.149 & 2.897 $\pm$ 0.152 & 11.342 $\pm$ 0.662 \\ \cline{2-7} 
 & 1500 & 2.711  $\pm$  0.214 & 3.785 $\pm$ 0.371 & 1.907 $\pm$ 0.148 & 2.883 $\pm$ 0.147 & 11.287 $\pm$ 0.645 \\ \hline
\end{tabular}
\caption{Table shows the average congestion $\pm$ correction factor (averaged over 10 runs) to each traffic light and to the entire network in last 200 decision epochs for arrival pattern 1. As expected, the average congestion is the highest for traffic light $T_2$ in all the algorithms, almost same for traffic lights $T_1,T_4$ and least for traffic light $T_3$. The overall network congestion is simply the sum of congestion at all the traffic lights. Moreover, for the last 200 decision epochs all the values are close to each other with high confidence implying that the algorithms indeed are attaining the local minima.}
\label{table: 50K_each_light_pattern_1}
\end{table}

For completeness, we now provide the plots of norm of the differences for  agents $T_2,T_3,T_4$ in Figure \ref{fig: norm_diff_other_agents_50k_pattern_1}. 

\begin{figure}[h!]
\centering
\begin{subfigure}[b]{0.48\textwidth}
\includegraphics[width =\textwidth]{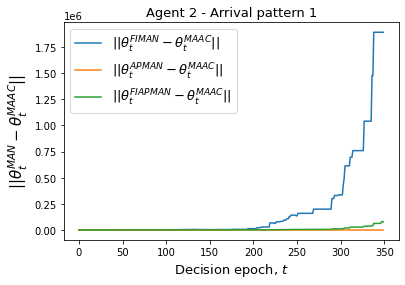}
\end{subfigure}
\begin{subfigure}[b]{0.48\textwidth}
\includegraphics[width =\textwidth]{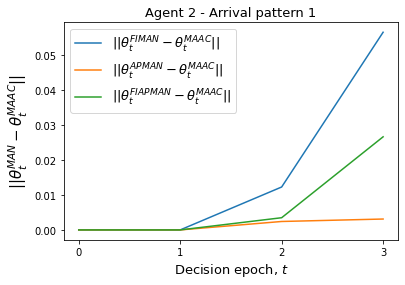}
\end{subfigure}
\begin{subfigure}[b]{0.48\textwidth}
\includegraphics[width =\textwidth]{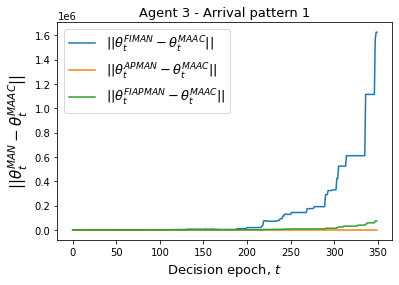}
\end{subfigure}
\begin{subfigure}[b]{0.48\textwidth}
\includegraphics[width =\textwidth]{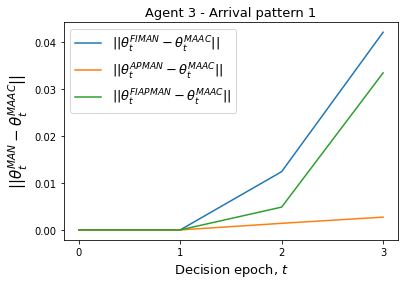}
\end{subfigure}
\begin{subfigure}[b]{0.48\textwidth}
\includegraphics[width =\textwidth]{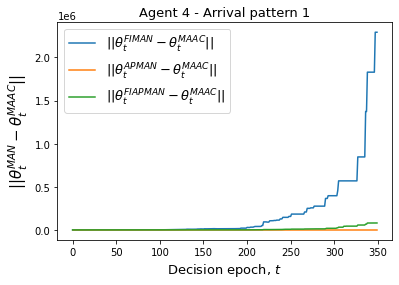}
\end{subfigure}
\begin{subfigure}[b]{0.48\textwidth}
\includegraphics[width =\textwidth]{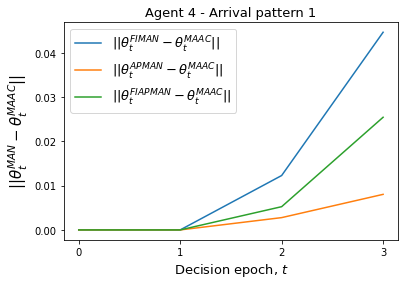}
\end{subfigure}
\caption{Norm of difference of agent $T_2,T_3,T_4$ for all the algorithms. The left panel is shown upto 350 epochs, however to show that actor parameters are actually differing from iterate 2 itself, we zoom them in the left figure. These observations illustrate Theorem \ref{thm: J_comp_t+1} and related discussions in Section \ref{subsec: algo_actor_comparison}.}
\label{fig: norm_diff_other_agents_50k_pattern_1}
\end{figure}

\subsubsection{More performance details of arrival pattern 2}
\label{app: 50k_pattern_2}
Table \ref{table: 50K_each_light_pattern_2} provides the average network congestion and the congestion to each traffic light for last 200 decision epochs. As expected from arrival pattern 2, all the traffic lights are almost equally congested. Again CF is the correction factor as defined earlier. 

% \input{Tables-Traffic Light/Mode 2/Each_light_50K}

% \usepackage{multirow}
\begin{table}[h!]
\centering
\setlength{\tabcolsep}{2.2pt}
\begin{tabular}{|c|c|c|c|c|c|c|}
\hline
\multirow{2}{*}{\textbf{Algos}} & \multirow{2}{*}{\textbf{\begin{tabular}[c]{@{}c@{}}Decision \\ Epochs\end{tabular}}} & \multicolumn{5}{c|}{\textbf{Congestion (Avg  $\pm$  CF)}} \\ \cline{3-7} 
 &  & \textbf{$T_1$} & \textbf{$T_2$} & \textbf{$T_3$} & \textbf{$T_4$} & \textbf{Network} \\ \hline \hline
\multirow{5}{*}{\textbf{MAAC}} & 1300 & 3.486 $\pm$ 0.074 & 3.378 $\pm$ 0.031 & 3.348 $\pm$ 0.042 & 3.445 $\pm$ 0.037 & 13.658 $\pm$ 0.130 \\ \cline{2-7} 
 & 1350 & 3.485 $\pm$ 0.073 & 3.375 $\pm$ 0.031 & 3.361 $\pm$ 0.041 & 3.448 $\pm$ 0.035 & 13.670 $\pm$ 0.130 \\ \cline{2-7} 
 & 1400 & 3.482 $\pm$ 0.071 & 3.373 $\pm$ 0.034 & 3.355 $\pm$ 0.040 & 3.448 $\pm$ 0.033 & 13.656 $\pm$ 0.129 \\ \cline{2-7} 
 & 1450 & 3.489 $\pm$ 0.073 & 3.381 $\pm$ 0.035 & 3.350 $\pm$ 0.034 & 3.442 $\pm$ 0.036 & 13.662 $\pm$ 0.130 \\ \cline{2-7} 
 & 1500 & 3.481 $\pm$ 0.070 & 3.379 $\pm$ 0.034 & 3.348 $\pm$ 0.032 & 3.438 $\pm$ 0.033 & 13.646 $\pm$ 0.122 \\ \hline \hline
\multirow{5}{*}{\textbf{FI}} & 1300 & 2.560 $\pm$ 0.048 & 2.554 $\pm$ 0.041 & 2.541 $\pm$ 0.046 & 2.609 $\pm$ 0.052 & 10.264 $\pm$ 0.087 \\ \cline{2-7} 
 & 1350 & 2.554 $\pm$ 0.047 & 2.547 $\pm$ 0.044 & 2.539 $\pm$ 0.045 & 2.605 $\pm$ 0.048 & 10.244 $\pm$ 0.083 \\ \cline{2-7} 
 & 1400 & 2.545 $\pm$ 0.047 & 2.540 $\pm$ 0.043 & 2.534 $\pm$ 0.044 & 2.599 $\pm$ 0.047 & 10.217 $\pm$ 0.079 \\ \cline{2-7} 
 & 1450 & 2.541 $\pm$ 0.046 & 2.535 $\pm$ 0.040 & 2.524 $\pm$ 0.044 & 2.593 $\pm$ 0.045 & 10.194 $\pm$ 0.073 \\ \cline{2-7} 
 & 1500 & 2.532 $\pm$ 0.047 & 2.530 $\pm$ 0.040 & 2.519 $\pm$ 0.045 & 2.589 $\pm$ 0.045 & 10.170$\pm$ 0.074 \\ \hline \hline
\multirow{5}{*}{\textbf{AP}} & 1300 & 3.515 $\pm$ 0.070 & 3.400 $\pm$ 0.026 & 3.377 $\pm$ 0.046 & 3.488 $\pm$ 0.037 & 13.780 $\pm$ 0.129 \\ \cline{2-7} 
 & 1350 & 3.514 $\pm$ 0.067 & 3.398 $\pm$ 0.028 & 3.388 $\pm$ 0.045 & 3.490 $\pm$ 0.031 & 13.792 $\pm$ 0.126 \\ \cline{2-7} 
 & 1400 & 3.513 $\pm$ 0.064 & 3.397 $\pm$ 0.032 & 3.383 $\pm$ 0.043 & 3.490 $\pm$ 0.029 & 13.783 $\pm$ 0.123 \\ \cline{2-7} 
 & 1450 & 3.518 $\pm$ 0.066 & 3.407 $\pm$ 0.033 & 3.377 $\pm$ 0.037 & 3.488 $\pm$ 0.035 & 13.790 $\pm$ 0.126 \\ \cline{2-7} 
 & 1500 & 3.511 $\pm$ 0.064 & 3.404 $\pm$ 0.032 & 3.376 $\pm$ 0.034 & 3.485 $\pm$ 0.033 & 13.776 $\pm$ 0.117 \\ \hline \hline
\multirow{5}{*}{\textbf{FIAP}} & 1300 & 2.591 $\pm$ 0.043 & 2.558 $\pm$ 0.037 & 2.570 $\pm$ 0.043 & 2.579 $\pm$ 0.038 & 10.298 $\pm$ 0.140 \\ \cline{2-7} 
 & 1350 & 2.586 $\pm$ 0.042 & 2.549 $\pm$ 0.036 & 2.564 $\pm$ 0.043 & 2.576 $\pm$ 0.034 & 10.275 $\pm$ 0.134 \\ \cline{2-7} 
 & 1400 & 2.577 $\pm$ 0.042 & 2.544 $\pm$ 0.035 & 2.557 $\pm$ 0.041 & 2.569 $\pm$ 0.035 & 10.247 $\pm$ 0.135 \\ \cline{2-7} 
 & 1450 & 2.574 $\pm$ 0.039 & 2.540 $\pm$ 0.034 & 2.549 $\pm$ 0.041 & 2.562 $\pm$ 0.034 & 10.226 $\pm$ 0.128 \\ \cline{2-7} 
 & 1500 & 2.566 $\pm$ 0.041 & 2.534 $\pm$ 0.034 & 2.543 $\pm$ 0.041 & 2.556 $\pm$ 0.036 & 10.199 $\pm$ 0.132 \\ \hline
\end{tabular}
\caption{Table shows the average congestion  $\pm$  correction factor (averaged over 10 runs) to each traffic light and to the entire network in last 200 decision epochs for arrival pattern 2. As expected, the average congestion is almost the same for all traffic lights  in all the algorithms. }
\label{table: 50K_each_light_pattern_2}
\end{table}

We again plot the norm of the difference of the actor parameter for remaining agent $T_2,T_3$ and $T_4$ in Figure \ref{fig: norm_diff_other_agents_50k_pattern_2} .

\begin{figure}[h!]
\begin{subfigure}[b]{0.48\textwidth}
\includegraphics[width =\textwidth]{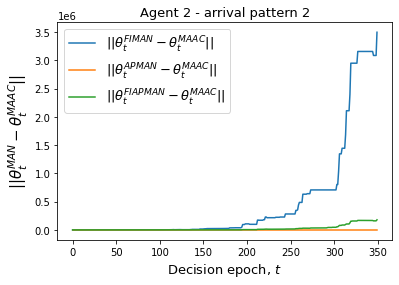}
\end{subfigure}
\begin{subfigure}[b]{0.5\textwidth}
\includegraphics[width =\textwidth]{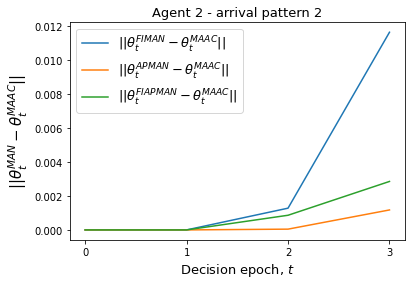}
\end{subfigure}
\begin{subfigure}[b]{0.48\textwidth}
\includegraphics[width =\textwidth]{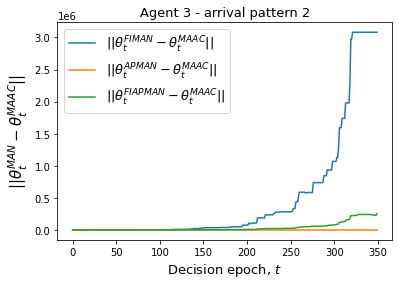}
\end{subfigure}
\begin{subfigure}[b]{0.5\textwidth}
\includegraphics[width =\textwidth]{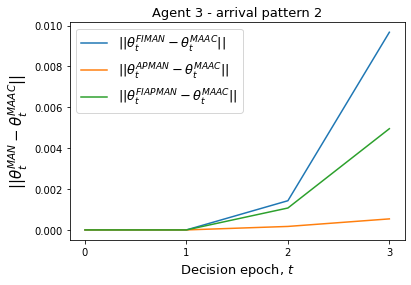}
\end{subfigure}
\begin{subfigure}[b]{0.48\textwidth}
\includegraphics[width =\textwidth]{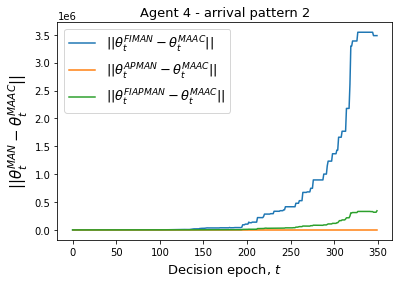}
\end{subfigure}
\begin{subfigure}[b]{0.51\textwidth}
\includegraphics[width =\textwidth]{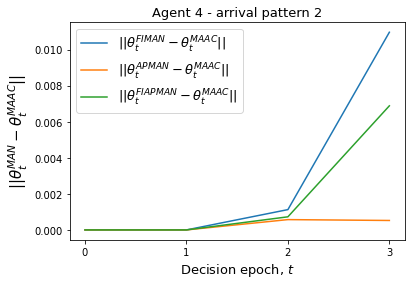}
\end{subfigure}
\caption{For arrival pattern 2, norm of difference of agent $T_2,T_3,T_4$ for all the algorithms. The left panel is shown upto 350 epochs, however to show that actor parameters are actually differing from iterate 2 itself, we zoom them in the left figure. These observations illustrate Theorem \ref{thm: J_comp_t+1} in Section \ref{subsec: algo_actor_comparison}.}
\label{fig: norm_diff_other_agents_50k_pattern_2}
\end{figure}

\subsubsection{Justification of Equation (\ref{eqn: binomial})}
\label{app: justification_eqn_40}
We first use the law of total probability to show that for a given arrival pattern $ap$ the number of arrivals to the source node $s$ at time $t$ follows the binomial distribution as given in Equation (\ref{eqn: binomial}). 

\begin{align*}
\mathbb{P}(N_t^s=k | M_t=ap ) &= \sum_{j=0}^{N_v} \mathbb{P}(N_t^s = k | N_t=j, M_t=ap) \mathbb{P}(N_t=j | M_t=ap) \\ 
&= \sum_{j=k}^{N_v} \binom{j}{k} p_{s,ap}^k (1-p_{s,ap})^{j-k} \mathbb{P}(N_t = j) \\
&= p_{s,ap}^k \sum_{j=0}^{N_v} \binom{j}{k} \binom{N_v}{j} (1-p_{s,ap})^{j-k} \left(\frac{1}{T}\right)^j  \left( 1-\frac{1}{T}\right)^{N_v-j}  \\
&= \binom{N_v}{k} p_{s,ap}^k \sum_{l=0}^{N_v-k}\binom{N_v-k}{l} (1-p_{s,ap})^{N_v-k-l} \left(\frac{1}{T} \right)^{N_v-l} \left(1-\frac{1}{T} \right)^l \\
&= \binom{N_v}{k} \left(\frac{p_{s,ap}}{T}\right)^k \left(1-\frac{p_{s,ap}}{T} \right)^{N_v-k}.
\end{align*}

\section{Some relevant background}
\label{app: relevant_background}
In this section, we will provide details of single-agent MDP, actor-critic algorithm, the MAAC algorithm \cite{zhang2018fully}, and the Kushner-Clark Lemma.

\subsection{Markov decision processes and actor-critic algorithms with linear function approximation}
\label{app: MDP_AC_fn_approx}

Markov decision process (MDP) is a stochastic control process that provides a mathematical framework for sequential decision making. Formally it is defined as below.
\begin{definition}[MDP \cite{puterman2014markov}]
\label{def: MDP}
A MDP is characterised by a quadruple $\mathcal{M} = \langle \mathcal{S,A},P, r\rangle$, where $\mathcal{S}$ and $\mathcal{A}$ are finite state and action space respectively. ${P} : \mathcal{S}\times\mathcal{A}\times\mathcal{S}\rightarrow [0,1]$ is the transition probability of taking action $a$ in state $s$ and reaching to state $s^{\prime}$. ${R}$ : $\mathcal{S}\times\mathcal{A}\rightarrow \mathbb{R}$ is the reward function defined as ${R}(s,a) = \mathbb{E}[r_{t+1}|s_t=s,a_t=a]$ where $r_{t+1}$ is the instantaneous reward at time $t$.
\end{definition}

% The instantaneous reward $r_t$ is assumed to be uniformly bounded for any $t\geq 0$. 
The policy of an agent is a decision rule using which agent takes action in each state; formally, it is a mapping $\pi: \mathcal{S}\times\mathcal{A}\rightarrow[0,1]$. It also represents the probability of taking action $a$ in the state $s$. The goal of the agent is to maximize the expected average reward,
\begin{equation*}
\label{eqn: obj_mdp}
\max_{\pi} ~~ J(\pi) = \max_{\pi} ~~ \underset{T}{lim}\hspace{2pt} \dfrac{1}{T}\sum_{t=0}^{T-1}\mathbb{E}(r_{t+1}) = \max_{\pi} ~~ \sum_{s \in \mathcal{S}}d_{\pi}(s)\sum_{a \in \mathcal{A}}\pi(s,a){R}(s,a),
\end{equation*}
where $d_{\pi}(s) = \underset{t}{lim}~ \mathbb{P}(s_t=s|\pi)$ is the stationary state distribution of the Markov chain under policy $\pi$. The action value function for a policy, $\pi(\cdot,\cdot)$ is defined as 
\begin{equation*}
\label{eqn: Q_fn_MDP}
Q_{\pi}(s,a) = \sum_{t\geq 0}\mathbb{E}[r_{t+1}-J(\pi)|s_0=s,a_0=a,\pi].
\end{equation*}
The state value function $V_{\pi}(s)$ for any state $s$ is defined as  
\begin{equation*}
\label{eqn: V_fn_MDP}
{V}_{\pi}(s) = \sum_{a \in \mathcal{A}}\pi(s,a){Q}_{\pi}(s,a). 
\end{equation*}

In most real-life scenarios, the state space, the action space, or both are large and infinite, and hence finding the optimal policy is computationally heavy 
% The most common technique to avoid this is to work with the function approximations, i.e., construct a function that mimics the behavior of the actual scenario, and at the same time, it is computationally tractable. To this end, we will apply the function approximations for the state action-value function and the policies. 
so, it is useful to consider the parameterized policies.
% The policy gradient methods majorly consist of three steps: define the class of randomized local policy $\pi$ for the agent parameterized by  $\theta \in \Theta \subseteq \mathbb{R}^{m}$, where $\Theta$ is the compact set; estimate the gradient of the average reward for the policy parameters, and then improve the policy by adjusting the parameters in the direction of the estimate of the objective function. 
Let the parameterized policy be denoted by $\pi_{\theta}$. We need the following regularity assumption on the parameterized policy function \cite{bhatnagar2009natural,sutton2018reinforcement}. For any $s \in \mathcal{S}$, and $a\in \mathcal{A}$, the policy function $\pi_{\theta}(s,a)>0$ for any $\theta \in \Theta$. Also, $\pi_{\theta}(s,a)$ is continuously differentiable with respect to the parameter $\theta$ over $\Theta$. 
Moreover, for any $\theta \in \Theta$, let $P^{\theta}$ be the transition matrix for the Markov chain $\{s_t\}_{t\geq 0}$ induced by policy $\pi_{\theta}$, that is, for any $s,s^{\prime} \in \mathcal{S}$,
\begin{equation*}
\label{eqn: P^theta_MDP}
    P^{\theta}(s^{\prime}|s) = \sum_{a\in \mathcal{A}} \pi_{\theta}(s,a) P(s^{\prime}|s,a).
\end{equation*}
Futhermore, the Markov chain $\{s_t\}_{t\geq 0}$ is assumed to be ergodic under $\pi_{\theta}$ with stationary distribution denoted by $d_{\theta}(s)$ over $\mathcal{S}$.

Under this parameterization, the policy gradient theorem \cite{sutton2018reinforcement} is: $\nabla_{\theta} J(\theta) = \mathbb{E}_{s\sim d_{\theta},a\sim \pi_{\theta}}[\nabla_{\theta}\hspace{2pt}log\hspace{2pt} \pi_{\theta}(s,a)\cdot\{Q_{\theta}(s,a)-b(s)\}]$.
The term $b(s)$ is usually referred to as the baseline.
This baseline helps in reducing the variance in the gradient of the objective function. It turns out that ${V}_{\theta}$(s) serves as the minimum variance baseline. We define the advantage function as follows: $ A_{\theta}(s,a) = {Q}_{\theta}(s,a)-{V}_{\theta}(s)$.

% The actor-critic algorithm is obtained by parameterizing the state-action value function $Q(s,a)$. At time step $t$, define ${Q}_t(w)$ := ${Q}_t(s_t,a_t;w)$ and $\psi_t:= \psi_t(s_t,a_t)$ = $\nabla_{\theta} \log\pi_{\theta_t}(s_t,a_t)$. One common actor-critic algorithm has the following form 
% \begin{eqnarray}
% \mu_{t+1}&=&(1-\beta_{w,t})\cdot\mu_t+\beta_{w,t}\cdot r_{t+1}
% \\
% w_{t+1}&=&w_t+\beta_{w,t}\cdot\delta_t \cdot \nabla_{w}{Q}_t(w_t)
% \\
% \theta_{t+1}&=&\theta_t+\beta_{\theta,t}\cdot A_t \cdot \psi_t
% \end{eqnarray}
% where $\mu_t$ tracks the estimate of averaged return $J(\theta)$, and $\beta_{\theta,t},\beta_{w,t}$ are step sizes satisfying following 
% \begin{eqnarray}
%     \sum_{t} \beta_{w,t} = \sum_{t} \beta_{\theta,t} = \infty
%     \\
%     \sum_t \beta_{w,t}^2 + \beta_{\theta,t}^2 <\infty
% \end{eqnarray}
% Moreover, $\beta_{\theta,t} = o(\beta_{w,t})$, and $lim_{t} \frac{\beta_{w,t+1}}{\beta_{w,t}} = 1$. 
% Furthermore, $\delta_t$ is the action-value temporal difference (TD) error that is given by
% \begin{equation}
% \label{eqn: TD_error_MDP_Q_fn}
% \delta_t=r_{t+1}-\mu_t+{Q}(s_{t+1},a_{t+1};w_t)-{Q}(s_{t},a_{t};w_t)
% \end{equation}

The actor-critic algorithm is obtained by parameterizing the state value function. Let the state value function, $V_{\theta}(s)$ be parameterized as $V_{t+1}(v_t) := v_t^{\top}\varphi(s_{t+1})$, where $v_t$ are the parameters and $\varphi(s_{t+1})$ are the features associated to the state $s_{t+1}$. Moreover, let $\mu_t$ be the estimate of the objective function at time $t$. The actor-critic algorithms use the following to update the critic and actor parameters
\begin{equation*}
\begin{aligned}
\textbf{Critic update:} \hspace{4 mm} \mu_{t+1}&=(1-\beta_{v,t})\cdot\mu_t+\beta_{v,t}\cdot r_{t+1}; \hspace{3 mm} v_{t+1}=v_t+\beta_{v,t}\cdot\delta_t \cdot \nabla_{v}{V}_t(v_t)
\\
\textbf{Actor update:} \hspace{4 mm} \theta_{t+1}&=\theta_t+\beta_{\theta,t}\cdot \delta_t \cdot \psi_t,
\end{aligned}
\end{equation*}
where $\delta_t$ is the TD error involving the state value function and defined as  $\delta_t=r_{t+1}-\mu_t+{V}_{t+1}(v_t)-{V}_t(v_t)$. It is known to be an unbiased estimate of the advantage function \cite{bhatnagar2009natural}, i.e., $\mathbb{E}[\delta_t | s_t=s,a_t=s, \pi_{\theta}] = A_{\theta}(s,a)$. 
Here $\beta_{\theta,t}$ and $\beta_{v,t}$ are step-sizes and satisfy the following conditions
\begin{equation*}
    a) \sum_{t} \beta_{v,t} = \sum_{t} \beta_{\theta,t} = \infty;~~~ b)\sum_t (\beta_{v,t}^2 + \beta_{\theta,t}^2) <\infty,
\end{equation*}
moreover, $\beta_{\theta,t} = o(\beta_{v,t})$, and $lim_{t} \frac{\beta_{v,t+1}}{\beta_{v,t}} = 1$.

\subsection{Taylor's formula with Lagrange's form of remainder in several variables \cite{folland1999real}}
\label{app: Taylor_expansion}

Let $f : \mathbb{R}^k \mapsto \mathbb{R}$ be a continuously double differentiable function and $S$ be an open convex set in $\mathbb{R}^k$. Then for any $\mathbf{a},\mathbf{h} \in S$, the first order Taylor expansion of $f(\mathbf{a}+\mathbf{h})$ with Lagrange form of remainder is \cite{folland1999real}
\begin{equation}
    f(\mathbf{a}+\mathbf{h}) = f(\mathbf{a}) + \mathbf{a}^{\top} \nabla f(\mathbf{a}) + R_1(\mathbf{h})
\end{equation}
where $R_1(\mathbf{h})$ is the Lagrange remainder given by
\begin{equation}
    R_1(\mathbf{h}) = \frac{1}{2!} \mathbf{a}^{\top} \nabla^2 f(\mathbf{a}+c\cdot \mathbf{h}) \mathbf{a} ~ for ~some ~c\in (0,1).
\end{equation}
% \textcolor{red}{We can make this result as Lemma}. \textcolor{black}{we will leave this as it is; we will recall the Eq 88 whenever needed. Else, we can write this as a lemma with a reference to a good book. }
Moreover, for the above Taylor expansion we have:
%\textcolor{black}{end sentence here?} following result from \cite{folland1999real}. 
If each entry of the Hessian, $\nabla^2 f(\mathbf{a})$ is uniformly bounded by $H$, i.e., $|\{\nabla^2 f(\mathbf{a})\}_{i,j}| \leq M,~\forall~ i,j\in [k]$, then
\begin{equation}
    |R_1(\mathbf{h})| \leq \frac{H}{2!} ||\mathbf{h}||^2_1, ~ where~ ||\mathbf{h}||_1 = |h_1| + |h_2|+ \dots + |h_n|.
\end{equation}
% \textcolor{black}{Since $f$ is assumed to be continuously double differentiable on $S$, such a $M$ exists by Weierstrass theorem.}
Next, we provide the Kushner-Clark lemma that we often use while proving the convergence. 

\subsection{Kushner-Clark Lemma \cite{kushner2003stochastic,metivier1984applications}}
\label{app: K-C_lemma}
% In this Section, we will now provide the Kushner-Clark lemma
Let $\mathcal{X}\subseteq \mathbb{R}^p$ be a compact set and let $h: \mathcal{X} \rightarrow \mathbb{R}^p$ be a continuous function. Consider the following recursion in $p$-dimensions 
\begin{equation}
\label{eqn: x_recursion}
    x_{t+1} = \Gamma\{x_t + \gamma_t[h(x_t) + \zeta_t + \beta_t]\}.
\end{equation}
Let $\hat{\Gamma}(\cdot)$ be transformed projection operator defined for any $x\in \mathcal{X}\subseteq \mathbb{R}^{p}$ 
% and $h: \mathcal{X} \rightarrow \mathbb{R}^p$ be a continuous function with $\mathcal{X}$ being a compact set 
 as
\begin{equation*}
    \hat{\Gamma}(h(x)) = lim_{0< \eta \rightarrow 0} \left\lbrace \frac{\Gamma(x+\eta h(x)) - x}{\eta} \right\rbrace,
\end{equation*}
then the ODE associated with Equation (\ref{eqn: x_recursion}) is $\dot x = \hat{\Gamma} (h(x))$.
\begin{assumption}
\label{ass: K-C_lemma}
Kushner-Clark lemma requires following assumptions
\begin{enumerate}
% \item $h(\cdot)$ is a continuous $\mathbb{R}^p$ valued function.
\item Stepsize $\{\gamma_t\}_{t\geq 0}$ satisfy $\sum_{t} \gamma_t = \infty$, and $\gamma_t \rightarrow 0$ as $t\rightarrow \infty$.
\item The sequence $\{\beta_t\}_{t\geq 0}$ is a bounded random sequence with $\beta_t \rightarrow 0$ almost surely as $t\rightarrow \infty$.
\item For any $\epsilon > 0$, the sequence $\{\zeta_t\}_{t\geq 0}$ satisfy
\begin{equation*}
    \lim_t~\mathbb{P}\left( sup_{p\geq t} \left\Vert \sum_{\tau = t}^p \gamma_{\tau}\zeta_{\tau} \right\Vert
    \geq \epsilon \right) = 0.
\end{equation*}
\end{enumerate}
\end{assumption}

Kushner-Clark lemma is as follows: suppose that ODE $\dot x = \hat{\Gamma} (h(x))$ has a compact set $\mathcal{K}^{\star}$ as its asymptotically stable equilibria, then under assumption  X. \ref{ass: K-C_lemma}, $x_t$ in Equation (\ref{eqn: x_recursion}) converges almost surely to $\mathcal{K}^{\star}$ as $t\rightarrow \infty$. 

\subsection{Multi-agent actor-critic algorithm based on state value function}
\label{app: MARL_algo2}
For completeness, we reproduce the multi-agent actor-critic (MAAC) algorithm based on the state value function \cite{zhang2018fully}. 
\begin{marl}[!ht]
\caption{Multi-agent actor-critic based on state value function} 
% 	\SetAlgoNoLine
	\KwIn{Initial values of $\mu^i_0, \tilde{\mu}^i_0,v^i_0, \tilde{v}^i_0, \lambda^i_0, \tilde{\lambda}^i_0, \theta^i_0,~ \forall i\in {N}$; initial state $s_0$; stepsizes $\{\beta_{v,t}\}_{t\geq 0}, \{\beta_{\theta,t}\}_{t\geq 0}$.\\
	Each agent $i$ implements $a^i_0 \sim \pi_{\theta^i_0}(s_0,\cdot)$.
	\\
	Initialize the step counter $t\leftarrow 0$.
	}
	\Repeat{Convergence}
    {\For {all $i\in {N}$}
    {Observe state $s_{t+1}$, and reward $r^i_{t+1}$. \\
    Update: $\tilde{\mu}^i_t \leftarrow (1-\beta_{v,t})\cdot \mu^i_t + \beta_{v,t} \cdot r^i_{t+1}$.
    \\
    $\tilde{\lambda}^i_t \leftarrow \lambda^i_t + \beta_{v,t} \cdot [r^i_{t+1} - \bar{R}_t(\lambda^i_t)]  \cdot \nabla_{\lambda}\bar{R}_t(\lambda^i_t)$,~ 
    where $\bar{R}_t(\lambda^i_t) = \lambda^{i^{\top}}_t f(s_t,a_t)$.
    \\
    Update: $\delta^i_t \leftarrow r^i_{t+1} - \mu^i_t + V_{t+1}(v^i_t) - V_{t}(v^i_t)$, 
    ~where $V_{t+1}(v^i_t) = v^{i^{\top}}_t \varphi(s_{t+1})$.
    \\
    \textbf{Critic Step:} $\tilde{v}^i_t \leftarrow v^i_t + \beta_{v,t} \cdot \delta^i_t \cdot \nabla_v V_t(v^i_t)$
    \\
    Update $\tilde{\delta}^i_t \leftarrow \bar{R}_{t}(\lambda^i_t) - \mu^i_t + V_{t+1}(v^i_t) - V_{t}(v^i_t)$;~ $\psi^i_t \leftarrow \nabla_{\theta^i} \log \pi^i_{\theta^i_t}(s_t,a^i_t)$.
    \\
    \textbf{Actor Step:} $\theta^i_{t+1} \leftarrow \theta^i_t + \beta_{\theta,t} \cdot \tilde{\delta}^i_t \cdot \psi^i_t$. 
    \\
    Send $\tilde{\mu}^i_t, \tilde{\lambda}^i_t, \tilde{v}^i_t$ to the neighbors over $\mathcal{G}_t$.}
    {\For {all $i\in {N}$}
    {\textbf{Consensus update:} $\mu^i_{t+1} \leftarrow  \sum_{j\in {N}} c_t(i,j) \tilde{\mu}^j_t$;
    \\
    $\lambda^i_{t+1} \leftarrow \sum_{j\in N} c_t(i,j) \tilde{\lambda}^j_t;~~v^i_{t+1} \leftarrow \sum_{j\in {N}}c_t(i,j)\tilde{v}^j_t $.
    }}
    Update: $t\leftarrow t+1$.}
\label{alg: MARL-algo2}
\end{marl}

\newpage

\end{appendix}
\end{document}